%% file: main_nh.tex
\documentclass{article}

 \usepackage[preprint]{neurips_2026}

\usepackage[utf8]{inputenc} %
\usepackage[T1]{fontenc}    %
\usepackage{hyperref}       %
\usepackage{url}            %
\usepackage{booktabs}       %
\usepackage{amsmath,amsthm,amssymb,bbm}
\usepackage{amsfonts}       %
\usepackage{mathtools}
\usepackage{nicefrac}       %
\usepackage{microtype}      %
\usepackage{xcolor}         %

\usepackage{algorithm,algorithmic}

\usepackage{enumitem}
\setlist[enumerate]{itemsep=0.1em}
\usepackage{cleveref}
\crefformat{equation}{(#2#1#3)}
\crefrangeformat{equation}{(#3#1#4) to~(#5#2#6)}
\crefname{equation}{}{}
\Crefname{equation}{}{}

\crefname{definition}{\textbf{definition}}{definitions}
\Crefname{definition}{Definition}{Definitions}
\crefname{assumption}{\textbf{assumption}}{assumptions}
\Crefname{assumption}{Assumption}{Assumptions}

\definecolor{maroon}{RGB}{192,80,77}
\definecolor{orange}{RGB}{219, 110, 0}

\newcommand{\explain}[2]{\underset{\mathclap{\overset{\uparrow}{#2}}}{#1}}

\newcommand{\yw}[1]{\textit{\textcolor{red}{[yuxiang]: #1}}}
\newcommand{\yf}[1]{\textit{\textcolor{blue}{[yoav]: #1}}}

\newcommand{\comment}[1]{}

 \newtheorem{theorem}{Theorem}
 \newtheorem{lemma}[theorem]{Lemma}
 \newtheorem{proposition}[theorem]{Proposition}
 
 \newtheorem{corollary}[theorem]{Corollary}
 \newtheorem{definition}[theorem]{Definition}

\newtheorem{remark}[theorem]{Remark}

\newcommand{\Exp}{\mathbb{E}}
\newcommand{\myset}[1]{\left \{ #1 \right \}}                     %
\newcommand{\setst}[2]{\left\{\; #1 \,:\, #2 \;\right\}}        %
\newcommand{\Indicator}[1]{\mathbf{1}_{#1}}
\newcommand{\eps}{\epsilon}
\newcommand{\floor}[1]{\left\lfloor #1 \right\rfloor}
\newcommand{\prob}[1]{\operatorname{Pr}\left[\,#1\,\right]}   
\newcommand{\card}[1]{\abs{#1}}
\newcommand{\abs}[1]{\lvert #1 \rvert}
\newcommand{\umax}{\alpha_\mathrm{max}}
\newcommand{\expect}[1]{\operatorname{E}\left[\,#1\,\right]}  
 
\newcommand{\Union}{\bigcup}
\newcommand{\erf}{\mathrm{erf}}
\newcommand{\smallfrac}[2]{{\textstyle \frac{#1}{#2}}}

\newcommand{\blue}[1]{#1}
\newcommand{\diff}{\mathrm{d}}

\newcommand{\Tp}{\textsf{T}}

\def\E{\mathbb{E}}
\def\P{\mathbb{P}}
\def\Cov{\mathrm{Cov}}
\def\Var{\mathrm{Var}}
\def\half{\frac{1}{2}}

\def\diag{\mathrm{diag}}

\def\T{^{\textsf{T}}}
\def\R{\mathbb{R}}

\def\cC{\mathcal{C}}
\def\cD{\mathcal{D}}

\def\cI{\mathcal{I}}
\def\cL{\mathcal{L}}

\def\cN{\mathcal{N}}

\def\cK{\mathcal{K}}

\newcommand{\ba}{\mathbf{a}} \newcommand{\bb}{\mathbf{b}} 
 
  \newcommand{\bh}{\mathbf{h}}

 \newcommand{\bp}{\mathbf{p}} \newcommand{\bq}{\mathbf{q}}
\newcommand{\br}{\mathbf{r}}  
\newcommand{\bu}{\mathbf{u}} \newcommand{\bv}{\mathbf{v}} \newcommand{\bw}{\mathbf{w}}
\newcommand{\bx}{\mathbf{x}}

  \newcommand{\bR}{\mathbf{R}}

\newcommand{\bell}{\boldsymbol{\ell}}
\newcommand{\bDelta}{\boldsymbol{\Delta}}
\usepackage{bm}

\newcommand{\LogPot}{{\log \Phi}}
\newcommand{\Kseg}{K_{\mathrm{seg}}}
\newcommand{\tstar}{t^{\star}}
\newcommand{\Reg}{\mathrm{Regret}}
\newcommand{\RegOCO}{\Reg^{\mathrm{OCO}}}
\newcommand{\AvgRegOCO}{\overline{\Reg}^{\mathrm{OCO}}}
\newcommand{\CP}{{\sf CP}}
\DeclareMathOperator{\TimeErr}{TimeErr}
\DeclareMathOperator{\MixErr}{MixErr}
\newcommand{\phiExp}{\phi_{\mbox{\tiny exp}}}

\newcommand{\phiNH}{\phi_{\mbox{\tiny NH}}}
\newcommand{\PhiNH}{\Phi_{\mbox{\tiny NH}}}

\title{A second order regret bound for NormalHedge}

\author{%
  Yoav Freund\thanks{Authors in alphabetical order.} \\
  UC San Diego \\
  \texttt{yfreund@ucsd.edu} \\
  \And
  Nicholas J. A. Harvey\\
  University of British Columbia \\
  \texttt{nickhar@cs.ubc.ca} \\
  \And
    Victor S. Portella \\
  Federal University of ABC\\
  \texttt{v.portella@ufabc.edu.br} \\
  \And
   Yabing Qi \\
  University of British Columbia \\
  \texttt{ yabingqi@cs.ubc.ca} \\
  \And
Yu-Xiang Wang \\
 UC San Diego \\
  \texttt{yuxiangw@ucsd.edu} \\
}

\begin{document}

\maketitle

\begin{abstract}
We consider the problem of prediction with expert advice for ``easy''  sequences.
We show that a variant of NormalHedge enjoys a second-order $\epsilon$-quantile regret bound of
$O\big(\sqrt{V_T \log(V_T/\epsilon)}\big)$
when $V_T > \log N$, where $V_T$ is the cumulative second moment of the
instantaneous per-expert regret, averaged with respect to a natural
distribution $\bq_t$ determined by the algorithm.
The algorithm is motivated by a continuous time limit using Stochastic Differential Equations. The discrete time analysis uses self-concordance techniques.
\end{abstract}

\section{Introduction}\label{sec:intro}

Prediction with expert advice is a classic model in online learning, for
which numerous generalizations have been extensively studied
\citep{CesaBianchiLugosi06}.
Bounds involving the variance of the losses have been known for two
decades \citep{CesaBianchiMansourStoltz07}, and bounds that compare to a
quantile of the experts have been known for over fifteen years
\citep{chaudhuri2009parameter}.
Ten years ago, \citet{freund2016secondorder} posed an open question of
designing an algorithm that simultaneously achieves quantile- and
variance-based bounds, and conjectured that NormalHedge
\citep{chaudhuri2009parameter} achieves such a bound.
The main contribution of this work is a positive answer to a
mild relaxation of this question:
a variant of NormalHedge does indeed enjoy a regret bound that is
simultaneously quantile- and variance-based, where the variance is
weighted by an algorithm-dependent distribution $\bq_t$ (essentially
the \emph{curvature} of the potential).
In contrast, if the variance is weighted in the standard way (often called ``variance over actions''), then the question was shown to have a \emph{negative} answer by \citet{marinov2021pareto}.
Our results do not contradict that lower bound; see the discussion
after Theorem~\ref{thm:main} and Appendix~\ref{sec:further_discussion_MZ}.

\paragraph{Deriving the curvature distribution $\bq_t$ from the continuous-time limit.}
Since we use a non-standard variance measure, it is crucial to understand why it is meaningful. 
We show that $\bq_t$ arises from studying the time increments of a potential-based player in the \emph{continuous-time limit}~\cite{freund2009method,harvey2024continuous}. 
In Section~\ref{sec:continuous_time}, we consider a large family of potential functions $\phi(x_i,t)$
that satisfy mild conditions, together with an algorithm that ensures that the total potential remains constant. 
Using a PDE characterization, we show that for potential functions satisfying the backward heat
equation $\partial_t \phi = -\tfrac12 \partial_{xx}\phi$, the time-increment that the learner is
forced to advance exactly equals $\Delta t = \E_{i\sim \bq}[(\Delta x_i)^2]$, where
$q_i \propto \partial_{xx}\phi(x_i,t)$.
In other words, $\bq_t$ is the exact distribution under which the instantaneous time-increment matches an instantaneous variance.
Working by analogy, we use this same distribution for the discrete-time NormalHedge and obtain a similar regret bound.

\paragraph{Continuous time and discretization error.}
Adopting the continuous-time viewpoint as the starting point of the
analysis yields expository advantages.
The main \emph{conceptual} contributions of the paper are the continuous-time arguments of
Section~\ref{sec:continuous_time}, which fit on a single page, and require only the
chain rule and the property that the total potential $\Phi(\bx,t)$ stays constant.
The main \emph{technical} contributions of the paper are the discrete-time arguments in Section~\ref{sec:quantization}.
A key question is: hpow
much does the discrete update at finite step-size $B$ deviate from the
continuous identity $\Delta t = \E_{i\sim \bq}[(\Delta x_i)^2]$?
We call this gap the \emph{discretization error}. 
We bound it via a new notion of \emph{local generalized
self-concordance} for the log-total-potential, applied to a
line-segment between consecutive iterates.
For Hedge the discretization error is~$0$, while for
NormalHedge the discretization error has negligible impact on the overall regret.
Self-concordance has a long history in continuous optimization
\citep{nesterov1994interior,sun2019generalized}, statistics
\citep{bach2010self,ostrovskii2021finite}, and online learning
\citep{abernethy2009competing,pmlr-v119-bilodeau20a}; but, these use the global definition of self-concordance.
To analyze the log-total-potential of NormalHedge, our local notion is essential.

   \paragraph{Related work.}
   Many existing algorithms adapt to easy sequences with ``small
   variance''
   \citep{gaillard2014second,KoolenVanErven15,derooij2014follow}.
   Others enjoy quantile regret bounds
   \citep{chaudhuri2009parameter,chernov2010prediction,luo2014drifting},
   but depend explicitly on the iteration number $T$.
   \citet{LuoSchapire15} proved quantile regret bounds with no $T$
   dependence, but their work only adapts to a first-order ``absolute
   deviation''.
   \citet{KoolenVanErven15} were the first to prove a second-order quantile
   bound depending on ``variance over time'' (their algorithm Squint uses
   a separate $V_T^{(i)}$ per expert).
   The only other existing result with a ``variance over action''-type
   quantity is \citet{negrea2021minimax}, whose stated bound has a
   suboptimal $\log(1/\epsilon)$ dependence.
   \citet{marinov2021pareto} proved that an adaptive second-order quantile regret
   $\sqrt{V_T \log(1/\epsilon)}$ is impossible for the standard 
   ``variance over actions''.
   Their construction does not cover our results because (1) our $V_T$ is
   weighted by the curvature distribution $\bq_t$ rather than the action
   distribution $\bp_t$, and (2) our bound is only $\sqrt{V_T\log(1/\epsilon)}$
   when $V_T > \log N$, a regime their construction does not cover.
   Although Squint and NormalHedge.BH are algorithmically similar, their
   performance can differ substantially on some sequences  since $V_{T}^{(i)}$ can be much smaller than our $V_T$; we provide
   numerical comparisons in \blue{Appendix~\ref{sec:numerical_experiments}}.
   A more detailed review of the literature is deferred to
   Appendix~\ref{sec:discussion}.
   
   While revising this manuscript, we became aware of the concurrent work of
   \citet{Luo26SquintVariant}, which introduced a Squint variant with a single
   shared $V_T$ and the same curvature weighting as we derive from the BHE. \citet{Luo26SquintVariant} used a different potential function that satisfies a version of BHE. Among other differences, their bound has a $\log \log T$ dependence, but not the $\log N$ ``burn-in'' period that we have, or the additional dependence on $\log V_T$ in the bound for our NormalHedge variant. Furthermore, the definition of $V_T$ in the regret bound uses the second-derivative of their potential at the \emph{one-step-ahead} iterate, rather than at the current iterate. 
   In \citet{Luo26SquintVariant}'s algorithm, the variance weights $q_t$ depend on the updated variance $V_t$, which itself satisfies $V_t=V_{t-1}+\sum_i q_{t,i}r_{t,i}^2$; this circular dependence is resolved by a line search. 
   In our algorithm, $q_t$ is computed from the previous round's cumulative regrets, so the variance increment can be computed directly once the current round's losses are observed.

\section{Problem setup}\label{sec:setup}

\noindent\textbf{Symbols and notation.}
We use standard probability and linear-algebra notation.
Boldface letters (e.g., $\bp,\bx,\bell$) denote vectors in $\R^N$, and
plain letters (e.g., $x,t,N,T$) denote scalars; $(\cdot)^\Tp$ is
transpose and $[N]=\{1,\dots,N\}$.
The inner product is $\langle \ba,\bb\rangle=\sum_{i=1}^N a_ib_i$, and
$\|\cdot\|_p$ is the $\ell_p$ norm (default $\|\cdot\|=\|\cdot\|_2$).
For random variables we write $i\sim\bp$ with $\bp\in\Delta^{N-1}$
(the probability simplex).
For a differentiable $f:\R^d\to\R$, $\nabla f$ and $\nabla^2 f$ are the
gradient and Hessian; $\nabla^2 f(\bx)[\ba,\bb]=\ba^\Tp\nabla^2 f(\bx)\bb$.
For block variables, $\nabla_{\bx}f$ and $\nabla^2_{\bx,\bx}f$ are
partial derivatives with respect to $\bx$.

\paragraph{Learning from expert advice.}
Consider the following variation of the game of prediction with expert
advice~\citep{CesaBianchiLugosi06}.
Given a number of experts $N$ and a constant $B>0$, for each iteration
$j=1,2,3,\ldots$:
\textbf{(1)} the player chooses weights $\bp_j\in\Delta^{N-1}$;
\textbf{(2)} nature reveals a loss vector $\bell_j\in\R^N$ satisfying
$\max_{i,i'\in[N]}|\ell_{j,i}-\ell_{j,i'}|\le B$;
\textbf{(3)} the player incurs loss $\langle \bp_j,\bell_j\rangle$.
The only difference from~\citet{CesaBianchiLugosi06} is that the loss
range is $[-B,B]$ rather than $[-1,1]$. This rescaling lets us consider
the limit $B\to 0$, which is the starting point of the analysis.

The regret vector at step $j$ is
$\bx_j=\sum_{k=1}^j(\langle\bp_k,\bell_k\rangle\bm 1-\bell_k)\in\R^N$.
We focus on the $\epsilon$-quantile regret (against the
$\lfloor N\epsilon\rfloor$-th best expert),
\begin{equation*}
\Reg_\epsilon(T) := x_{T,(\lfloor N\epsilon\rfloor)},
\qquad
\text{where $x_{T,(i)}$ is the $\lfloor i\rfloor$-th largest coordinate of
$\bx_T$.}
\end{equation*}

\section{Variance as a local-time measure in continuous-time}
\label{sec:continuous_time}

Our algorithm and its analysis are motivated by previous work on PDE-based
approaches to prediction with expert advice~\citep{zhu2014two,drenska2020prediction} and online linear
optimization~\citep{zhang2022pde}. These works study online learning from a dynamic-programming or control-theoretic
viewpoint, where one seeks a potential/value function whose value does
not increase under the interaction between the learner and the
adversary. This perspective naturally leads to PDEs such as the backwards heat equation. Related work has also
interpreted these dynamics through stochastic calculus and
continuous-time limits~\citep{freund2009method,
  harvey2023optimal,harvey2024continuous,ZhangYCP24}

 A rigorous analysis of the continuous-time limit is beyond the scope of
this paper.
However, we will describe an informal analysis that already provides insight on two key aspects of our approach: the requirement that the
potential function obeys the backwards heat equation, and the weighting
of the variance using the second derivative of the potential function.

One can view the learner as maintaining a total (separable) potential
$\Phi(\bx,t) = \sum_{j = 1}^N \phi(t,x_j)$ which can be controlled by the choice of probabilities on experts $\bp$ and by advancing an internal time variable $t$.
At each round, after the adversary reveals a regret increment $\Delta \bx$ (which depends on $\bp$),
the learner seeks to increase $t$ by the smallest amount $\Delta t > 0$
such that the total potential does not increase, i.e., $\Phi(x + \Delta \bx, t + \Delta t) \leq \Phi(\bx,t)$. 

In this section we consider the {\em diffusion (scaling) limit},
meaning that $\|\Delta \bx\|_\infty \leq B$, and $\Delta t \leq B^2$
while $B \to 0$. In this
regime, higher-order terms become negligible and the dynamics are
governed by the first- and second-order terms of the Taylor expansion. In the diffusion
limit, the smallest such time increment is determined by the quadratic
term in the Taylor expansion, which is why the natural variance
quantity is weighted by the curvature of the potential.
This scaling was used in nearly all previous work on continuous-time analyses of online learning
(either explicitly or via the use of stochastic calculus formulation), and one can find some discussion on this scaling by \citet[\S 1.2.3]{zhu2014two} and by~\citet{drenska2020prediction}.

More precisely, a Taylor expansion under this scaling yields (omitting low order terms)
\begin{equation}\label{eq:taylor_example}
\Phi(\bx + \Delta \bx, t + \Delta t) - \Phi(\bx, t)
\approx
\langle \nabla_x \Phi(\bx,t), \Delta \bx \rangle
+ \partial_t \Phi(\bx,t) \Delta t
+ \frac{1}{2} (\Delta \bx)\T \nabla^2_x \Phi(\bx,t) \Delta \bx.
\end{equation}
To eliminate the
first-order term in \eqref{eq:taylor_example}, it is natural to choose
the learner's probabilities proportional to the gradient of the
potential:
\[
p_i \propto \partial_{x_i} \Phi(\bx,t).
\]
With this choice, the linear term
$\langle \nabla_x \Phi(\bx,t), \Delta \bx \rangle$ vanishes.

As we are operating in the diffusion scaling limit, the terms
$\partial_t \Phi \cdot \Delta t$ and
$\partial_{x_i x_i} \Phi \cdot (\Delta x_i)^2$ are of the same order
of magnitude, while higher-order terms are negligible in the limit.
Our discrete-time analysis essentially controls these higher-order
terms via self-concordance.

\paragraph{The role of the Backwards Heat Equation.}
Suppose now that the potential $\Phi$ is separable,
i.e.,
$\Phi(\bx,t)=\sum_{i=1}^N \phi(x_i,t)$. Then the quadratic term in
\eqref{eq:taylor_example} simplifies to $
\tfrac{1}{2}\sum_{i=1}^N \partial_{x_i x_i}\Phi(\bx,t)(\Delta x_i)^2
$.
One can think of this term as the increase in potential caused by the
variance in the regret increments. Thus, to keep the potential
constant it suffices to choose $\Delta t$ large enough such that
\begin{equation}
   \label{eq:cont_time_1}
\partial_t \Phi(\bx,t)\Delta t
+
\frac12
\sum_{i=1}^N
\partial_{x_i x_i}\Phi(\bx,t)(\Delta x_i)^2
= 0.
\end{equation}
A condition on $\phi$ that significantly simplifies the above equation when trying to find a value of $\Delta t$ is the 
backwards heat equation, that is,
\begin{equation}\label{eq:bhe}
\partial_t \phi(y,t) = -\tfrac12 \,\partial_{yy}\phi(y,t)
\qquad\text{(\bf backward heat equation, BHE)}.
\end{equation}
In words, introducing the separable ansatz on $\Phi$ naturally leads to requiring $\phi$ to satisfy the BHE.

\paragraph{The natural definition of the cumulative variance $V_T$ and of $\bq$.}
Substituting $\partial_t\Phi(\bx,t)=-\tfrac12 \sum_{i}
\partial_{xx}\phi(x_i,t)$ into~\eqref{eq:cont_time_1} and
demanding equality to leading order, we obtain
\begin{equation}\label{eq:dt_continuous}
\Delta t \;=\; \frac{\sum_{i=1}^N \partial_{xx}\phi(x_i,t)\,(\Delta x_i)^2}{\sum_{i=1}^N \partial_{xx}\phi(x_i,t)}
\;=\; \E_{i\sim\bq}\bigl[(\Delta x_i)^2\bigr],
\qquad
q_i\propto \partial_{xx}\phi(x_i,t).
\end{equation}
Thus, in the continuous-time limit, $\Delta t$ \emph{equals} the
expected squared instantaneous regret under the curvature distribution
$\bq$.
Notice: there is no
freedom in the definition of $\bq$! It is the unique distribution
under which the local time-increment $\Delta t$ is an instantaneous
variance. Summing the equality \eqref{eq:dt_continuous} over $T$ iterations\footnote{More formally, the diffusion limit requires $T$ to scale on the order of $B^{-2}$ as $B \to 0$, and the sum can be thought of as approximating an integral under suitable regularity. We omit this bookkeeping here since this section is mostly heuristic.} gives
the cumulative variance
$V_T := \sum_{j=1}^T \E_{i\sim\bq_j}[(\Delta x_{j,i})^2]$ and a clean
identity $t_T = t_0 + V_T$. In the next section we will look at how a constant potential translates to regret bounds.

\section{The Constant-Potential algorithm in discrete time}\label{sec:cp_algorithm}

We now describe the discrete-time algorithm based on the continuous-time heuristic from Section~\ref{sec:continuous_time}. We begin by formally defining the class of potentials we consider, which we call \emph{good potentials}.

\begin{definition}\label{def:good_potential}
Let $\cD=\R$ or $\cD=[x_0,\infty)$.
Let $\Pi_{\cD}$ denote projection onto $\cD$.
We say that
$\phi:\R\times[0,\infty)\to\R$ is a \textbf{good potential function}
on $\cD$ if it is jointly strictly convex, three-times differentiable,
and satisfies
(1)~$\partial_x\phi\ge 0$;
(2)~$\partial_t\phi\le 0$;
(3)~$\lim_{t\to\infty}\phi(x,t)=\inf_{x,t}\phi(x,t)\ge 0$;
(4)~$\phi(\Pi_\cD(\bx),t)\le \phi(\bx,t)$ for all $\bx\in\R$;
(5)~the backward heat equation \eqref{eq:bhe}.
\end{definition}
Conditions (1)--(3) are regularity conditions used by the
continuous-time argument; condition (4) is a projection condition
needed only because in discrete time the iterate $\bx_j$ may step
outside $\cD$ between rounds; condition (5) is the BHE inspired by the continuous-time analysis discussed in Section~\ref{sec:continuous_time}.

Each good potential induces a hedge algorithm \CP$(\phi,\cD)$
(Figure~\ref{fig:CP}) that we call \emph{Constant Potential}.
The algorithm is a discrete-time realization of the continuous-time
algorithm we discusssed earlier. In particular, line~5 finds the $\Delta t$ that keeps constant the total potential
 $\Phi(\bm x,t) = \sum_{i=1}^N \phi(x_i,t)$.
Furthermore, line~6 records the variance update under the curvature
distribution $\bq$ identified in \eqref{eq:dt_continuous}.

\begin{figure}[tbh]
\fbox{\parbox{\textwidth}{ 
\CP($\phi,\cD$) algorithm

\vspace{0.4em}

Initialize $t=t_0$, $x_{0,i}=0\;\forall i\in[N]$, $V_0=0$.
Let $\widetilde{\bx}_0\leftarrow\Pi_{\cD}(\bx_0)$.

\vspace{0.4em}

For $j=1,\ldots,T$:
\begin{enumerate}
\item Play probability $\bp_j$ with
$p_{j,i}\propto \partial_x\phi(\widetilde{x}_{j-1,i},t_{j-1})$ for $i\in[N]$.
\item Nature reveals losses $\ell_{j,i}$ for every expert $i\in[N]$.
\item Algorithm incurs instantaneous regret
$\Delta x_{j,i}\coloneqq\E_{k\sim\bp_j}[\ell_{j,k}]-\ell_{j,i}$ for $i\in[N]$.
\item Update $\bx_j\leftarrow\bx_{j-1}+\Delta\bx_j$;
$\widetilde{\bx}_j\leftarrow\Pi_{\cD}(\bx_j)$.
\item Find $\Delta t_j\ge 0$ such that
$\Phi(\widetilde{\bx}_j,t_{j-1}+\Delta t_j)=\Phi(\widetilde{\bx}_{j-1},t_{j-1})$;
set $t_j\leftarrow t_{j-1}+\Delta t_j$.
\item Set $V_j\leftarrow V_{j-1}+\E_{i\sim\bq_j}[(\Delta x_{j,i})^2]$
where $q_{j,i}\propto\partial_{xx}\phi(\widetilde{x}_{j-1,i},t_{j-1})$.
\end{enumerate}
Return $V_T$.
}}
\vspace{-0.5em}
\caption{The Constant-Potential algorithm.}\label{fig:CP}
\end{figure}

\paragraph{Efficient computation for $\Delta t$.} Note that
$\Delta t$ can be chosen as a function of $x,t,\Delta x$ ---
information available to the algorithm. Since $\Phi$ is
monotonically decreasing with $t$ (Property (2) of Definition~\ref{def:good_potential}), $\Delta t$ can be obtained efficiently using bisection.

\paragraph{From potentials to regret bounds with examples.} In Section~\ref{sec:continuous_time} we discussed how to advance time $t$ (in the continuous-time setting) in a way to guarantee constant total potential throughout the game. However, the regret bounds obtained by this approach naturally depend on the choice of $\phi$. This is explained by the next lemma, whose proof is standard and deferred to Appendix~\ref{app:log_total_potential_properties}.

\begin{lemma}[Generic regret bound template]\label{lem:general_bound}
Assume $\phi,\cD$ satisfy Definition~\ref{def:good_potential}.
Let $t$ be the time variable that \CP\ encounters at any iteration $j$.
Then
\begin{equation}\label{eqn:general-bound}
\Reg_\epsilon\le x\quad\text{where $x$ is the unique solution of}\quad
(\epsilon N)\cdot\phi(x,t)=\Phi(\Pi_{\cD}(\bm 0),t_0).
\end{equation}
\end{lemma}

We consider two potentials satisfying Definition~\ref{def:good_potential}: the exponential potential \citep{chernov2010prediction} with $\eta > 0$, and a variant of NormalHedge \citep{chaudhuri2009parameter} modified by a factor $t^{-1/2}$ to satisfy the backward heat equation:
\begin{align} 
  \label{eqn:exp-pot} \phiExp(x,t) &= \exp(\sqrt{2}\eta\,x - \eta^2 t), & \cD &= \R; \\
  \label{eqn:normalhedge-pot} \phiNH(x,t) &= t^{-1/2}\exp(x^2 / 2t), & \cD &= [0,\infty).
\end{align}
Applying \eqref{eqn:general-bound} yields the following regret bounds:

\begin{lemma}[Exponential potential] \label{ex:hedge}
For \eqref{eqn:exp-pot} with $t_0 = 0$, the $\epsilon$-quantile regret satisfies
$$\mathrm{Regret}_{\epsilon}\leq \frac{\eta t}{\sqrt{2}} + \frac{\log( 1/\epsilon)}{\sqrt{2}\eta}.$$
\end{lemma}
\begin{proof}
    Solving $N\epsilon \phiExp(x,t) = \Phi(\bm 0, 0) = N$ for $x$ implies $\sqrt{2}\eta x - \eta^2 t = \log(1/\epsilon)$. The result follows by Lemma~\ref{lem:general_bound}.
\end{proof}

\begin{lemma}[Normal potential]\label{ex:normalhedge}
For \eqref{eqn:normalhedge-pot} with $t_0>0$, the $\epsilon$-quantile regret satisfies
$$\mathrm{Regret}_{\epsilon} \leq \sqrt{t (\log (t/t_0) + 2\log(1/\epsilon))}.$$
\end{lemma}
\begin{proof}
     Solving $N\epsilon \phiNH(x,t) = \Phi(\bm 0, t_0) = N t_0^{-1/2}$ for $x$ implies $x^2/2t - \frac{1}{2}\log t = \log(t_0^{-1/2}) - \log(1/\epsilon)$. Rearranging for $x$ yields the bound via Lemma~\ref{lem:general_bound}.
\end{proof}

\noindent\textbf{Remarks.} The bounds above are \emph{iteration-free} (e.g., invariant to \emph{vacuous rounds} where $\ell_{j,i}=\ell_{j,i'}$ for all $i,i'$, that increase $T$ but not the regret) and hold uniformly for all $\epsilon$. We refer to the algorithm in Lemma~\ref{ex:normalhedge} as \textbf{NormalHedge.BH}; unlike the potential of \citet{chaudhuri2009parameter}, it satisfies the \emph{backward heat equation}, has no hyperparameter, and \emph{adapts} to the accumulated time parameter $t$.

\paragraph{Squint-type potentials with arbitrary priors.}
Since the BHE is linear, for any (positive) measure $\pi$ on $\R$ the
\emph{Squint potential}
\begin{equation}\label{eq:squint_potential}
\phi_{\pi}(x,t) \;\coloneqq\; \int_{\R}\exp(\sqrt 2 \eta\,x - \eta^2 t)\;\diff\pi(\eta)
\end{equation}
is again a BHE solution (\citet{KoolenVanErven15} introduced this
construction with a half-normal $\pi$).
If $\pi$ is supported on $\R_{\ge 0}$ then $\phi_\pi$ inherits
$\partial_x\phi_\pi\ge 0$ on $\R$, and $\phi_\pi\ge 0$ holds whenever
$\pi$ is a nonnegative measure---\emph{whether $\pi$ is finite or
not}.
Thus \eqref{eq:squint_potential} defines a good potential for any
$\sigma$-finite nonnegative $\pi$, including improper priors.

The original Squint of \citet{KoolenVanErven15} does not fit our
continuous-time analysis as written, since it uses a separate
per-expert variance $V_T^{(i)}=\sum_j r_{j,i}^2$ rather than a single
shared $V_T$.  Concurrent work by \citet{Luo26SquintVariant} considered a shared $V_T$, with a slight variant of the potential $\int_{0}^{1/2}\frac{\exp(\sqrt 2 \eta\,x - \eta^2 t) - 1}{\eta}\;\diff \eta$
which satisfies BHE, but not the third property of Definition~\ref{def:good_potential}.  Finally, it might be interesting to readers that the NormalHedge potential can be viewed as a Squint potential.

\paragraph{The improper prior $\pi(\diff\eta)=\diff\eta$ on $\R$ recovers NormalHedge.}
Although NormalHedge.BH was presented in (ii) as a stand-alone
potential, it falls out of \eqref{eq:squint_potential} as the
\emph{improper-flat-prior} limit:
\begin{equation*}
\int_{-\infty}^{\infty} e^{\sqrt 2 \eta x - \eta^2 t}\,\diff\eta
\;=\; \sqrt{\pi/t}\;e^{x^2/(2t)}
\;=\;\sqrt\pi\,\phiNH(x,t).
\end{equation*}
This identity is the continuous-time reason why parameter-free
NormalHedge can be regarded as a Squint algorithm with the
``noninformative'' Lebesgue prior over learning rates.

In the remainder of the paper, we focus on analyzing the \CP{} with \eqref{eqn:exp-pot} and \eqref{eqn:normalhedge-pot}. However, we believe the \CP{} family is a much larger design space that merits further investigation.

\section{Main results}\label{sec:results}

While in continuous-time we argued we would have the clean identity $t_T = t_0+V_T$,  in discrete time we can only guarantee
$t_T \le t_0 + c\cdot V_T$, where $c$ is some constant that depends on $\phi$ and comes from the discretization error analysis (Section~\ref{sec:quantization}).
The following three theorems summarize the regret consequences.

\begin{theorem}[Exponential weights]\label{thm:exponential-weights}
\CP($\phiExp,\R$) with parameter $\eta>0$ returns $V_T$ such that
$$\Reg_\epsilon(T)\le \frac{e^{2\sqrt 2\eta B}\eta}{\sqrt 2}V_T
+ \frac{\log(1/\epsilon)}{\sqrt 2\eta}\quad\text{for any $0<\epsilon\le 1$.}$$
Choosing $\eta\asymp\min\{\sqrt{\log(1/\epsilon)/V_T},\,1/B\}$ yields
$\Reg_\epsilon(T)=O\big(B\log(1/\epsilon)+\sqrt{V_T\log(1/\epsilon)}\big)$.
\end{theorem}

\begin{theorem}[NormalHedge.BH]\label{thm:main}
Let $t_0=\max\{512 e^2 B^2\log N,1\}$.
\CP($\phiNH,\R_+$) returns $V_T$ such that, for all $\epsilon\in(0,1)$
and $T\ge 1$,
$$\Reg_\epsilon(T)\le \sqrt{(t_0+2V_T)\big(\log(t_0+2V_T)+2\log(1/\epsilon)\big)}.$$
\end{theorem}

\paragraph{Relation to the COLT'16 open problem and impossibility results.}
Theorem~\ref{thm:main} gives a simultaneous quantile/second-order regret guarantee for NormalHedge under a curvature-based variance notion. Compared with the COLT'16 open problem
\citep{freund2016secondorder}, our statement includes a mild burn-in term and uses a variance $V_T$ defined via the curvature distribution $\bq$ rather than the action distribution $\bp$. Consequently, the lower-bound construction of
\citet{marinov2021pareto} does not apply verbatim to our setting (see Appendix~\ref{sec:further_discussion_MZ}). As discussed in Section~\ref{sec:openproblems}, subsequent to the conference submission we obtained a sharper version of our result that removes the additional $(\log V_T$) dependence, yielding a bound of order $\sqrt{(V_T+\log N)\log(1/\epsilon)}$. \blue{Moreover, Haipeng Luo communicated personally to us a separate direct construction showing that a simultaneous quantile regret bound of this form
is impossible when $V_T$ is defined using the learner's action distribution, i.e., when
\(\bq_t=\bp_t\). That is, such second-order regret bounds are indeed impossible with a variance measure that uses the learner's distribution \(\bp_t\). See Appendix~\ref{sec:further_discussion_MZ} for further discussion; we plan to give a more detailed treatment in the journal version of this work.}

\paragraph{When is $\bq$ different from $\bp$?}
For exponential weights, $\bp=\bq$. For NormalHedge.BH,
\begin{align*}
p_i&\propto \partial_x\phi(x_i,t) = \frac{x_i}{t}\phi(x_i,t)
   \propto \sqrt{\frac{x_i^2}{t}}\phi(x_i,t),\\
q_i&\propto \partial^2_{xx}\phi(x_i,t)=\Big(\frac{1}{t}+\frac{x_i^2}{t^2}\Big)\phi(x_i,t)
   \propto \Big(1+\frac{x_i^2}{t}\Big)\phi(x_i,t).
\end{align*}
For all $i$, $\sqrt{x_i^2/t}=\tilde O(1)$, so for the leading experts
$p_i$ and $q_i$ are proportional. The interesting difference is
between experts whose regret $x_i$ is positive (where $p_i\propto
\sqrt{x_i^2/t}\phi$) and those whose clipped regret is zero (where
$p_i=0$ but $q_i$ remains positive). The continuous-time derivation
of Section~\ref{sec:continuous_time} explains this asymmetry:
$\bq$ is determined by the local \emph{curvature} of $\Phi$, which is
nonzero even where $\Phi$ has zero gradient. \blue{See Corollary~\ref{cor:sparsity} for an improved $V_T$ that captures the ``sparsity'' of active set, and Appendix~\ref{sec:variance_examples} for a numerical comparison of $V_T^p$ and  $V_T^q$.}

\begin{theorem}[Lower bound]\label{thm:lower-bound}
Consider a ``random walk'' adversary where each loss is chosen
independently as $\ell_{j,i}\in\{-\sigma_j,\sigma_j\}$ with equal
probability and $\sigma_j\le B/2,\sum\sigma_j^2\to\infty$.
For any algorithm and $0<\epsilon<1/40$, there exist $N=\Theta(1/\epsilon^2)$, 
$T(\epsilon)$ and a realization of the loss sequence such that
$$\frac{\Reg_\epsilon(T)}{\sqrt{\sum_{j=1}^T\sigma_j^2}}
\;\ge\;
\sqrt{2\log(1/\epsilon)}-6.$$
\end{theorem}

The above is similar to \citep[Theorem 1]{negrea2021minimax},
but more general in that it allows arbitrary time-dependent sequence of $\sigma_j$.
Moreover, our proof (in Appendix~\ref{sec:lbappendix}) is longer, but 
self-contained. The next corollary demonstrates that NormalHedge.BH is therefore adaptively optimal up to a multiplicative
$\log V_T$ factor.
\begin{corollary}[Algorithm-dependent lower bound]
For any algorithm that outputs $\bp_{1:T},\bq_{1:T}$ and any $0<\epsilon<1/40$, there exists a problem sequence such that (note that both the LHS and RHS depend on the algorithm)
$$\Reg_\epsilon(T)\ge\sqrt{V_T \log(1/\epsilon) /2}-6\sqrt{V_T/4}.$$
\end{corollary}
\begin{proof}
The result follows since $V_T\le \sum_{j=1}^T(2\sigma_j)^2$.
\end{proof}

\section{Controlling the discretization error in CP}\label{sec:quantization}

To prove a regret bound for \CP, we need to bound the growth of
$t_j$. In Section~\ref{sec:continuous_time} we argued that $\Delta t_j =\E_{i \sim \bq_j}[(\Delta x_{j,i})^2]$ in each iteration works in ``continuous-time'' or, more precisely,  ignoring discretization errors. In this section we will show that
$\Delta t_j = O(\E_{i \sim \bq_j}[(\Delta x_{j,i})^2])$
in each iteration for both the exponential potential and for NormalHedge.BH. This
requires analyzing the total potential function with perturbed inputs and
calculating how large $\Delta t_j$ needs to be to guarantee
$\Phi(\widetilde{\bx}_j, t_{j-1} + \Delta t_j) \leq \Phi(\widetilde{\bx}_{j-1}, t_{j-1})$.

Recall the following insight from the asymptotic analysis of \CP:
due to the choice of $p_{j,i}\propto \partial_x\phi(\widetilde{x}_{j-1,i},t_{j-1})$,
we have $\langle \nabla_{\bx} \Phi(\widetilde{\bx}_{j-1},t_{j-1}), \Delta \bx_j \rangle = 0$.
This is a property that we need for the discretized analysis too.

However, we first need to deal with a subtle problem. The difference between
consecutive projected iterates is
$\widetilde{\Delta \bx_j}\coloneqq\widetilde{\bx}_j-\widetilde{\bx}_{j-1}$.
Due to the projection onto $\cD$, this need not equal $\Delta\bx_j$, so
$\langle\nabla_{\bx}\Phi(\widetilde{\bx}_{j-1},t_{j-1}),\widetilde{\Delta\bx_j}\rangle$
need not vanish. To handle this, let $\breve{\Delta\bx_j}$ be any shift
from $\widetilde{\bx}_{j-1}$ satisfying
$\Pi_{\cD}(\widetilde{\bx}_{j-1}+\breve{\Delta\bx_j})=\widetilde{\bx}_j$ (e.g., $\breve{\Delta\bx_j}=\bx_j-\widetilde{\bx}_{j-1}$) . The next lemma bounds the algorithmic choice of $\Delta t_j$
by the unprojected time increment $\breve{\Delta t_j}$ defined by
$\Phi(\widetilde{\bx}_{j-1}+\breve{\Delta\bx_j},t_{j-1}+\breve{\Delta t_j})
=\Phi(\widetilde{\bx}_{j-1},t_{j-1})$.
Unlike line~5 of \CP{}, this equation evaluates the potential before projection.
This reduction allows us to choose $\breve{\Delta\bx_j}$ so that
$\langle\nabla_{\bx}\Phi(\widetilde{\bx}_{j-1},t_{j-1}),\breve{\Delta\bx_j}\rangle=0$,
preserving the orthogonality needed for the analysis
We defer the proof of the next lemma to Appendix~\ref{app:fixing_projections}.

\begin{lemma}\label{lem:reduction_delta_t}
    Let $\phi$ be a good potential. Let $\widetilde{\bx} \in \cD$,  $\breve{\Delta \bx}\in \R^N$. 
    Define $\widetilde{\Delta \bx} := \Pi_\cD(\widetilde{\bx} + \breve{\Delta \bx}) - \widetilde{\bx}$. 
    We define $\Delta t$ and $\breve{\Delta t}$ by
    \begin{align*}
    \text{($\Delta t$ definition)}\qquad
    &\Phi( \widetilde{\bx} +\widetilde{\Delta \bx}, t + \Delta t) = \Phi(\widetilde{\bx}  , t) \\
    \text{($\breve{\Delta t}$ definition)}\qquad
    &\Phi( \widetilde{\bx} + \breve{\Delta \bx}, t + \breve{\Delta t}) = \Phi(\widetilde{\bx} , t).
    \end{align*}
    Then 
    $
    \Delta t \leq \breve{\Delta t}.
    $
\end{lemma}

\paragraph{Simplifying notation.} To avoid notational clutter, for the remainder of this section (and all of the deferred proofs in the appendix), we will take $(\bx, t)$ and $(\bx',t')$ to be, respectively, consecutive iterates $(\widetilde{\bx}_{j-1}, t_{j-1})$ and %
$(\widetilde{\bx}_{j-1}+\Delta \bx_j, t_{j} + \Delta t_j)$
of \CP$(\phi, \cD)$ with a good potential function $\phi$ clear from context. Moreover, throughout we 
shall denote by $\Phi$ the total potential $\Phi(t,\bx) = \sum_{i = 1}^N \phi(t,x_i)$.

First,  we prove in Appendix~\ref{app:log_total_potential_properties} that $\Delta t_j$ as in line 5 of \CP~exists and is non-negative.

\begin{lemma}[Existence and nonnegativity of $\Delta t$]\label{lem:nonegativity}
There exists a unique \(\Delta t > 0\) such that $\Phi(\bm x + \Delta \bm x,  t + \Delta t) = \Phi(\bm x, t)$.
\end{lemma}
The proof (given in Appendix~\ref{app:log_total_potential_properties}) uses the fact that the first-order term w.r.t. $\Delta \bx$ vanishes,
and several properties of a good potential function, including that $\phi$ is strictly

\subsection{Log-potential and the explicit formula for $\Delta t$}\label{sec:log_potential}

Working with $\log\Phi$ rather than $\Phi$ is convenient because the
log-total-potential satisfies generalized self-concordance.
The next lemma formalizes the discrete counterpart of the
continuous-time identity \eqref{eq:dt_continuous}.
The proof is deferred to Appendix~\ref{app:log_total_potential_properties}.

\begin{lemma}[Second-order expansion of log-total-potential]\label{lem:taylor_theorem_log_total_potential}\label{lem:taylor_log_potential_for_chosen_Delta_t}
Define $\bDelta\coloneqq(\Delta\bx,\Delta t)$.
There is $(\bar\bx,\bar t)$ on the segment from $(\bx,t)$ to $(\bx',t')$
such that
\begin{equation*}
\log\Phi(\bx',t')-\log\Phi(\bx,t)
= -\frac{\sum_{i=1}^N \partial_{xx}\phi(x_i,t)\,\Delta t}{2\Phi(\bx,t)}
+\tfrac12\nabla^2\log\Phi(\bar\bx,\bar t)[\bDelta,\bDelta].
\end{equation*}
In particular, if $\log\Phi(\bx',t')=\log\Phi(\bx,t)$, then
\begin{equation}\label{eq:generic_Delta_t_bound}
\Delta t = \frac{\Phi(\bx,t)\,\nabla^2\log\Phi(\bar\bx,\bar t)[\bDelta,\bDelta]}{\sum_{i=1}^N\partial_{xx}\phi(x_i,t)}.
\end{equation}
\end{lemma}

The next lemma extracts from \eqref{eq:generic_Delta_t_bound} the
discrete-time analog of the continuous identity
\eqref{eq:dt_continuous}, with the Hessian
perturbation factor $C$ and the finite-step errors
\blue{$\MixErr_\Phi$ (mixed space--time)} and
$\TimeErr_\Phi$ (time--time).
The proof is in Appendix~\ref{app:log_total_potential_properties}.

\begin{lemma}\label{lemma:delta_t_discretization_error}
Suppose there is $C\ge 1$ such that for every $(\bar\bx,\bar t)$ on the
segment between $(\bx,t)$ and $(\bx',t')$,
$\nabla^2\log\Phi(\bar\bx,\bar t)[\bDelta,\bDelta]\preceq
C\cdot\nabla^2\log\Phi(\bx,t)[\bDelta,\bDelta]$.
Then
\begin{equation*}
	\begin{aligned}
		\Delta t \le {}& C\cdot\left( \E_{i\sim\bq}[(\Delta x_i)^2]
			+ \Delta t\left(\blue{\MixErr_\Phi}+\TimeErr_\Phi\right)
			\right),
	\end{aligned}
\end{equation*}

where the \textbf{discretization errors} are
\begin{equation*}
\begin{aligned}
\MixErr_\Phi(\bx,t;\Delta\bx)&\coloneqq
-\frac{\sum_{i=1}^N\partial_{xxx}\phi(x_i,t)\,\Delta x_i}
     {\sum_{i=1}^N\partial_{xx}\phi(x_i,t)},\\
\TimeErr_\Phi(\bx,t;\Delta t)&\coloneqq\Delta t\left(
\frac{\sum_{i=1}^N\frac{\partial^4}{\partial x^4}\phi(x_i,t)}
     {4\sum_{i=1}^N\partial_{xx}\phi(x_i,t)}
- \frac{1}{4\Phi(\bx,t)}\sum_{i=1}^N\partial_{xx}\phi(x_i,t)\right).
\end{aligned}
\end{equation*}
We omit the arguments of $\MixErr_\Phi$ and $\TimeErr_\Phi$ whenever
the base point \((\bx, t)\) and increments \((\Delta \bx, \Delta t)\) are clear from context.
\end{lemma}
Under $\Delta\bx=O(B)$ and $\Delta t=O(B^2)$, at a fixed base point \((\bx, t)\)
we roughly have $\MixErr_\Phi=O(B)$ and $\TimeErr_\Phi=O(B^2)$.
Thus, in the diffusion limit where only terms of order at least $\Omega(B^2)$ are retained, the contributions of 
$\Delta t\,\MixErr_\Phi=O(B^3)$ and
$\Delta t\,\TimeErr_\Phi=O(B^4)$ vanish.  In the discrete case, however, the contributions of these terms depend on $\phi$, as the next lemmas show.

\begin{lemma}\label{lem:star_bound_for_exp_potential}
For $\phi=\phiExp$, \blue{$\MixErr_\Phi=0$} and $\TimeErr_\Phi=0$.
\end{lemma}
\begin{proof}
One may verify that $(\sum_i\partial_{xx}\phi(x_i,t))^2/\Phi(\bx,t)
= \sum_i\partial^4_{xxxx}\phi(x_i,t) = 4\eta^4\Phi(\bx,t)$.
\blue{Also, $\partial_{xxx}\phiExp=2\eta^2\partial_x\phiExp$, so $\MixErr_\Phi=0$ by
$\sum_i\partial_x\phi(x_i,t)\Delta x_i=0$.}
\end{proof}

\begin{lemma}[\protect{\citet{bhatia2000better}-style bound}]\label{lem:bahia-davis}
For $\phi=\phiNH$,
$$\TimeErr_\Phi\le \frac{\Delta t}{4t}\big(\max_i\frac{ x_i^2}{t} + 4\big),  \blue{
	\quad\quad
	|\MixErr_\Phi|\le B\sqrt{\frac{1}{t}\max_i\frac{x_i^2}{t}}.
}$$

\end{lemma}
The proof is in Appendix~\ref{sec:proof_using_bahia_davis}. Thus, 
both discretization errors are $0$ for Hedge, while for NormalHedge.BH, since we show
$\max_i x_i^2/t=\tilde O(1)$ once $t$ is large enough,  the
time error is bounded above by $O(\Delta t\,(\log t)/t)$ and the
mixed space-time error by $O(B\sqrt{(\log t)/t})$.

\subsection{Local generalized self-concordance and the Hessian perturbation factor $C$}

To bound the Hessian perturbation factor $C$ we extend the standard
notion of self-concordance to a local, arbitrary-norm version.

\begin{definition}[Local generalized self-concordance]\label{def:gsc-multi-general}
Let $\mathrm{dom}(f)\subset\R^p$ be open. A $C^3$ function $f$
satisfies $(M,2)$-generalized self-concordance in $\|\cdot\|_*$
on $\cC\subset\mathrm{dom}(f)$ if for all $\bx\in\cC$,
$\bu,\bv\in\R^p$,
\begin{equation}\label{eq:gsc-multi-general}
\big|\nabla^3 f(\bx)[\bu,\bu,\bv]\big|
\le M\|\bv\|_*\,\bu^\Tp\nabla^2 f(\bx)\bu.
\end{equation}
\end{definition}
This generalizes \citep{bach2010self,sun2019generalized} in two ways:
(1) any norm $\|\cdot\|_*$, not just Euclidean; and (2) the bound is
required only on a local set $\cC$, not the whole domain.
The standard global self-concordance is the case
$\cC=\mathrm{dom}(f),\,\|\cdot\|_*=\|\cdot\|_2$.

\begin{lemma}[Sandwich bound for $(M,2)$-GSC]\label{lem:sd-order-nu2-star}
If $f$ is $(M,2)$-GSC in norm $\|\cdot\|_*$ on a convex set $\cC$,
then for any $x,y\in\cC$,
\begin{equation}\label{eq:hess-sd-order-star}
e^{-M\|y-x\|_*}\,\nabla^2 f(x)
\preceq \nabla^2 f(y) \preceq
e^{M\|y-x\|_*}\,\nabla^2 f(x).
\end{equation}
\end{lemma}
The proof is in Appendix~\ref{sec:proof_of_sandwich}.

\subsection{Proof of Theorem~\ref{thm:exponential-weights}: global GSC of $\log\Phi$ for Hedge}\label{sec:proof_hedge}

For $\phi=\phiExp$ we have
$\log\Phi(\bx,t)=-\eta^2 t+\Psi(\bx)$ with
$\Psi(\bx)=\log\sum_i\exp(\sqrt 2\eta x_i)$.

\begin{proposition}\label{prop:gsc_hedge}
$\Psi$ is $(2\sqrt 2\eta,2)$-globally GSC in $\|\cdot\|_\infty$, and
$\log\Phi$ is $(2\sqrt 2\eta,2)$-globally GSC in the norm
$\|(\bx,t)\|_*=\|\bx\|_\infty$.
\end{proposition}
The proof, by direct computation on the log-sum-exp function, is in
Appendix~\ref{sec:detailed_proof}.
The 2-GSC of $\log\Phi$ follows since it decomposes into $\Psi(\bx)$
plus a linear term in~$t$, which is GSC

\begin{proof}[Proof of Theorem~\ref{thm:exponential-weights}]
By Proposition~\ref{prop:gsc_hedge} and
Lemma~\ref{lem:sd-order-nu2-star}, the Hessian perturbation factor in
Lemma~\ref{lemma:delta_t_discretization_error} is
$C=e^{2\sqrt 2\eta\|\bx-\bar\bx\|_\infty}\le e^{2\sqrt 2\eta B}$.
Combined with Lemma~\ref{lem:star_bound_for_exp_potential},
$\Delta t\le e^{2\sqrt 2\eta B}\E_{i\sim\bq}[(\Delta x_i)^2]$, with
$q_i\propto\partial_{xx}\phiExp$ giving $\bq=\bp$.
Summing over $j$ and applying Lemma~\ref{lem:reduction_delta_t} yields
$t=t_0+\sum_j\Delta t_j\le e^{2\sqrt 2\eta B}V_T$.
The result follows from $\Reg_\epsilon\le \tfrac{\eta t}{\sqrt 2}+\tfrac{\log(1/\epsilon)}{\sqrt 2\eta}$.
\end{proof}

\subsection{Proof sketch of Theorem~\ref{thm:main}: local GSC for NormalHedge}\label{sec:proof_normalhedge}

For NormalHedge.BH the Hessian perturbation factor cannot be made
small \emph{globally}: $\log\Phi$ is not globally GSC.
We work locally on the line-segment
\begin{equation}\label{eq:segment}
(\bx(s),t(s))=(\bx,t)+s(\Delta\bx,\Delta t),\qquad s\in[0,1],
\end{equation}
and define
\begin{equation}\label{eq:params}
\tstar:=\min_{s\in[0,1]}t(s)=t,\qquad
\Kseg:=\sup_{s\in[0,1]}\,\max_{1\le i\le N}\,\frac{x_i(s)^2}{t(s)}.
\end{equation}
The above definitions imply the bounds
\begin{equation}\label{eq:seg-controls}
\frac{|x_i(s)|}{t(s)}\le \sqrt{\frac{\Kseg}{\tstar}}\quad\forall i,s,
\qquad\text{and}\qquad
\frac{|\Delta t|}{t(s)}\le \frac{|\Delta t|}{\tstar}\quad\forall s.
\end{equation}

\begin{proposition}[Local 2-GSC for NormalHedge]\label{prop:localGSC_normalhedge}
$\log\Phi$ for $\phiNH$ satisfies local generalized self-concordance
with $M=1$, $\nu=2$, and $\cC$ the line-segment \eqref{eq:segment}, in
the norm
$\|(\bar\bx,\bar t)\|_* = A_x(\tstar,\Kseg)\|\bar\bx\|_\infty
+ A_t(\tstar,\Kseg)|\bar t|$, where
\begin{equation}\label{eq:AxAt}
A_x(\tstar,\Kseg) = \frac{8\sqrt{\Kseg\vee 1}}{\sqrt{\tstar}},
\quad
A_t(\tstar,\Kseg) = \frac{16(\Kseg\vee 1)}{\tstar}.
\end{equation}
\end{proposition}
The proof is in Appendix~\ref{sec:gsc_normal_hedge_proof}.
Define $\Lambda := A_x(\tstar,\Kseg)\|\Delta\bx\|_\infty + A_t(\tstar,\Kseg)|\Delta t|$.
By Lemma~\ref{lem:sd-order-nu2-star} we obtain the following sandwich:
\begin{corollary}\label{cor:sandwich_normal}
For any $(\bar\bx,\bar t)$ in the segment \eqref{eq:segment},
$e^{-\Lambda}\nabla^2\log\Phi(\bx,t)\preceq
\nabla^2\log\Phi(\bar\bx,\bar t)\preceq e^{+\Lambda}\nabla^2\log\Phi(\bx,t)$.
\end{corollary}

The Hessian perturbation factor in Lemma~\ref{lemma:delta_t_discretization_error} is therefore $C=e^\Lambda$, but
$\Lambda$ (and \(\TimeErr_{\Phi}\)) depend on $\Delta t$, which is what we are trying to bound in the first phase.
We close the loop in two phases.

\paragraph{Phase 1 (crude bound).}
For the remainder of this proof, we slightly overload notation and let \(\Delta t>0\) denote an arbitrary candidate time increment; the true equalizing increment is no larger than any $\Delta t$  for which \(\Phi(\bx',t+\Delta t)\leq\Phi(x,t)\). Then
Lemma~\ref{lemma:delta_t_discretization_error}  yields 
$\Phi(\bx',t+\Delta t)\le\Phi(\bx,t)$ whenever
$\Delta t \ge e^\Lambda(B^2+ \Delta t(\MixErr_\Phi+\TimeErr_\Phi))$. For $\Phi = \PhiNH$, Lemma~\ref{lem:bahia-davis}
gives $|\MixErr_\Phi|\le B\sqrt{K/t}$ and
$\TimeErr_\Phi\le \Delta t (K+4)/(4t)$.
Picking the candidate $\Delta t = 2eB^2$ and assuming that $t\ge\Omega(B^2 K)$ we get the bound $\Lambda \le 1$ and that the discretization errors $|\MixErr_\Phi|$ and $\TimeErr_\Phi$  are small enough. So the inequality $\Delta t \ge e^\Lambda(B^2+ \Delta t(\MixErr_\Phi+\TimeErr_\Phi))$ is satisfied by this choice of \(\Delta t\). This is summarized in the following lemma (proof deferred to Appendix~\ref{sec:proof_of_crude_bound}).
\begin{lemma}[Crude bound on $\Delta t$]\label{lem:crude_Delta_t_bound}
If $\max_i x_i^2/t\le K$, then
\blue{$\Delta t\le 2eB^2$ if $t\ge 256e^2 B^2\max\{K,1\}$.}
\end{lemma}

\paragraph{Defining $K$ and bounding $\Lambda$.}
Lemma~\ref{lem:crude_Delta_t_bound} requires
$\max_i x_i^2/t\le K$ such that $K/t$ is also small. Since $\Phi$ is constant, this
holds with $K \equiv K(t)\coloneqq \blue{\max\{1,\log(t/t_0)+2\log N\}}$:
\begin{lemma}\label{lemma:pot_regret_bound}
For $t\blue{{}\ge{}}t_0$, $\max_i x_i^2/t\le K(t)$\blue{; once the crude bound applies,}
$\Kseg\le \log(1+t/t_0)+2\log N\le 2K(t)$.
\end{lemma}
Since $K(t)$ grows only logarithmically in $t$, choosing
$t_0\ge 256e^2 B^2 K(t_0)$, i.e.\ $t_0=\Omega(B^2\log N)$, ensures the
condition of Lemma~\ref{lem:crude_Delta_t_bound} holds for all
$t\ge t_0$ and yields $\Lambda\le 0.414$ (proof in
Appendix~\ref{sec:proof_circular}).

\paragraph{Phase 2 (refined bound).} We can use the crude bound $\Delta t \le 2eB^2$ to bound the discreization errors. Namely,
since $\TimeErr_\Phi\le 5K(t)\Delta t/(4t)$ and
\blue{$|\MixErr_\Phi|\le B\sqrt{K(t)/t}$} (from
Lemmas~\ref{lem:bahia-davis} and \ref{lemma:pot_regret_bound}), simple manipulations show that \(\MixErr_{\Phi} + \TimeErr_{\Phi} \leq 5/64e\) since \(t\) is large enough (by the choice of \(t_0\)). We can plug this error on the discretization error back on  Lemma~\ref{lemma:delta_t_discretization_error}, yielding
\begin{equation}\label{eq:Delta_t_handy_expression}
\begin{aligned}
\Delta t &\le e^\Lambda\Big(\E_{i\sim\bq}[(\Delta x_i)^2]
+ \blue{B\sqrt{K(t)/t}\,\Delta t} + 5 \Delta t / 64 e \Big).
\end{aligned}
\end{equation}
An algebraic manipulation (worked out in Appendix~\ref{sec:proof_circular})
gives $\Delta t\le 2\E_{i\sim\bq}[(\Delta x_i)^2]$.
Summing over iterations and applying Lemma~\ref{lem:reduction_delta_t}
yields $t=t_0+\sum_j\Delta t_j\le t_0+2V_T$. Plugging into the
NormalHedge.BH regret template completes the proof of
Theorem~\ref{thm:main}.

\subsection{Sparsity and re-entry}\label{sec:sparsity}

The regret bounds in Theorem~\ref{thm:main} are obtained by instantiating Lemma~\ref{lem:reduction_delta_t} with the choice $\breve{\Delta\bx} = \Delta \bx$. For $\phiNH$ on $\cD=\R_+$, the bound can be improved by choosing $\breve{\Delta\bx}$ more cleverly, exploiting the fact that experts with negative regret need not contribute to the variance measure used in the regret. In other words, we can improve the bound from Theorem~\ref{thm:main} to exploit any sparsity in the regret vectors. In this s  ection we formalize this observation.

First, note that an expert with nonpositive regret contributes nothing to the variance \emph{up until} the moment its regret becomes
positive. To capture this, for all rounds $j$ and coordinates $i$, define
\begin{equation}\label{eq:active_reentry_increment}
\Delta x^{\mathrm{ar}}_{j,i}\coloneqq
\begin{cases}
\Delta x_{j,i}, & x_{j-1,i}>0,\\
[x_{j-1,i}+\Delta x_{j,i}]_+, & x_{j-1,i}\le0,
\end{cases}
\qquad
V_T^{q,\mathrm{ar}}\coloneqq
\sum_{j=1}^T\E_{i\sim\bq_j}
\bigl[(\Delta x^{\mathrm{ar}}_{j,i})^2\bigr].
\end{equation}
Here $\bq_j$ is the same distribution as in \CP. The following corollary shows we can refine the regret bound from Theorem~\ref{thm:main} using this improved variance measure. We defer the proof to Appendix~\ref{apx:sparse-regret}

\begin{corollary}\label{cor:sparsity}
Under the assumptions of Theorem~\ref{thm:main}, its guarantee holds
with $V_T^{q,\mathrm{ar}}$ in place of $V_T$:
\[
\Reg_\epsilon(T)\le
\sqrt{(t_0+2V_T^{q,\mathrm{ar}})
\bigl(\log(t_0+2V_T^{q,\mathrm{ar}})+2\log(1/\epsilon)\bigr)}.
\]
Moreover, $V_T^{q,\mathrm{ar}}\le V_T$.
\end{corollary}

\paragraph{Comparison with action-weighted variance.}
Unlike exponential weights, for which \(\bq=\bp\), NormalHedge.BH has different action and curvature distributions, and the corresponding variance measures $V_T^{p}$ and \(V_T^q\), respectively, are generally incomparable. Indeed, \(\bp_j\) places no mass on inactive experts, whereas \(\bq_j\) may place substantial mass on them; \(V_T^{q,\mathrm{ar}}\) removes their contribution while they remain inactive, without renormalizing the remaining \(\bq\)-mass. Consequently, in some cases \(V_T^{q,\mathrm{ar}}\) can be  smaller than both \(V_T^p\) and \(V_T^q\). For example, on the sequence \(\bell_j=(0,1,\ldots,1)\) for all \(j\), once the learner concentrates on the first expert, \(V_T^{q,\mathrm{ar}}\) stops growing, whereas \(V_T^q\) continues to grow linearly. \blue{Appendix~\ref{sec:numerical_experiments} complements a regret comparison between NormalHedge.BH and Squint with direct comparisons of the three second-order measures. These examples also show that repeated re-entry can make $V_T^{q,\mathrm{ar}}$ larger than $V_T^p$.}

\section{Subsequent progress and open problems}
\label{sec:openproblems}

\begin{enumerate}
\item
Our regret upper bound in Theorem~\ref{thm:main} for NormalHedge.BH has the form
\[ \sqrt{(V_T+\log N)(\log(V_T + \log N)+\log(1/\epsilon))},\]
ignoring constants.
In contrast, our lower bound in Theorem~\ref{thm:lower-bound} has the form
$ \sqrt{V_T\log(1/\epsilon)}$,
again ignoring constants.
Since submitting our work to this conference, we have improved both of these to
$ \sqrt{(V_T+\log N)\log(1/\epsilon)} $.
Thus, we now have upper and lower bounds that match up to constant factors.
A detailed explanation of these results will appear in the journal version
of this work.

\item One interpretation of the SDE-based analysis is that we are
designing an algorithm with minimal regret against the Itô process
\citep{freund2009method,greenstreet2022efficient,harvey2023optimal,harvey2024continuous}.
Yet the upper bounds hold in the worst case for arbitrary discrete-time sequences,
with no stochastic assumption.
A better understanding of the relationship between the Itô process and
the game of prediction with expert advice would be welcome.

\end{enumerate}

\begin{ack}
YW was partially supported by NSF Award DMS \#2134214.
VSP acknowledges that this study was financed, in part, by the São
Paulo Research Foundation (FAPESP), Brazil, process numbers
2024/09381-1 and 2023/03167-5.
We thank Francesco Orabona and Jeffrey Negrea for helpful discussion
related to their results \citep{negrea2021minimax}, and Julian Zimmert
for clarifying details of the lower bound \citep{marinov2021pareto}. We thank Haipeng Luo for sharing an unpublished note concerning impossibility results for quantile regret bounds based on the learner's action variance.

\end{ack}

\bibliographystyle{plainnat}

\bibliography{refs_ito_normal_hedge}

\appendix

\input{appendix.tex}

\end{document}

%% file: appendix.tex
\section{Further discussion on the related work}\label{sec:discussion}

The problem of online learning and predictions have its origin in classical work in statistics and information-theory \citep{robbins1951asymptotically,blackwell1956analog,cover1966behavior}. 
The ``expert advice'' problem was first proposed by
\citep{littlestone1988learning} and has since been studied extensively
\citep{littlestone1994weighted,vovk1995game,CesaBianchiLugosi06}. It
is known that the minimax regret with $N$ experts and $T$ rounds is of $\Theta(\sqrt{T \ln N})$. %
As a result, subsequent work focuses
on \emph{adaptive} regret bounds that replaces $T$ or $\log N$ with
sequence-dependent or comparator-dependent quantities that could be
much smaller than the worst-case when the sequences are ``easy''.

\paragraph{First and second order regret bounds.} In these
adaptive regret bounds, $T$ is typically replaced with a quantity that
depends on the loss vector $\ell_j$ or the instantaneous regret vector
$\Delta \bm x_j= \langle p_j, \bm \ell_j\rangle \bm{1} - \bm
\ell_j$. $\Delta \bm x_j$ is also known as ``excess loss (vector)''
\citep{gaillard2014second,KoolenVanErven15}.  Broadly speaking, first
order (``small loss'') regret bounds uses
$L_T(i) =\sum_{j=1}^T|\ell_{j,i}|$ and can be improved to
$S_T(i) := \sum_{j=1}^T|\Delta x_{j,i}|$ when comparing with expert
$i$.  Second order (``small variance'') regret bound uses
$V_T(i) = \sum_{j=1}^T (\Delta x_{j,i})^2$ or
$V_T^{\mathrm{avg},q} = \sum_{j=1}^T \E_{i\sim \bm q_j}[(\Delta
x_{j,i})^2]$.  When $q=p$ where $p$ is the weights being selected by
the algorithm, we recover
$V_T^{\mathrm{avg}}= \sum_{j=1}^T \Var_{i\sim p_j}[\ell_{j,i}]$ due to
\cite{CesaBianchiMansourStoltz07}.
From our discussion in Section~\ref{sec:continuous_time}, we believe it is more natural for $\bm q_j$ to be defined by the second derivative (curvature) of the potential function rather than tied up to the action weights $\bm p_j$ (typically the first derivative in potential-based algorithms). They happen to be the same for exponential weights, but in almost all examples of adaptive algorithms with second-order regret, $\bm q_j\neq \bm p_j$. The list include \citep{negrea2021minimax,KoolenVanErven15,Luo26SquintVariant} and this paper. 

Second order bounds are stronger than first order bounds.  Among second order bounds, it has been a focal point of debate whether the comparator-dependent   $V_T(i)$ or the algorithm-dependent $V_T$ (such as $V_T^{\mathrm{avg}}$ or $V_T^{\mathrm{avg},q}$) are more natural. Our results (Theorem~\ref{thm:exponential-weights} and~\ref{thm:main}) are of the latter type.

\paragraph{Quantile regret bounds.} Along a different dimension, $\log N$ can often be replaced
with $\log(1/\epsilon)$ where regret is measured relative to the 
$(\epsilon N)^{th}$ best expert
\citep{Luo26SquintVariant, chaudhuri2009parameter,LuoSchapire15}. This is referred to as
``quantile regrets'' and is a strict-generalization if the algorithm
enjoys such quantile bounds for all $0<\epsilon<1$.  Quantile regrets
are closely related to PAC-Bayesian-style regret bounds for algorithms
that take a prior $\pi$ as input. When we are comparing to the
average performance of a subset $\cI\subset [N]$ of the experts over
the prior, the corresponding regret bound often replaces $T$ with
$L^{\pi}_T(\cI)$, $S^\pi_T(\cI)$ or $V^\pi_T(\cI)$ and also replaces
$\log N$ with the Kullback-Liebler divergence
$\mathrm{KL}(\pi_{\cI}, \pi)$ \citep{KoolenVanErven15,LuoSchapire15}.
Here $\pi_{\cI}$ denotes the restriction of $\pi$ to set $\cI$, and
$V^\pi_T(\cI) = \sum_{i\in\cI} \pi_{\cI}(i) V_T(i)$.  The special case
when $\pi= \mathrm{Uniform}([N])$ and $\cI$ chosen to be the
best-performing $\epsilon$-fraction of the experts is particularly
relevant because in it provides a $\epsilon$-quantile regret bound
parameterized by $\log(1/\epsilon)$, $S_T(\epsilon)$ or
$V_T(\epsilon)$.

	\begin{table}[tbh]
	\centering
    \resizebox{\textwidth}{!}{
	\begin{tabular}{@{}p{1.55in}lcc@{}}
		\toprule
		Work & Cumulative variance & Regret bound & Adaptive Minimax? \\
		\midrule
		Hedge \citep{CesaBianchiMansourStoltz07} &
		$V_T^{\mathrm{avg}}:= \sum_{j=1}^T\E_{i\sim p_j}[(\Delta x_{j,i})^2]$ &
		$\sqrt{V_T^{\mathrm{avg}}\log N}$ &
		No \\
		Hedge \citep{CesaBianchiMansourStoltz07}&
		$V_T(i):= \sum_{j=1}^T(\Delta x_{j,i})^2$ &
		$\sqrt{V_T(i)\log N}$ &
		No \\
		NormalHedge \citep{chaudhuri2009parameter} & T& $\sqrt{T \log(1/\epsilon)}$ & Yes\\
		Squint \citep{KoolenVanErven15} &
		$V_T (\epsilon):= \E_{i\sim \text{Unif}(\mathrm{Top}_\epsilon)}\left[V_T(i)\right]$ & $\sqrt{V_T(\epsilon)\log(1/\epsilon)}$ &
		Yes \\
		Ada-Normal-Hedge \citep{LuoSchapire15} &
		$S_T(\epsilon):= \E_{i\sim \text{Unif}(\mathrm{Top}_\epsilon)}[\sum_{j=1}^T |\Delta x_i|]$ &
		$\sqrt{S_T(\epsilon)\log(1/\epsilon)}$ &
		Yes \\
		AdaHedge \citep{derooij2014follow}& $V_T^{\mathrm{avg}}:= \sum_{j=1}^T\E_{i\sim p_j}[(\Delta x_{j,i})^2]$ &$\sqrt{V_T^{\mathrm{avg}}\log N}$ & Yes\\
          Adapt-ML-Prod \citep{gaillard2014second} & $V_T(i)$ &$\sqrt{V_T(i)\log N}$& Yes\\
          FTRL \citep{negrea2021minimax} &$V_T^{\mathrm{avg},q}:= \sum_{j=1}^T\E_{i\sim \bm q_j}[(\Delta x_{j,i})^2]$&$\log(1/\epsilon) \sqrt{ V_T^{\mathrm{avg},q}}$& Yes\\
          Squint \citep{Luo26SquintVariant} &$V_T^{\mathrm{avg},q}:= \sum_{j=1}^T\E_{i\sim \bm q_{j+1}}[(\Delta x_{j,i})^2]$&$ \sqrt{ V_T^{\mathrm{avg},q}\log(1/\epsilon)}$& Yes\\
		\midrule
		Hedge  (Theorem~\ref{thm:exponential-weights}\textbf{  this paper}) & $V_T^{\mathrm{avg}}$ & $\sqrt{V_T^{\mathrm{avg}}\log (1/\epsilon)}$&  No  \\
		NormalHedge  (Theorem~\ref{thm:main} \textbf{this paper}) & $V_T^{\mathrm{avg},q}$&   $\sqrt{V_T^{\mathrm{avg},q}\log (1/\epsilon)}$ & Yes\\
		Lower bound (Theorem~\ref{thm:lower-bound} \textbf{this paper})& $V_T^{\mathrm{avg},q}$ \text{for any $q^{1:T}$} & $\sqrt{V_T^{\mathrm{avg},q}\log (1/\epsilon)}$ & not applicable \\
		\bottomrule
	\end{tabular}
    }
	\caption{Summary of instance-dependent regret bounds.    For those that depends on priors, e.g., \cite{KoolenVanErven15}, we choose the prior to be uniform, so as to be comparable to quantile regret bounds.  For clarity, we omit multiplicative $\log(V_T)$ factors and additive $\mathrm{poly}\log(N)$ factors. Also, note that $V_T^{\mathrm{avg},q}$ in \cite{negrea2021minimax} is different from the $V_T^{\mathrm{avg},q}$ in our Theorem~\ref{thm:main} due to different choices of $q^{1:T}$, both are natural for the algorithm under consideration.
        }\label{tab:comparison}
\end{table}

\paragraph{Adaptively minimax algorithms.}  Finally, if the algorithm achieves optimal regret simultaneously for many input sequence classes or comparator classes, then we call them adaptively minimax. These algorithms are also referred to as ``parameter-free''\citep{chaudhuri2009parameter,orabona2015optimal}. In the expert-advice setting, this means that the algorithm achieves smaller regret parameterized by $L_T, V_T$ or $\log(1/\epsilon)$ simultaneously for all possible values of $L_T, V_T$, or $\epsilon$. 

Among adaptive algorithms, \citet{gaillard2014second,derooij2014follow} adapts to variants of $V_T$. \citet{chaudhuri2009parameter,luo2014drifting} adapts to $\epsilon$. \cite{LuoSchapire15} adapts to both $L_T(i)$ and $\epsilon$.

The work of \citet{KoolenVanErven15} enjoys an adaptive second-order quantile regret of $\sqrt{ V_T(\epsilon) \log(1/\epsilon)}$  where $V_T(\epsilon) := \E_{i\sim \text{Unif}(\mathrm{Top}_\epsilon)}\left[\sum_{j=1}^N (\Delta x_{j,i})^2\right]$. The main difference from ours is that their $\bm q$ is uniform over the top $\epsilon$ quantile of experts throughout while ours has a different $\bm q_j$ per iteration and our $V_T$ is decoupled from the choice of $\epsilon$.

There are two papers in the literature that also give a regret bound that is both quantile and variance based. The first is the concurrent work of \citet{Luo26SquintVariant}, which we became aware of while we were revising the present manuscript. This work introduced a Squint variant that uses a single
shared parameter $V_T$ for all actions. 
The potential function used in this work satisfies a version of the BHE, and the weighting on the actions is based on the curvature of the potential, as we do. 
However, there are some key differences.
First, their bound has a $\log \log T$ dependence, but not the $\log N$ ``burn-in'' period that we have, or the additional dependence on $\log V_T$ in the bound for our NormalHedge variant.
Furthermore, the definition of $V_T$ in the regret bound uses the second-derivative of their potential at the \emph{one-step-ahead} iterate, rather than at the current iterate.
Thus, there is a circular relationship: the increment of $V_T$ affects the one-step-ahead iterate, and hence the distribution $q_t$, which determines the increment of $V_T$. 
This issue is handled by a line search.
In our algorithm, $q_t$ is computed from the previous round's cumulative regrets, so the variance increment can be computed directly once the current round's losses are observed.

The next paper
closest to ours is \cite{negrea2021minimax}, who obtained a bound that
reads $\tilde{O}(\log(1/\epsilon)\sqrt{ V_T^{\mathrm{avg,q}}})$ for a
Follow-the-Regularized Leader (with adaptively chosen regularizer).  They wrote that the suboptimal dependence on $\log(1/\epsilon)$ might be required for their algorithm. However, a private communication with a subset of the authors of \cite{negrea2021minimax} reveals that a different variant could enjoy a $\sqrt{V_T^{\mathrm{avg,q}}\log (1/\epsilon)}$-type regret similar to what we obtained for NormalHedge.BH, although such a result was not published in the paper or any of their subsequent work.

\paragraph{Comparison between Squint and NormalHedge.BH}
Give the similarity between the algorithms of Squint \citep{Luo26SquintVariant} and NormalHedge.BH (Theorem~\ref{thm:main} \textbf{this paper}), one might wonder if their performance would also be similar. In fact, they are not. In some scenarios, NormalHedge.BH can incur substantially less regret than Squint. We provide numerical experiments performing such a comparison in \blue{Appendix~\ref{sec:numerical_experiments}}.

\paragraph{Open problem and impossibility results.}  \cite{marinov2021pareto} resolved the open problem by \cite{freund2016secondorder} on the negative by proving that there is no adaptive algorithm that achieves $O(\sqrt{ V_T^{\mathrm{avg}}\log(1/\epsilon)})$ for all $\epsilon$ for $V_T^{\mathrm{avg}}$ defined as the ``variance over actions'' in \cite{CesaBianchiMansourStoltz07} verbatim.  \citet{marinov2021pareto} poses no contradiction with our results due to the nuanced differences that we discussed in the remark after Theorem~\ref{thm:main}. This has been confirmed via a private communication with the authors of \citet{marinov2021pareto}. For completeness, we include a self-contained discussion on \citet{marinov2021pareto}'s construction in Appendix~\ref{sec:further_discussion_MZ}.

\section{\blue{Numerical experiments}}
\label{sec:numerical_experiments}
\subsection{A numerical comparison of regret between NormalHedge and Squint}
\label{subsec:comparison}

In this section, we provide a numerical example showing that the regret incurred by NormalHedge.BH can be significantly less than the regret incurred by Squint.

\textbf{Setup.}  Two experts compete on a sequence of length $T = 60$
split into four parts of lengths $(10,10,10,30)$.  In parts~1--2,
expert~1 incurs zero loss while expert~2 alternates
between losses of $0.5$ and $0$, accumulating loss at average rate $0.25$.  In
parts~3--4 the roles reverse: expert~2 incurs zero
loss and expert~1 incurs loss~$1$ on every round.  The two cumulative
trajectories are exactly tied at $j=25$, after which Expert~2 pulls
permanently ahead, so the best fixed expert in hindsight is Expert~2.
This is shown in Figure~\ref{fig:family_losses}, zoomed to the first $30$ rounds.  To investigate how each algorithm responds
to the size of the expert pool, we use either $1$ or $100$ total copies
of expert~1 (thus $N$ is either $2$ or $101$), and run NormalHedge.BH and
Squint on both versions (Figure~\ref{fig:compare_1_vs_100}).
NormalHedge.BH starts from zero regret with the theoretical initialization
$t_0=\max\{512e^2 B^2\log N,1\}$ from Theorem~\ref{thm:main}.
Here $B=1$, giving $t_0\approx2622.31$ for $N=2$ and
$t_0\approx17459.91$ for $N=101$.
For this particular loss sequence, setting $t_0=1$ produces identical
play probabilities and regret curves since, at every round, all the positive values of the regret vector are identical. 

\begin{figure}[H]
\centering
\includegraphics[width=0.62\textwidth]{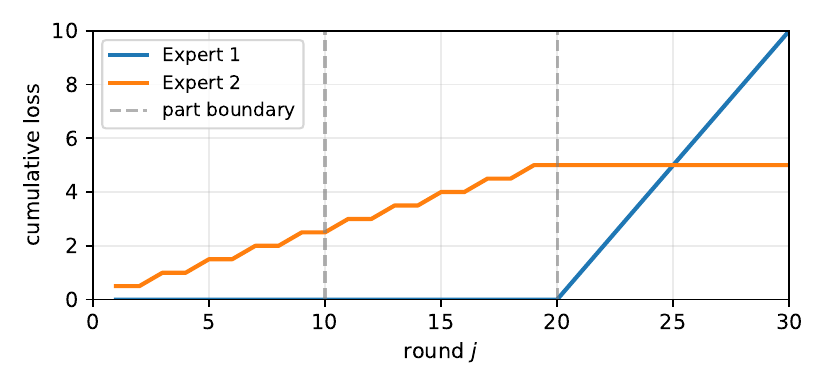}
\caption{Cumulative losses of the two experts when there is a single
copy of Expert~1, zoomed to the first $30$ rounds.  Dashed grey lines
mark the four part boundaries.  The trajectories cross at $j=25$, so
any algorithm that ``commits'' to Expert~1 in the first half pays
regret linear in the remaining horizon.}
\label{fig:family_losses}
\end{figure}

\begin{figure}[H]
\centering
\includegraphics[width=0.95\textwidth]{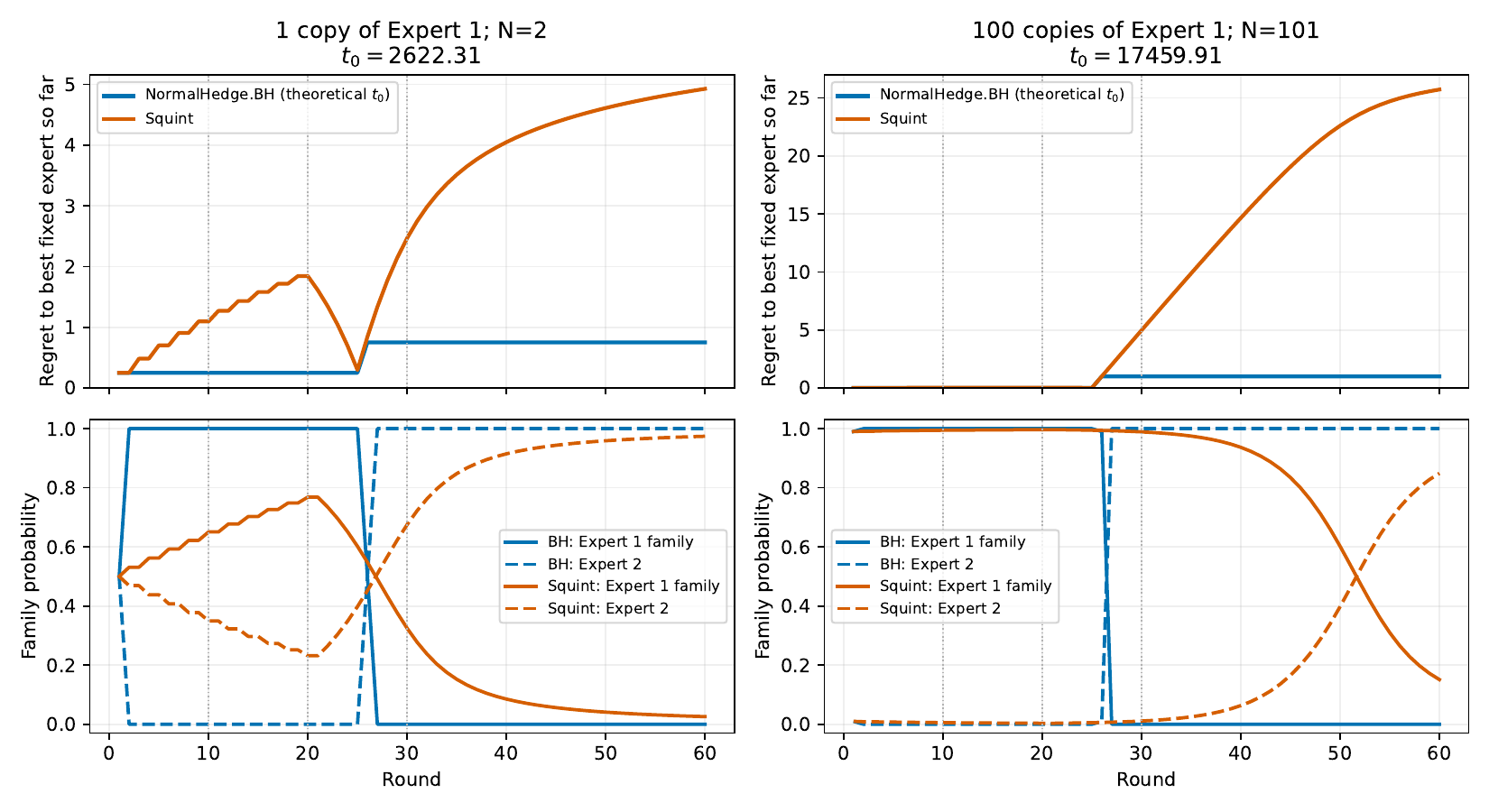}
\caption{Effect of duplicating expert~1.  Top row: regret to the best
fixed expert up to each round.  Bottom row: total weight on the expert~1 \emph{family}
(sum over duplicates) and on expert~2.  \emph{Left} ($N=2$): both
algorithms transfer weight to expert~2, with NormalHedge.BH doing so sooner than Squint.
\emph{Right} ($N=101$): NormalHedge.BH is almost unchanged while Squint spreads its uniform prior over
all $N$ experts, giving the expert~1 family $100/101$ of the prior
mass; it therefore clings to that family longer, switches later, and suffers significantly larger regret.}
\label{fig:compare_1_vs_100}
\end{figure}

This small example illustrates a situation where the regret of Squint is larger than that of NormalHedge.BH.
Squint is variance-adaptive but \emph{not}
duplication-invariant, and even moderate duplication produces a
quantitatively different regret curve.  In this example, NormalHedge.BH
plays only expert~2 from round~27 onward for both pool sizes, and its
final regret remains below~1.

\begingroup
\subsection{Comparing the three second-order measures by examples}
\label{sec:variance_examples}

In this section, we compare $V_T^p$, $V_T^q$ and $V_T^{q,\mathrm{ar}}$ for
\textbf{NormalHedge.BH}. 

Recall that $V_T^p$ is the variance-term initially conjectured in the regret bounds by \citet{freund2016secondorder}, but proven impossible by \citet{marinov2021pareto}.  $V_T^q$ and $V_T^{q,\mathrm{ar}}$ are second-order moments under a different measure $q$  that we have shown to be the \emph{right} ``variance''  measure for NormalHedge. They were stated in Theorem~\ref{thm:main} and Corollary~\ref{cor:sparsity} respectively. 

As it turns out, $V_T^p$ bound is not always stronger than the $V_T^q$ bound.  We give concrete examples in which   $V_T^q$ and $V_T^{q,\mathrm{ar}}$ can be orders of magnitude smaller than $V_T^p$.

The aim is to illustrate differences between these
measures on the \emph{same} run, not to compare algorithms or establish
typical-case performance. Every trajectory starts at zero regret and uses
the theorem's initialization $t_0=512e^2\log N$, with loss range $B=1$.
All constructions are deterministic. The examples illustrate two conclusions: $V_T^p$ and
$V_T^q$ are incomparable, we always have $V_T^{q,\mathrm{ar}}\le V_T^q$,
 while $V_T^{q,\mathrm{ar}}$ and $V_T^p$ are also incomparable.

\paragraph{Definitions and implementation.}
At the beginning of round $j$, with raw cumulative regret $\bx_{j-1}$
and internal time $t_{j-1}$, we compute both distributions before observing
the loss:
\[
p_{j,i}\propto [x_{j-1,i}]_+
 \exp\!\left(\frac{[x_{j-1,i}]_+^2}{2t_{j-1}}\right),\qquad
q_{j,i}\propto\left(1+\frac{[x_{j-1,i}]_+^2}{t_{j-1}}\right)
 \exp\!\left(\frac{[x_{j-1,i}]_+^2}{2t_{j-1}}\right).
\]
We use uniform weights when all gradient weights vanish. In particular and worth noting,
$\bq$ includes nonpositive-regret coordinates: the curvature at the
projected boundary is $\phi_{xx}(0,t)=t^{-3/2}$, not zero.
We then record
\[
v_j^p=\sum_i p_{j,i}(\Delta x_{j,i})^2,\qquad
v_j^q=\sum_i q_{j,i}(\Delta x_{j,i})^2,\qquad
v_j^{q,\mathrm{ar}}=\sum_i q_{j,i}(\Delta x_{j,i}^{\mathrm{ar}})^2,
\]
where $\Delta x_{j,i}=\langle\bp_j,\bell_j\rangle-\ell_{j,i}$ and
$\Delta x_{j,i}^{\mathrm{ar}}$ is defined in
\eqref{eq:active_reentry_increment}. Uppercase quantities are cumulative
sums. We also note that, as one may verify, the active-reentry increment \(\Delta x_{j,i}^{\mathrm{ar}}\) defined in \eqref{eq:active_reentry_increment} is neither a support mask applied to
$\Delta x_{j,i}$ nor the difference of two clipped regrets.
We recompute $p_j,q_j$ from the current state on every round and solve
for $t_j$ by bisection to keep the total potential constant.

\paragraph{Alternating specialists plus a neutral pool ( $V_T^q<V_T^p$).}
Take $N=M+2$ experts and start with
\[
\bell_1=(0,0,1,\ldots,1).
\]
The first two experts, which we call \emph{specialists}, have equal positive regret and the $M$ neutral
experts have negative regret, so $\bp_2=(1/2,1/2,0,\ldots,0)$.
Each subsequent two-round cycle, indexed by even $j$, uses
\[
\bell_j=(0,1,1/2,\ldots,1/2),\qquad
\bell_{j+1}=(1,0,a_j,\ldots,a_j),\qquad a_j=p_{j+1,1}.
\]
Here $a_j$ is computed by the algorithm after round $j$. The regret increments
are $(1/2,-1/2,0,\ldots,0)$ and
$(-(1-a_j),a_j,0,\ldots,0)$, restoring equality of the specialist
regrets. The remainder of the experts remain inactive and have zero instantaneous
regret throughout these cycles. They dilute the curvature weights of
the specialists, while $p_j$ is supported on the specialists alone.
Thus $V_T^q$ can be much smaller than $V_T^p$, while 
$V_T^{q,\mathrm{ar}}$ differs from $V_T^q$ only on the first round.

\paragraph{A permanent leader ($V_T^p<V_T^q$).}
Take $N=M+1$ and the oblivious loss sequence
\[
\bell_j=(0,1,\ldots,1)\qquad\text{on every round}.
\]
After the first round, $\bp_j=(1,0,\ldots,0)$ and
$\Delta\bx_j=(0,-1,\ldots,-1)$. Hence, $V_T^p$ stops growing,
whereas $V_T^q$ keeps charging all inferior experts. These experts
never re-enter, so $V_T^{q,\mathrm{ar}}$ also stops growing after the
first round. This is exactly the contribution that active-reentry
refinement removes.

\paragraph{Repeated re-entry ($V_T^p < V_T^{q,\mathrm{ar}}$)}
The preceding examples do not imply that active-reentry refinement $V_T^{q, \mathrm{ar}}$
always improves upon $V_T^p$. Take $N=M+1$ and start with
\[
\bell_1=(0,1,\ldots,1),\qquad \bell_2=(1,0,\ldots,0).
\]
After the first round,
$\bx_1=(M,-1,\ldots,-1)/(M+1)$ and $\bp_2=(1,0,\ldots,0)$.
Thus $v_2^p=0$, but all $M$ inactive experts re-enter with
$\Delta x_{2,i}^{\mathrm{ar}}=M/(M+1)$ for $i\ge2$.
Their contributions must be retained.

To make the \emph{same} pool re-enter repeatedly, after round 2 alternate
the following two phases. First, drive its common regret down to $-1/2$
using
\[
\bell_j=(0,c_j,\ldots,c_j),\qquad
c_j=\min\!\left\{1,\frac{x_{j-1,2}+1/2}{p_{j,1}}\right\}.
\]
The leader's regret stays positive, so $p_{j,1}>0$, and every pool expert
has increment $-p_{j,1}c_j$. The final fractional loss lands exactly at
$-1/2$. Each driving phase terminates: the leader's regret is at least
$M/(M+1)$, the pool's clipped regret is at most 1, and $t_j\ge t_0$,
so the displayed formula for $p_j$ bounds $p_{j,1}$ away from zero.
Next, reveal $\bell_j=(1,0,\ldots,0)$: the learner plays only
the leader, while the pool re-enters at regret $+1/2$. On each such round,
\[
v_j^p=0,\qquad
v_j^{q,\mathrm{ar}}=\frac14\sum_{i=2}^{M+1}q_{j,i}>0.
\]
Repeat the two phases, always updating the actual regret state and
constant-potential clock; no states are reset. This construction is
deterministic and adaptive, and all losses lie in $[0,1]$.
The intervening rounds also contribute to $V_T^p$ and are included in
all reported totals. For $M=100$ and $T=201$, the pool re-enters 60 times
(including round 2), and $V_T^{q,\mathrm{ar}}$ is about twice $V_T^p$.

\begin{figure}[htbp]
\centering
\includegraphics[width=\linewidth]{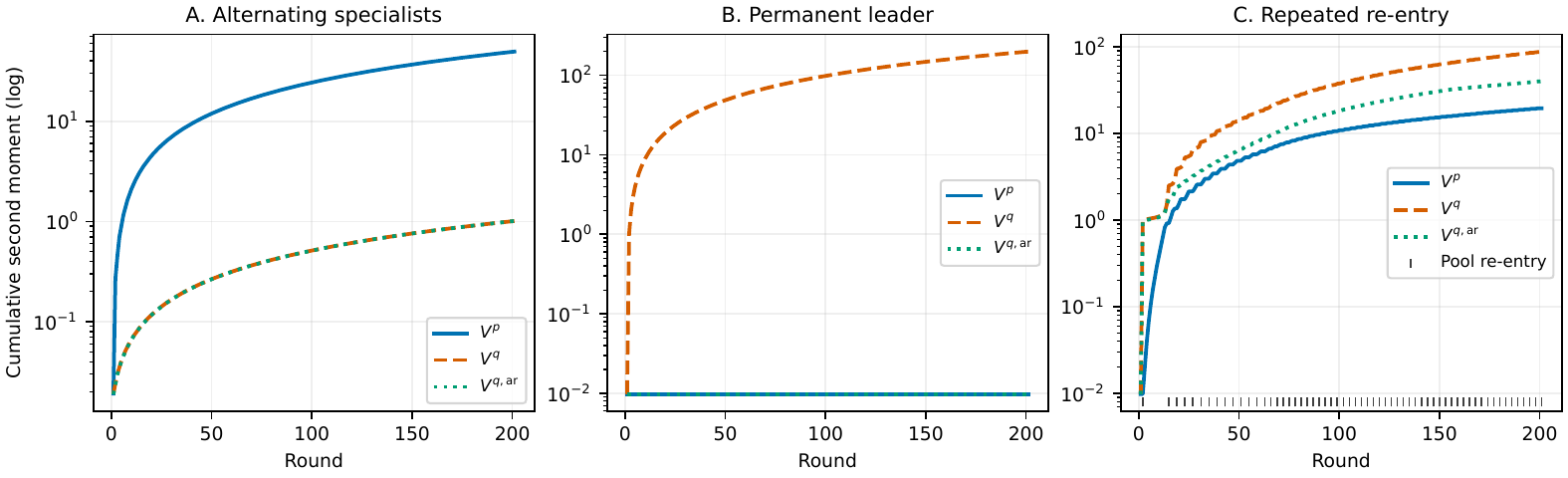}
\caption{\blue{Demomstration of incomparability of second-order measures, with $M=100$ and $T=201$,
run from zero regret using the theorem's initialization. Left:
$V_T^{q,\mathrm{ar}}\approx V_T^q\ll V_T^p$ for alternating specialists
with a neutral pool; the dotted and dashed curves nearly overlap.
Middle: $V_T^{q,\mathrm{ar}}<V_T^p\ll V_T^q$ for a permanent leader;
the dotted and solid curves nearly overlap. Right: repeated re-entry
gives $V_T^p<V_T^{q,\mathrm{ar}}<V_T^q$; short ticks mark the 60 pool
re-entry rounds. Vertical axes are logarithmic.}}
\label{fig:p_vs_q_by_examples}
\end{figure}

The final values for all three constructions, with $M=100$, are:
\begin{center}
\begin{tabular}{lrrrr}
\toprule
Construction & $T$ & $V_T^p$ & $V_T^q$ & $V_T^{q,\mathrm{ar}}$\\
\midrule
Alternating specialists & $201$ & $49.503185$ & $1.011868$ & $1.011491$\\
Permanent leader & $201$ & $0.009803$ & $198.029440$ & $0.009706$\\
Repeated re-entry & $201$ & $19.515386$ & $87.959040$ & $39.977563$\\
\bottomrule
\end{tabular}
\end{center}
The alternating-specialist separation grows with the pool size. For
$M=10,30,100,300,1000$, the alternating-specialist ratios $V_T^p/V_T^q$
are respectively $5.77,15.36,48.92,144.85,480.68$.

\paragraph{Conclusion.}
The first two constructions show that neither $V_T^p\le V_T^q$ nor
$V_T^q\le V_T^p$ holds in general. In contrast,
$V_T^{q,\mathrm{ar}}\le V_T^q$ always, since
$|\Delta x_{j,i}^{\mathrm{ar}}|\le|\Delta x_{j,i}|$ coordinatewise.
The alternating-specialist and re-entry constructions show that
$V_T^{q,\mathrm{ar}}$ can be either smaller or larger than $V_T^p$.
Thus active-reentry refinement never increases the curvature-based
quantity, but does not uniformly improve upon action variance.

\subsection{Reproducibility and implementation details}
\label{sec:experiment_reproducibility}
Both Appendices~\ref{subsec:comparison} and~\ref{sec:variance_examples}
use the specified deterministic loss sequences and evaluate mixture losses
without sampling actions. Initial play is uniform; there are no data
splits, training, hyperparameter searches or sampling error bars.
The implementations use Python 3.12, NumPy 2.3.5 and SciPy 1.18.1
(BSD-3-Clause licenses), and Matplotlib 3.11.2 (Matplotlib license);
no external datasets or pretrained models are used.
For Appendix~\ref{sec:variance_examples}, the log-domain constant-potential
solver uses tolerance $10^{-13}$ for the log potential and relative
bracket width, with independent 50-digit trajectory checks.
On an Apple M4 Pro CPU with 24\,GiB RAM, numerical runs for
Appendix~\ref{subsec:comparison} take less than a second, and regeneration
of Appendix~\ref{sec:variance_examples}, including its figure, takes a few
seconds; no GPU is needed.
The code for reproducing the simulations is available upon request and
may be released by the authors on their websites.
\endgroup

\section{Handling projection onto $\cD$}
\label{app:fixing_projections}

\begin{proof}[Proof of Lemma~\ref{lem:reduction_delta_t}]
By Property (6) in Definition~\ref{def:good_potential}, $\Phi(\Pi_\cD(\bm x),t)\leq \Phi(\bm x, t)$ for any $\bx,t$, thus
\begin{align*}
\Phi(\widetilde{\bx} +  \breve{\Delta \bx} , t + \Delta t) &\geq \Phi(\Pi_\cD(\widetilde{\bx} +  \breve{\Delta \bx}), t + \Delta t) \\
&\explain{=}{\text{Def. of }\widetilde{\Delta \bx}} \Phi(\widetilde{\bx} + \widetilde{\Delta \bx}, t + \Delta t) \explain{=}{\text{Choice of }\Delta t} \Phi(\widetilde{\bx}, t) \explain{=}{\text{Choice of }\breve{\Delta t}} \Phi(\widetilde{\bx} +  \breve{\Delta \bx} , t + \breve{\Delta t}) 
\end{align*}
This shows that $t + \Delta t \leq t + \breve{\Delta t}$, since $\Phi$ monotonically decreases as $t$ increases, due to property (4) of Definition~\ref{def:good_potential}.
Thus $\Delta t \leq \breve{\Delta t}$.
\end{proof}

\section{Properties of Log-Total-Potential of $\CP$}
\label{app:log_total_potential_properties}

\begin{proof}[Proof of Lemma~\ref{lem:nonegativity}]
By Taylor's theorem (with Lagrange-form of the remainder) and the twice differentiability of $\Phi$, there exists $(\bar{\bm x},\bar{t})$ on the line segment between $(\bm x,  t)$ and $(\bm x + \Delta \bm x,  t + \Delta t)$ such that
\begin{equation}\label{eq:taylor_theorem_on_Phi}
    \Phi(\bm x + \Delta \bm x,  t + \Delta t) = \Phi(\bm x, t) + \left\langle \nabla \Phi(\bm x, t), \begin{bmatrix}\Delta \bm x\\ \Delta t \end{bmatrix}\right\rangle + \frac{1}{2}\begin{bmatrix}\Delta \bm x\\ \Delta t\end{bmatrix}^\Tp \nabla^2 \Phi(\bar{\bm x}, \bar{t}) \begin{bmatrix}\Delta \bm x\\ \Delta t\end{bmatrix}.
\end{equation}
By the algorithmic choice of $p_i\propto \nabla_{x_i} \Phi(\bm x, t) = \frac{\partial}{\partial x}\phi(x_i,t),$ we have 
$$\left\langle \nabla_{\bm x} \Phi(\bm x, t), \Delta \bm x\right\rangle = \left\langle \bp, \Delta \bm x\right\rangle  = \sum_{i} p_i (\ell_i - \bm \ell^\Tp\bm p) = 0.$$
Plugging this into \eqref{eq:taylor_theorem_on_Phi} we get 
$$
    \Phi(\bm x + \Delta \bm x,  t + \Delta t) = \Phi(\bm x, t) + \left\langle \nabla_t \Phi(\bm x, t), \Delta t \right\rangle + \frac{1}{2}\begin{bmatrix}\Delta \bm x\\ \Delta t\end{bmatrix}^\Tp \nabla^2 \Phi(\bar{\bm x}, \bar{t}) \begin{bmatrix}\Delta \bm x\\ \Delta t\end{bmatrix}.
$$
First, observe that at $\Delta t = 0$, 
$$\Phi(\bm x + \Delta \bm x,  t + \Delta t) = \Phi(\bm x, t)  + \Delta \bm x^\Tp \nabla_{\bar{\bm x},\bar{\bm x}}^2 \Phi(\bar{\bm x}, \bar{t}) \Delta \bm x \geq \Phi(\bm x, t). $$
The inequality is strict since $\Phi$ is strictly convex, which is implied by the strict convexity of $\phi$.

Second, by the backwards heat equation assumption and the convexity in $x$
$$
\partial_t \Phi = - \frac{1}{2}\sum_i \partial_{xx} \phi(x_i,t) \leq 0,
$$
thus the total potential is monotonically decreasing (and strictly decreasing if strictly convex) in $t$ anywhere.   These two facts ensure that if the equalizing $\Delta t$ exists, it must be non-negative. 

Lastly, the existence follows from the intermediate value theorem using the continuity of $\phi$, as well as the assumptions that $\lim_{t\rightarrow \infty}\phi(x, t) = \inf_{x,t}\phi(x,t)$ and $\Phi(\bm x,t) > \inf_{x,t}\phi(x,t)$.
\end{proof}

\begin{proof}[Proof of Lemma~\ref{lem:taylor_log_potential_for_chosen_Delta_t}]
By the second order Taylor's theorem of $\log \Phi(\bm x',t')$ at $(\bx,t)$, there exists $(\bar{\bm x}, \bar{t})$ on the line segment between $(\bm x,t)$ to $(\bm x',t')$ such that
\begin{align*}	
\log \Phi(\bm x',t') - \log \Phi(\bm x,t) =&  \frac{\langle\nabla \Phi(\bm x,t), \Delta\rangle}{\Phi(\bm x,t)} +  \frac{1}{2} \nabla^2 \log \Phi(\bar{\bm x},\bar{t})[\bm\Delta,\bm\Delta]\\
	=&   \frac{\frac{\partial}{\partial t} \Phi(\bm x,t) \Delta t}{\Phi(\bm x,t)}     +\frac{1}{2} \nabla^2 \log \Phi(\bar{\bm x},\bar{t})[\bm\Delta,\bm\Delta]\\
	=&- \frac{\sum_{i=1}^N\partial_{xx} \phi(x_i,t) \Delta t}{2\Phi(\bm x,t)}    +\frac{1}{2} \nabla^2 \log \Phi(\bar{\bm x},\bar{t})[\bm\Delta,\bm\Delta] %
\end{align*}
Where the choice of $p_i$ in the algorithm ensures that $\langle\nabla_x \Phi(\bm x,t), \Delta \bm x\rangle = 0$.  The reverse heat equation allows us to write  $ \partial_t \Phi(\bm x,t) = -\frac{1}{2}\sum_{i=1}^N \partial_{xx} \phi(x_i,t) $. In particular, if $\Delta t$ is such that $\log \Phi(\bm x',t') - \log \Phi(\bm x,t) = 0$, then rearranging the above gives the 	Z expression for $\Delta t$.
\end{proof}

\begin{lemma}[Restatement of Lemma~\ref{lemma:delta_t_discretization_error}]
	Define \(\bDelta \coloneqq (\Delta \bm x, \Delta t)\) and assume there is \(C \geq 1\) such that for any $(\bar{\bx}, \bar{t})$ in the line segment between \((\bx,t)\) and \((\bx', t')\), we have
	\begin{equation}
		\label{eq:hessian_comparison_assumption}
		\nabla^2 \log\Phi(\bar{\bx},\bar t)[\bDelta, \bDelta] \preceq C \cdot \nabla^2 \log\Phi(\bx,t)[\bDelta, \bDelta].
	\end{equation}
	Then,
	\begin{equation*}
		\begin{aligned}
		\Delta t \le {}& C\cdot\left( \E_{i\sim\bq}[(\Delta x_i)^2]
			+ \Delta t\left(\blue{\MixErr_\Phi}+\TimeErr_\Phi\right)
			\right),
		\end{aligned}
	\end{equation*}
	where the \textbf{mixed space--time error} and \textbf{time--time error} are
	\begin{equation*}
		\begin{aligned}
		\MixErr_\Phi(\bx,t;\Delta\bx)&\coloneqq
		-\frac{\sum_{i=1}^N\partial_{xxx}\phi(x_i,t)\,\Delta x_i}
		     {\sum_{i=1}^N\partial_{xx}\phi(x_i,t)},\\
		\TimeErr_\Phi(\bx,t;\Delta t)&\coloneqq\Delta t\left(
		\frac{\sum_{i=1}^N\frac{\partial^4}{\partial x^4}\phi(x_i,t)}
		     {4\sum_{i=1}^N\partial_{xx}\phi(x_i,t)}
		- \frac{1}{4\Phi(\bx,t)}\sum_{i=1}^N\partial_{xx}\phi(x_i,t)\right).
		\end{aligned}
	\end{equation*} 
	We omit the arguments whenever the base point and increments are clear from context.
\end{lemma}	
\begin{proof}
		By Lemma~\ref{lem:taylor_log_potential_for_chosen_Delta_t}, there is \((\bar{\bx}, \bar{t})\) in the line segment between \((\bx,t)\) and \((\bx', t')\) such that	
		\begin{equation}
			\label{eq:delta_t_LogPot_formula}
			\Delta_t = \frac{\Phi(\bm x,t) \nabla^2 \log\Phi(\bar{x},\bar{t})[\bm \Delta, \bm \Delta]}{\sum_{i=1}^N\partial_{xx} \phi(x_i,t) }
			\stackrel{\eqref{eq:hessian_comparison_assumption}}{\leq}
			C \cdot \frac{\Phi(\bm x,t) \nabla^2 \log\Phi(\bm x,t)[\bm \Delta, \bm \Delta]}{\sum_{i=1}^N\partial_{xx} \phi(x_i,t) }.
		\end{equation}
		Moreover, note that
	\begin{align}
&\nabla^2 \log\Phi(\bm x,t)[\bm \Delta, \bm \Delta] = \frac{\nabla^2\Phi(\bm x,t)[\bm \Delta, \bm \Delta] }{\Phi(\bm x,t)} - \frac{\bm \Delta^\Tp\nabla \Phi(\bm x,t) \nabla \Phi(\bm x,t)^\Tp\bm \Delta}{\Phi(\bm x,t)^2}\nonumber\\
=& \frac{1}{\Phi(\bm x,t)}   \left( \nabla^2_{xx}\Phi(\bm x,t)[\Delta \bm x, \Delta \bm x]   + \frac{\partial^2}{\partial t^2}\Phi(\bm x,t) \Delta t^2   + \blue{2\sum_i\partial_{xt}\phi(x_i,t)\Delta x_i\Delta t} \right) \nonumber\\
		&\quad -\frac{1}{\Phi(\bm x,t)^2}   \left(  {\underbrace{\langle\nabla_x \Phi(\bm x,t), \Delta \bm x\rangle}_{=0 \text{ by our algorithm.}}}^2 +   (\frac{\partial}{\partial t} \Phi(\bm x,t) \Delta t)^2    \right) \tag{Separate $t$ and $x$ in the partial derivatives.}\nonumber\\ 
=&\frac{\nabla^2_{xx}\Phi(\bm x,t)[\Delta \bm x, \Delta \bm x] }{\Phi(\bm x,t)}
\blue{{}-\frac{\Delta t}{\Phi(\bm x,t)}\sum_i\partial_{xxx}\phi(x_i,t)\Delta x_i}\\
&+\Delta t^2 \left(\frac{\frac{\partial^2}{\partial t^2}\Phi(\bm x,t)}{\Phi(\bm x,t)} - \frac{(\frac{\partial}{\partial t}\Phi(\bm x,t))^2}{\Phi(\bm x,t)^2}\right).\nonumber
\end{align}
\blue{Here the mixed space--time error is retained: the backward heat equation gives
$2\partial_{xt}\phi=-\partial_{xxx}\phi$, whereas orthogonality holds only at
the fixed base point. Thus}
\begin{align*}
\nabla^2\log\Phi(\bm x,t)[\bm\Delta,\bm\Delta]
=& \frac{1}{\Phi(\bm x,t)}\sum_i\partial_{xx} \phi(x_i,t)\Delta x_i^2
\blue{{}-\frac{\Delta t}{\Phi(\bm x,t)}\sum_i\partial_{xxx}\phi(x_i,t)\Delta x_i}\\
&+ \Delta t^2 \left(\frac{1}{4\Phi(\bm x,t)}\sum_i \frac{\partial^4}{\partial x^4} \phi(x_i,t) -  \frac{1}{4\Phi(\bm x,t)^2} (\sum_i\partial_{xx} \phi(x_i,t))^2 \right)
\nonumber
\\
=& \frac{\sum_i\partial_{xx} \phi(x_i, t) }{\Phi(\bm x,t)}
\left(\blue{\E_{i\sim\bq}[\Delta x_i^2]}
+\Delta t\left(\blue{\MixErr_\Phi}+\TimeErr_\Phi\right)\right).
\end{align*}
Plugging the above into \eqref{eq:delta_t_LogPot_formula} concludes the proof. 
\end{proof}

\section{Simple alternative proof of Theorem~\ref{thm:exponential-weights}: Exponential weights.} \label{sec:alt_proof_hedge}
This section presents a simple and self-contained argument for bounding $\Delta t$ for the exponential potential in Example~\ref{ex:hedge}. This does not require using more abstract results about a generic potential $\phi$ such as Lemma~\ref{lem:taylor_theorem_log_total_potential},\ref{lemma:delta_t_discretization_error}, and Lemma~\ref{lem:star_bound_for_exp_potential}.

We continue from the part of Section~\ref{sec:proof_hedge} after Proposition~\ref{prop:gsc_hedge} (self-concordance of $\Psi$).

\paragraph{Taylor's Theorem on $\Psi$.}
Start by taking the difference of the log-total-potential $\LogPot$ of each step of the algorithm (recall $\LogPot(\bm x,t)=\Psi(\bm x)-\eta^2 t$):
\begin{align}
\LogPot(\bm x+\Delta \bx,t+\Delta t)- \LogPot(\bm x,t)
&=- \eta^2\Delta t + \Psi(\bm x+\Delta \bx) - \Psi(\bm x)\nonumber\\
&= - \eta^2\Delta t + \langle \nabla\Psi(\bm x),\Delta \bx\rangle 
+ \frac{1}{2} \Delta \bx^\Tp \nabla^2\Psi(\bar{x})\,\Delta \bx \label{eq:taylor_log_potential_hedge}
\end{align}
where the last step applies the second-order Taylor's theorem with $\bar{\bx} = \bx + \tau \Delta \bx$ for a $ \tau\in[0,1]$. 

By Proposition~\ref{prop:gsc_hedge} which states that $\Psi$-is $(M,2)$-global self-concordant and the sandwich formula in Lemma~\ref{lem:sd-order-nu2-star} that gives a semidefinite ordering
$$
\nabla^2\Psi(\bar{x})\preceq \exp(M\tau \|\Delta \bx\|_\infty)\,\nabla^2\Psi(\bm x)\preceq \exp(M\|\Delta \bx\|_\infty)\,\nabla^2\Psi(\bm x).
$$  
Substituting the above inequality with $M=2\sqrt{2}\eta$ into \eqref{eq:taylor_log_potential_hedge} yields 
\begin{align}
\underbrace{\LogPot(\bm x+\Delta x,\,t+\Delta t)-\LogPot(\bm x,t)}_{=0}
&\le \langle \underbrace{\nabla\Psi(\bm x),\Delta \bx\rangle}_{=0} - \eta^2\,\Delta t
+ \frac{1}{2}\,e^{2\sqrt{2}\eta B}\,\underbrace{\Delta \bx^\Tp \nabla^2\Psi(\bm x)\,\Delta \bx}_{=2\eta^2 \E_{i\sim q}[\Delta x_i^2]}. \label{eq:upper-main}
\end{align}

Observe that by Lemma~\ref{lem:gh}, $\nabla\Psi(\bm x) = \sqrt{2}\eta \bm p$ where $p$ is the algorithm's choice of the probability (thus the linear term vanishes) and $\Delta x^\Tp \nabla^2\Psi(\bm x)\,\Delta x = 2\eta^2 \E_{i\sim q}[\Delta x_i^2] =2\eta^2 \Var_{i\sim  \bp} [\ell_i].$ Moreover, by our algorithmic choice of $\Delta t$, the LHS is $0$.

Move $\Delta t$ to the left hand side, we obtain
\begin{equation}\label{eq:Delta_t_bound_hedge}
    \Delta t \leq e^{2\sqrt{2}\eta B}\Var_{i\sim  \bp} [\ell_i].
\end{equation}

\section{Detailed Proofs for Generalized Self-Concordance}
\subsection{Proof of Lemma~\ref{lem:sd-order-nu2-star}: the ``Sandwich'' Formula}\label{sec:proof_of_sandwich}

\begin{proof}[Proof of Lemma~\ref{lem:sd-order-nu2-star}]
	Let $\Delta \coloneqq y-x$ and define $x_\lambda\coloneqq x+\lambda \Delta$ for $\lambda \in[0,1]$. By convexity of $\cC$, we have $x_t\in\cC$.
	Fix any $u\in\mathbb{R}^p$ and define the scalar function
	\[
	q(t):=u^\Tp \nabla^2 f(x_t)u \ \ge 0.
	\]
	Since $f\in C^3$, the funciton $t\rightarrow \nabla^2 f(x_t)$ is differentiable and thus $q$ is differentiable with
	\[
	q'(\lambda)
	=
	\left\langle \frac{d}{d\lambda}\nabla^2 f(x_\lambda)\,u,\;u\right\rangle
	=
	\langle \nabla^3 f(x_\lambda)[\Delta,\cdot,\cdot]\,u,\;u\rangle.
	\]
	Applying local self-concordance as in \eqref{eq:gsc-multi-general} at $x_\lambda$ with $u=\Delta$ yields
	\[
	|q'(\lambda)|
	\;\le\;
	M\,\|\Delta\|_*\, q(\lambda).
	\]
	Therefore, $q'(\lambda) - M \|\Delta\|_* q(\lambda) \leq 0$ and, thus
	\begin{align*}
	 \frac{\diff}{\diff \lambda}\big(e^{-M \|\Delta\|_*\lambda} q(\lambda)\big)
		= e^{-M \|\Delta\|_*\lambda}  \big(q'(\lambda) - M \|\Delta\|_* q(\lambda)\big) \leq 0,
	\end{align*}
	We conclude that $e^{-M \|\Delta\|_*\lambda} q(\lambda)$ is non-increasing in $\lambda$, and therefore
	\begin{equation*}
		e^{-M \|\Delta\|_*} q(1) \leq q(0).
	\end{equation*}
	Similarly, one can verify that $e^{M \|\Delta\|_*\lambda} q(\lambda)$ is non-decreasing in $\lambda$, which implies
	\begin{equation*}
		e^{M \|\Delta\|_*} q(1) \geq q(0).
	\end{equation*}
	Therefore,
	\[
	 e^{-M\|\Delta\|_*}\,q(0)\ \le\ q(1)\ \le\ e^{M\|\Delta\|_*}\,q(0).
	\]
	Substituting back $q(0)=u^T\nabla^2 f(x)u$ and $q(1)=u^T\nabla^2 f(y)u$, we obtain for all $u$,
	\[
	e^{-M\|y-x\|_*}\,u^T\nabla^2 f(x)u
	\le
	u^T\nabla^2 f(y)u
	\le
	e^{M\|y-x\|_*}\,u^T\nabla^2 f(x)u.
	\]
	Since the inequality holds for every $u\in\mathbb{R}^p$, it is equivalent to \eqref{eq:hess-sd-order-star}.
\end{proof}

\subsection{Proof of Proposition~\ref{prop:gsc_hedge}: Self-Concordance of log-total-potential for Hedge}\label{sec:detailed_proof}

We will start with two lemmas concerning function $\Psi$ in Section~\ref{sec:proof_hedge}.

\begin{lemma}[Gradient and Hessian of $\Psi$]\label{lem:gh}
For all $x\in\R^n$, $a>0$. Let $\Psi(\bm x)= \log\sum_{i=1}^n\exp(ax_i)$
and that 
\begin{equation*}
p_i(x) \;=\; \frac{e^{a x_i}}{\sum_{j=1}^n e^{a x_j}},\qquad p(x)=(p_1,\ldots,p_n).
\end{equation*}
Then
\begin{align*}
\nabla\Psi(\bm x) &= a\,p(x),\\
\nabla^2\Psi(\bm x) &= a^2\big(\diag(p(x)) - p(x)\,p(x)^T\big).
\end{align*}
Consequently, for any $h\in\R^n$,
\begin{equation*}
h^\Tp \nabla^2\Psi(\bm x)\,h \;=\; a^2\!\left(\sum_{i=1}^n p_i(x)\,h_i^2 - \Big(\sum_{i=1}^n p_i(x)\,h_i\Big)^2\right)
\;=\; a^2\,\Var_{p(x)}(H),
\end{equation*}
where $H$ is the discrete r.v.\ taking values $h_i$ with probabilities $p_i(x)$.
\end{lemma}
\begin{proof}
Let $Z(x):=\sum_j e^{a x_j}$. Then $\Psi(\bm x)=\log Z(x)$ and $\partial_{x_i}\Psi=a e^{a x_i}/Z=a\,p_i(x)$, proving the gradient formula.
For the Hessian, $\partial_{x_j}p_i(x)=a(\mathbf{1}_{i=j}p_i-p_ip_j)$, hence
\(
\partial_{x_j}\partial_{x_i}\Psi=a\,\partial_{x_j}p_i=a^2(\mathbf{1}_{i=j}p_i-p_ip_j)
\),
which yields the stated matrix form. Plugging the Hessian into $h^T(\cdot)h$ gives the variance identity.
\end{proof}
This are the standard identities of log-sum-exp function.

\begin{lemma}[Moment/covariance identities along a path]\label{lem:cov-id}
Fix $x\in\R^n$ and $\Delta\in\R^n$. For $t\in\R$ define $p^{(t)}:=p(x+t\Delta)$ and let expectations be w.r.t.\ $p^{(t)}$. For any fixed $v\in\R^n$, set the r.v.\ $V$ to take values $v_i$, and $W$ to take values $\Delta_i$. Then
\begin{align*}
\phi_v(t) &:= v^\Tp \nabla^2\Psi(\bm x+t\Delta)\,v \;=\; a^2\,\Var_{p^{(t)}}(V),\\
\phi_v'(t) &= D^3\Psi(\bm x+t\Delta)[v,v,\Delta] \;=\; a^3\,\Cov_{p^{(t)}}\!\big((V-\mathbb{E}V)^2,\,W\big).
\end{align*}
\end{lemma}

\begin{proof}
The first identity is Lemma~\ref{lem:gh} applied at $x+t\Delta$. For the second, note $p^{(t)}_i\propto e^{a(x_i+t\Delta_i)}$, so
\(
\frac{d}{dt}\mathbb{E}_{p^{(t)}}[F]=a\,\Cov_{p^{(t)}}(F,W)
\)
for any function $F$ of the index $i$. Taking $F=V^2$ and $F=V$ and using $\Var(V)=\mathbb{E}[V^2]-(\mathbb{E}V)^2$ gives
\(
\frac{d}{dt}\Var(V)=a\,\Cov(V^2,W)-2\mathbb{E}V\cdot a\,\Cov(V,W)=a\,\Cov\big((V-\mathbb{E}V)^2,W\big).
\)
Multiplying by $a^2$ yields the stated $\phi_v'(t)$.
\end{proof}

Now we are ready to prove Proposition~\ref{prop:gsc_hedge}.

\begin{proof}[Proof of Proposition~\ref{prop:gsc_hedge}]
Fix $x,v,\Delta$ and consider $\phi_v(t)$ as in Lemma~\ref{lem:cov-id}. From the covariance form,
\begin{align*}
|\phi_v'(t)|
&= a^3\left|\mathbb{E}\!\left[(V-\mathbb{E}V)^2\,(W-\mathbb{E}W)\right]\right|
\;\le\; a^3\,\Big(\sup_i |W_i-\mathbb{E}W|\Big)\,\mathbb{E}\!\left[(V-\mathbb{E}V)^2\right]\\
&\le\; a^3\big(\max_i \Delta_i-\min_i \Delta_i\big)\,\Var(V)
\;\le\; 2a^3\|\Delta\|_\infty\,\Var(V).
\end{align*}
By Lemma~\ref{lem:gh}, $\phi_v(t)=a^2\Var(V)$, hence
\(
|\phi_v'(t)|\le (2a)\,\|\Delta\|_\infty\,\phi_v(t).
\)
Evaluating at $t=0$ gives \eqref{eq:gsc-multi-general} in Definition~\ref{def:gsc-multi-general} with $M=2a=2\sqrt{2}\eta$ and $\|\cdot\|_{\infty}$
\end{proof}

\subsection{Proof of Proposition~\ref{prop:localGSC_normalhedge}: Local self-concordance of log-total potential }\label{sec:gsc_normal_hedge_proof}

	\begin{proposition}[Restatement of Proposition~\ref{prop:localGSC_normalhedge}] 
    The function $ \LogPot(\bm x, t)$ satisfies local generalized self-concordance with $M=1$, $\nu=2$ and $\cC$ being the line-segment in \eqref{eq:segment} with respect to a special norm $\|\cdot\|_*$ defined as 
    $$
    \|(\bar{\bx},\bar{t})\|_* =  A_x(t_\star,K_{\rm seg}) \|\bar{x}\|_\infty + A_t(t_\star,K_{\rm seg})|\bar{t}|.
    $$
    where 
	\begin{equation*}
        \begin{aligned}
A_x(t_\star,K_{\rm seg})
= 
\frac{8\sqrt{K_{\rm seg}\vee 1}}{\sqrt{t_\star}}, \quad \text{ and }\quad
A_t(t_\star,K_{\rm seg})= 
\frac{16 (K_{\rm seg}\vee 1) }{t_*}.
\end{aligned}
		\end{equation*}
	\end{proposition}

	\begin{proof}[Proof of Proposition~\ref{prop:localGSC_normalhedge}]

    		Fix $(\bx,t) \in \R^n \times \R_{+}$ and $\bh:=(\Delta \bx,\Delta t)$, and define $(\bx', t') \coloneqq (\bx, t) + \bh$. For any point $(\bar{\bx},\bar{t})$ on the segment \eqref{eq:segment} and any direction $\bu=(u_1,\ldots,u_n,u_t)\in\mathbb R^{n+1}$, it suffices to prove
		\begin{equation}\label{eq:localGSC}
			\big|\nabla^3 \LogPot(\bar{x},\bar{t})[\bu,\bu,\bh]\big|
			\ \le\
			\Big(A_x(t_\star,K_{\rm seg})\,\|\Delta \bx\|_\infty
			+ A_t(t_\star,K_{\rm seg})\,|\Delta t|\Big)
			\;\nabla^2 \LogPot(\bar{\bx},\bar{t})[\bu,\bu].
		\end{equation}
 for $A_x$ and $A_t$ in \eqref{eq:AxAt}.
		Overloading a bit of notation, define the path
		\begin{equation*}
			\big(\bx(s),t(s)\big):=(\bx,t)+s\,\bh,\qquad \forall s \in [0,1].
		\end{equation*}
		For any $\bu\in\mathbb R^{n+1}$ define
		\begin{equation}\label{eq:Phi}
			\Psi_u(s):=\nabla^2 \LogPot \big(\bx(s),t(s)\big)[(\bu,\bu)], \qquad \forall s \in [0,1].
		\end{equation}
		By the chain rule along the path $s \in [0,1] \mapsto (\bx(s),t(s))$, whose derivative in $s$ is $\bh$ we have
		\begin{equation}\label{eq:PhiPrime}
			\Psi_u'(s)=\nabla^3 \LogPot \big(\bx(s),t(s)\big)[(\bu,\bu,\bh)] \qquad \forall s \in [0,1].
		\end{equation}	
		Our goal now is to show
		\begin{equation}\label{eq:goal}
			|\Psi_u'(s)|\ \le\ \big(A_x(t_\star,K_{\rm seg})\|\Delta \bx\|_\infty+A_t(t_\star,K_{\rm seg})|\Delta t|\big)\,\Psi_u(s),
		\end{equation}
		with $A_x,A_t$ given by \eqref{eq:AxAt}. %
		
		\paragraph{Log-sum-exp cumulant identities along the path.} Let us start by stating some identities connecting $\LogPot$ with $\log \phi$ and its derivatives. 	These identities are standard
		consequences of cumulant expansions for the log-partition function
		of a finite exponential family; see, e.g.,
		\cite{nesterov1994interior,wainwright2008graphical,amari2000methods}.  We include the proof  in the appendix for completeness.
		
		\begin{lemma}[Cumulant identities for $\LogPot$]\label{lem:cumulants_specialized}
			Define
			\[
			f_i(\bx,t):= \log \phi(x_i,t) = -\frac12\log t+\frac{x_i^2}{2t},\qquad \text{so that}~
			\LogPot(\bm x,t):=\log\sum_{i=1}^n e^{f_i(\bx,t)},
			\]
			and let $\bp(\bx,t) \in \Delta^{N-1}$ be the softmax weights of $f_i(\bx,t)$, that it, $p_i(\bx,t) \propto \exp(f_i(\bx,t))$.
			Let  $u=(u_1,\dots,u_n,u_t)$ and
			$h=(h_1,\dots,h_n,h_t)$ be in $\mathbb R^{n+1}$.
			Define, for each $i \in [n]$,
			\[
			A_i:=\nabla f_i(\bx,t)[\bu],\qquad
			B_i:=\nabla^2 f_i(\bx,t)[\bu,\bu],\qquad
			C_i:=\nabla^2 f_i(\bx,t)[\bu,\bh],
			\]
			\[
			H_i:=\nabla f_i(\bx,t)[\bh],\qquad
			T_i:=\nabla^3 f_i(\bx,t)[\bu,\bu,\bh].
			\]
			Let \(I\) be a random variable taking values in \([n]\) such that
			$\P(I = i) = p_i(\bx,t)$ for all $i 
			\in [n]$.
			Then
			\begin{align}
				\nabla^2 \LogPot(\bm x,t)[\bu,\bu]
				&= \mathbb E[B_I]+\mathrm{Var}(A_I), \label{eq:G2}\\
				\nabla^3 \LogPot(\bm x,t)[\bu,\bu,\bh]
				&= \mathbb E[T_I]+\mathrm{Cov}(B_I,H_I)
				+2\,\mathrm{Cov}(A_I,C_I)
				+\mathrm{Cov}\big((A_I-\mathbb E A_I)^2,H_I\big).
				\label{eq:G3mixed}
			\end{align}
		\end{lemma}

	\paragraph{Remark regarding notation.} Throughout the rest of the proof we fix $s\in[0,1]$, we suppress $s$ in the notation, and we omit the dependence on $s$, writing $x_i$ and $t$ instead of $x_i(s)$ or $t(s)$. We also often omit the dependency on $(\bx(s), t(s))$ whenever it is clear from context, writing $\nabla^2 \LogPot[(\bu,\bu)]$ and $\nabla^3 \LogPot[(\bu,\bu,\bh)]$ instead of $\nabla^2 \LogPot(\bx(s),t(s))[(\bu,\bu)]$ and $\nabla^3 \LogPot(\bx(s),t(s))[(\bu,\bu,\bh)]$.  They should be interpreted as functions of $s \in [0,1]$, defined on a given line-segment.
		
	Equipped with the above lemma, to establish an upper bound of $\nabla^3 \LogPot[(\bu,\bu,\bh)]$ in terms of	$\nabla^2 \LogPot[(\bu,\bu)]$ as stated in \eqref{eq:goal}, it suffices to show that each element in \eqref{eq:G3mixed} is proportional to either $\mathbb E[B_I] $ or $\mathrm{Var}(A_I)$.  Let us start by computing the values of $A_i, B_i,C_i, H_i, T_i$ for $i \in [n]$. We will do so by introducing some new notation to reparameterize the derivatives of $f_i$ and simplify some of the calculations

		\paragraph{Per-coordinate reparameterization and explicit derivatives.}
		
		Fix $i \in [n]$ and let $\bw \coloneqq (w_i,w_t) \in \R^2$. Define
		\[
		a_i^\star(\bw):=w_i-\frac{x_i}{t}\,w_t,\qquad
		b(\bw):=\frac{w_t}{t},\qquad
		r_i(\bw):=\frac{(a_i^\star(\bw))^2}{t},\qquad
		y(\bw):=b(\bw)^2.
		\]
		Moreover, define notation for the partial derivatives of $f_i$ as follows:
		\begin{equation*}
		f_x \coloneqq \frac{x_i}{t},\quad f_t\coloneqq -\frac{1}{2t}-\frac{x_i^2}{2t^2},\quad
		f_{xx} \coloneqq \frac{1}{t},\quad f_{xt} \coloneqq-\frac{x_i}{t^2},\quad f_{tt} \coloneqq \frac{1}{2t^2}+\frac{x_i^2}{t^3},
		\end{equation*}
		With the above notation, we obtain the following expressions
		\paragraph{First derivative.} We have
		\begin{equation*}
		\nabla f_i[\bw]=f_x w_i+f_t w_t= \frac{x_i}{t}w_i - \frac{1}{2t} w_t - \frac{x_i^2}{2t^2} w_t =\frac{x_i}{t}\,a_i^\star(\bw)+\frac12\left(\frac{x_i^2}{t}-1\right)\,b(\bw).
		\end{equation*}
		Thus, 
		$$
		A_i =\frac{x_i}{t}\,a_i^\star(\bu)+\frac12\left(\frac{x_i^2}{t}-1\right)\,b(\bu).
		$$
		and 
		\begin{equation}\label{eq:Df}
		H_i =\frac{x_i}{t}\,a_i^\star(\bh)+\frac12\left(\frac{x_i^2}{t}-1\right)\,b(\bh).
	\end{equation}

		\paragraph{Second derivative.} We have
		\begin{align*}
	\nabla^2 f_i[(\bw,\bw)] &= f_{xx}w_i^2+2f_{xt}w_iw_t+f_{tt}w_t^2\\
		&=\frac{w_i^2}{t} - \frac{2w_iw_t x_i}{t^2} + \frac{w_t^2x_i^2}{t^3} + \frac{w_t^2}{2t^2} \\
		&=\frac{(a_i^\star(\bw))^2}{t}+\frac{1}{2}b(\bw)^2 = r_i(\bw)+\tfrac12 y(\bw).
		\end{align*}
		Thus 
		\begin{equation}
			\label{eq:second_derivatives_appx_proof}
			B_i  =  r_i(\bu)+\tfrac12 y(\bu).
		\end{equation}
		
		\paragraph{Mixed second derivative.} We have
		\begin{align}
		C_i = \nabla^2 f_i[(\bu,\bh)] &=  f_{xx} u_i h_i + f_{tt}u_t h_t + f_{xt}u_i h_t + f_{tx} h_i u_t  \nonumber\\
		&= \frac{u_ih_i}{t} + \frac{u_t h_t}{2t^2} + \frac{u_t h_t x_i^2}{t^3}  - \frac{x_i}{t^2}(u_i h_t + h_i u_t)\nonumber\\
		&= \frac{1}{t}a_i^\star(\bu)a_i^\star(\bh)+\frac{1}{2}b(\bu)b(\bh). \label{eq:D2mix}
	\end{align}

		\paragraph{Third mixed derivative.}
		Differentiating $\nabla^2 f_i[(\bu,\bu)]=r_i(\bu)+\frac12y(\bu)$ in the direction $\bh$ yields
		\[
		\begin{aligned}
			\nabla^3 f_i[(\bu,\bu,\bh)] &= \nabla r_i(\bu)[\bh]+\tfrac12\nabla y(\bu)[\bh],\\
			\nabla r_i(\bu)[\bh] &= \nabla_x \left(\frac{(a_i^\star(\bu))^2}{t}\right) h_i +  \nabla_t  \left(\frac{(a_i^\star(\bu))^2}{t}\right)  h_t\\
			&=   \frac{2a_i^\star(\bu)}{t}\Big(-\frac{u_t}{t}h_i+\frac{x_i u_t}{t^2}h_t\Big) -\frac{(a_i^\star(\bu))^2}{t^2}h_t,\\
			\tfrac12\nabla y(\bu)[\bh] &=   \frac{1}{2} \nabla_x b(\bw)^2 h_i  +  \frac{1}{2} \nabla_t b(\bu)^2 h_t = \frac{1}{2} \nabla_t b(\bu)^2 h_t \\
			&=b(\bu)\nabla_t b(\bu)h_t  = - \frac{u_t}{t} \frac{u_t}{t^2} h_t=     -\frac{1}{t} (\frac{w_t}{t})^2h_t= -\frac{y(\bu)}{t}h_t.
		\end{aligned}
		\]
		Hence
		\begin{equation}\label{eq:T-final}
		T_i:= \nabla^3 f_i[(\bu,\bu,\bh)] =
		\frac{2 a_i^\star(\bu)}{t}\Big(-\frac{u_t}{t}h_i+\frac{x_i u_t}{t^2}h_t\Big)
		-\frac{(a_i^\star(\bu))^2}{t^2}h_t-\frac{y(\bu)}{t}h_t.
	\end{equation}

		\paragraph{Bounding intermediate states using the second derivatives}
		From \eqref{eq:second_derivatives_appx_proof} we have $B_i = \nabla^2 f_i[(\bu,\bu)]=\frac{a_i^*(\bu)^2}{t} +  
	 \frac{1}{2}b(\bu)^2 =  r_i(\bu) + \frac{1}{2} y(\bu)$, which yields
		\begin{equation}\label{eq:triangle}
		|a_i^\star(\bu)|\le \sqrt{t B_i},\qquad |b(\bu)|\le \sqrt{2B_i}, \qquad 
		r_i(\bu) \leq B_i,\qquad y(\bu) \leq 2B_i.
		\end{equation}
		In addition, recall that $\bh=(\Delta \bx,\Delta t)$. From the bounds in terms of $\Kseg$ and $\tstar$ from \eqref{eq:seg-controls}, 
		\begin{equation}\label{eq:diamond}
			|a_i^\star(\bh)|\le  | \Delta x_i - \frac{x_i \Delta _t}{t} |  \leq \|\Delta \bx\|_\infty+\sqrt{\frac{K_{\rm seg}}{t_\star}}\,|\Delta t|,
			\qquad
			|b(\bh)|=\frac{|\Delta t|}{t}\le \frac{|\Delta t|}{t_\star}.
		\end{equation}
		
		These will simplify all subsequent bounds. 
		Recall that from Lemma~\ref{lem:cumulants_specialized},
		$$
					\nabla^3 \LogPot[(\bu,\bu,\bh)] =\mathbb E[T_I]+\mathrm{Cov}(B_I,H_I)+2\,\mathrm{Cov}(A_I,C_I)+\mathrm{Cov}\!\big((A_I-\mathbb E A_I)^2,H_I \big).
		$$
		Again, our goal it to bound the above as a multiple of $\nabla^2 \LogPot [\bu,\bu] $. The latter, again by Lemma~\ref{lem:cumulants_specialized}, is equal to $\mathbb E[B_I] + \mathrm{Var}[A_I]$. Therefore, we will bound each of the four terms above by a multiple of either $\mathbb E[B_I]$ or $\mathrm{Var}[A_I]$.

		\paragraph{Term 1: bound $\big|\mathbb E[T_I]\big|$.}
		Recall \eqref{eq:T-final}:
		\[
		T_i=
		\frac{2 a_i^\star(\bu)}{t}\Big(-\frac{u_t}{t}h_i+\frac{x_i u_t}{t^2}h_t\Big)
		-\frac{(a_i^\star(\bu))^2}{t^2}h_t-\frac{y(\bu)}{t}h_t .
		\]
		Recall that $h_i=\Delta x_i$ and $h_t=\Delta t$.
		Using that $b(\bu)=u_t/t$, the bounds from \eqref{eq:triangle} 
		and $t\ge t_\star$, we bound each term.
		First,
		\begin{align*}
		\Big|\frac{2 a_i^\star(\bu)}{t}\frac{u_t}{t}h_i\Big|
		&=
		\Big|\frac{2 a_i^\star(\bu)}{t}\,b(\bu)\,h_i\Big|
		\ \le\ 
		\frac{2|a_i^\star(\bu)|}{t}\,|b(\bu)|\,|\Delta x_i|\\
		&\ \le\ 
		\frac{2\sqrt{tB_i}}{t}\,\sqrt{2B_i}\,\|\Delta \bx\|_\infty
		\ \le\ 
		\frac{2\sqrt{2}}{\sqrt{t_\star}}\,B_i\,\|\Delta \bx\|_\infty.
		\end{align*}
		Next,
		\begin{align*}
		\Big|\frac{2 a_i^\star(\bu)}{t}\frac{x_i u_t}{t^2}h_t\Big|
		&=
		\Big|\frac{2 a_i^\star(\bu)}{t}\,\frac{x_i}{t}\,b(\bu)\,h_t\Big|
		\ \le\
		\frac{2|a_i^\star(\bu)|}{t}\cdot \frac{|x_i|}{t}\cdot |b(\bu)|\,|\Delta t|\\
		&\ \le\
		\frac{2\sqrt{tB_i}}{t}\cdot \sqrt{\frac{K_{\rm seg}}{t_\star}}\cdot \sqrt{2B_i}\,|\Delta t|
		\ \le\
		\frac{2\sqrt{2}\,\sqrt{K_{\rm seg}}}{t_\star}\,B_i\,|\Delta t|.
		\end{align*}
		Also,
		\[
		\Big|\frac{(a_i^\star(\bu))^2}{t^2}\,h_t\Big|
		\le
		\frac{B_i}{t}\,|\Delta t|
		\le
		\frac{B_i}{t_\star}\,|\Delta t|,
		\qquad
		\Big|\frac{y(\bu)}{t}\,h_t\Big|
		\le
		\frac{2B_i}{t}\,|\Delta t|
		\le
		\frac{2B_i}{t_\star}\,|\Delta t|.
		\]
		
		Summing and taking expectation over $I \sim \bp$,
		\begin{equation}\label{eq:MeanThirdBound}
			\big|\mathbb E[T_I]\big|
			\ \le\
			\left(\frac{2\sqrt{2}}{\sqrt{t_\star}}\ \|\Delta \bx\|_\infty
			\ +\ \frac{3+2\sqrt{2}\sqrt{K_{\rm seg}}}{t_\star}\ |\Delta t|\right)\,\mathbb E[B_I].
		\end{equation}

		\paragraph{Term 2: bound $|\mathrm{Cov}(B_I,H_I)|$.}
		We will use the following inequality and a bound on $\sup|H|$.
\begin{lemma}\label{lem:rough_cov_inequality}
	Let $X\ge 0$ be an integrable random variable and let $Y$ be bounded. Then
	\begin{equation}\label{eq:Cov-rough}
		|\mathrm{Cov}(X,Y)| 
		\le 2\,\sup|Y|\,\mathbb E[X].
	\end{equation}
\end{lemma}
The proof of the above lemma is deferred to the end.
		For the bound on $\sup_i |H_i|$, from \eqref{eq:Df} and \eqref{eq:diamond} we have for any $i \in [n]$,%
		\begin{align}
			|H_i|=|\nabla f_i[\bh]|
			&\ \le\   \frac{x_i}{t}\,a_i^\star(\bh)+\frac12\left(\frac{x_i^2}{t}-1\right)\,b(\bh)\nonumber
			\\
			&\leq			
			\sqrt{\frac{K_{\rm seg}}{t_\star}}\;\|\Delta \bx\|_\infty
			\ +\ \Big(\frac{K_{\rm seg}}{t_\star}+\frac{1}{2t_\star}\Big)\,|\Delta t|. \label{eq:H-infty}
		\end{align}
		
		Applying Lemma~\ref{lem:rough_cov_inequality} with $X=B_I$, $Y=H_I$, yields %
		\begin{equation}\label{eq:CovBH}
			|\mathrm{Cov}(B_I,H_I)|
			\ \le\ \left(\frac{2\sqrt{K_{\rm seg}}}{\sqrt{t_\star}}\|\Delta \bx\|_\infty
			+\frac{2K_{\rm seg}+1}{t_\star}|\Delta t|\right)\,\mathbb E [B_I].
		\end{equation}
	
	Observe that $\mathbb E [B_I]\leq \mathbb E[B_I]+\mathrm{Var}(A_I) = \Psi_u(s)$ by \eqref{eq:G2}.
		
		\paragraph{Term 3: bound $2\,|\mathrm{Cov}(A_I,C_I)|$.}
		From \eqref{eq:D2mix}, \eqref{eq:triangle} and \eqref{eq:diamond},
		\[
		|C_i|
		\ \le\ \frac{|a_i^\star(\bu)|}{t}\,|a_i^\star(\bh)|+\frac{|b(\bu)|}{2}\,|b(\bh)|
		\ \le\ \sqrt{B_i}\left(\frac{\|\Delta \bx\|_\infty}{\sqrt{t_\star}}
		+\Big(\frac{\sqrt{K_{\rm seg}}}{t_\star}+\frac{1}{\sqrt{2}\,t_\star}\Big)|\Delta t|\right).
		\]
		By Cauchy--Schwarz (first line) and AM--GM (third line),
		\begin{align}
			|\mathrm{Cov}(A_I,C_I)|
			&\le \sqrt{\mathrm{Var}(A_I)}\sqrt{\mathrm{Var}(C_I)}
			\ \le\ \sqrt{\mathrm{Var}(A_I)}\sqrt{\mathbb E[C_I^2]}\nonumber\\ 
			&\le \left(\frac{\|\Delta \bx\|_\infty}{\sqrt{t_\star}}
			+\Big(\frac{\sqrt{K_{\rm seg}}}{t_\star}+\frac{1}{\sqrt{2}\,t_\star}\Big)|\Delta t|\right)  \sqrt{\mathrm{Var}(A_I) \mathbb E[B_I] } \nonumber\\
			&\leq \frac12 \left(\frac{\|\Delta \bx\|_\infty}{\sqrt{t_\star}}
			+\Big(\frac{\sqrt{K_{\rm seg}}}{t_\star}+\frac{1}{\sqrt{2}\,t_\star}\Big)|\Delta t|\right)  
			\big(\mathrm{Var}(A_I)+\mathbb E[B_I]\big).
		\end{align}
		Observe that $\mathrm{Var}(A)+\mathbb E[B_I] = \nabla^2 \LogPot[(\bu,\bu)] = \Phi_u(s)$ (from Lemma~\ref{lem:cumulants_specialized}), Term 3 is bounded as follows:
		\begin{equation}\label{eq:CovAC}
			2\,|\mathrm{Cov}(A_I,C_I)|
			\ \le\ \left(\frac{\|\Delta \bx\|_\infty}{\sqrt{t_\star}}
			+\Big(\frac{\sqrt{K_{\rm seg}}}{t_\star}+\frac{1}{\sqrt{2}\,t_\star}\Big)|\Delta t|\right)\, \Phi_u(s).
		\end{equation}
		
		\paragraph{Term 4: bound $\big|\mathrm{Cov}((A_I-\mathbb E A_I)^2,H_I)\big|$.}

		Applying Lemma~\ref{lem:rough_cov_inequality} with  $X=(A_I -\mathbb E A_I)^2$ and $Y=H_I$ together on the bound on \(\sup_i |H_i|\) from \eqref{eq:H-infty} yields
		\begin{equation}\label{eq:CovA2H}
			\big|\mathrm{Cov}\big((A_I-\mathbb E A_I)^2,H_I\big)\big|
			\ \le\ \mathrm{Var}(A_I)  \left(\frac{2\sqrt{K_{\rm seg}}}{\sqrt{t_\star}}\|\Delta \bx\|_\infty
			+\frac{2K_{\rm seg}+1}{t_\star}|\Delta t|\right).
		\end{equation}
	
	Note that $\mathrm{Var}_p(A)  \leq  \Phi_u(s)$ by \eqref{eq:G2}.
		
		\paragraph{Collecting bounds and integrating.}
		Add \eqref{eq:MeanThirdBound}, \eqref{eq:CovBH}, \eqref{eq:CovAC}, \eqref{eq:CovA2H}. Group the coefficients of $\|\Delta \bx\|_\infty$ and $|\Delta t|$ yields
		\[
		|\Phi_u'(s)|\ \le\ \Big(\tilde{A}_x\,\|\Delta \bx\|_\infty+\tilde{A}_t\,|\Delta t|\Big)\,\Phi_u(s),
		\]
		with
		\[
		\tilde{A}_x=\frac{2\sqrt{2}}{\sqrt{t_\star}}
		+\frac{2\sqrt{K_{\rm seg}}}{\sqrt{t_\star}}
		+\frac{1}{\sqrt{t_\star}}
		+\frac{2\sqrt{K_{\rm seg}}}{\sqrt{t_\star}}
		=\frac{2\sqrt{2}+1+4\sqrt{K_{\rm seg}}}{\sqrt{t_\star}},
		\]
		\[
		\tilde{A}_t=\frac{3+2\sqrt{2}\sqrt{K_{\rm seg}}}{t_\star}
		+\frac{2K_{\rm seg}+1}{t_\star}
		+\left(\frac{\sqrt{K_{\rm seg}}}{t_\star}+\frac{1}{\sqrt{2}\,t_\star}\right)
		+\frac{2K_{\rm seg}+1}{t_\star}
		=\frac{5+\tfrac{1}{\sqrt{2}}}{t_\star}
		+\frac{(2\sqrt{2}+1)\sqrt{K_{\rm seg}}}{t_\star}
		+\frac{4K_{\rm seg}}{t_\star}.
		\]      
The proof is complete by relaxing the above bounds so that
\begin{equation*}
	\tilde{A}_x \leq \frac{8\sqrt{K_{\rm seg}\vee 1}}{\sqrt{t_\star}}= A_x(t_\star,K_{\rm seg}) 
	\qquad \text{and}\qquad \tilde{A}_t \leq \frac{16 (K_{\rm seg}\vee 1) }{t_*} = A_t(t_\star,K_{\rm seg}).
\end{equation*} 
\end{proof}

	\subsection{Proof of Auxiliary Lemmas}

\begin{proof}[Proof of Lemma~\ref{lem:cumulants_specialized}] Throughout this proof we omit the dependency on \((\bx, t)\). First, define \(Z \coloneqq \sum_{i = 1}^n \exp(f_i)\).	
	\paragraph{Step 1: a basic derivative rule for $p_i$.}
	Since $p_i=e^{f_i}/Z$ and $\LogPot=\log Z$, for any direction $\bw\in\mathbb R^{n+1}$,
	\begin{equation}\label{eq:dp_rule}
		\nabla p_i[\bw]
		= p_i \cdot \big(\nabla f_i[\bw]-\nabla G[\bw]\big).
	\end{equation}
	Indeed, $\nabla p_i[\bw]=p_i\,\nabla f_i[\bw]-p_i\,\nabla G[\bw]$ by the quotient rule.
	
	\paragraph{Step 2: gradient identity.}
	Differentiating $\LogPot =\log Z$ yields
	\[
	\nabla \LogPot [\bw]=\frac{1}{Z}\sum_{i=1}^n e^{f_i}\,\nabla f_i[\bw]
	=\sum_{i=1}^n p_i\,\nabla f_i[\bw].
	\]
	In particular, for $\bw=\bu$ we have $\nabla f_i[\bu]=A_i$, so
	\begin{equation}\label{eq:grad_spec}
		\nabla \LogPot [\bu]=\mathbb E[A_I].
	\end{equation}
	
	\paragraph{Step 3: Hessian identity.}
	Differentiate \eqref{eq:grad_spec} in direction $\bu$:
	\[
	\nabla^2 \LogPot[\bu,\bu]=\nabla\Big(\sum_{i=1}^n p_i A_i\Big)[\bu]
	=\sum_{i=1}^n \nabla p_i[\bu]\;A_i+\sum_{i=1}^n p_i\,\nabla A_i[\bu].
	\]
	Now $\nabla A_i[\bu]=\nabla^2 f_i[\bu,\bu]=B_i$, and by \eqref{eq:dp_rule} with $\bw=\bu$,
	\[
	\nabla p_i[\bu]=p_i(A_i-\nabla \LogPot[\bu])=p_i(A_i-\mathbb E[A_I]).
	\]
	Therefore
	\[
	\nabla^2 \LogPot[\bu,\bu]=\sum_{i=1}^n p_i(A_i-\mathbb E[A_I])A_i+\sum_{i=1}^n p_i B_i
	=\mathbb E[B]+\big(\mathbb E[A_I^2]-(\mathbb E[A_I])^2\big),
	\]
	which is \eqref{eq:G2}.
	
	\paragraph{Step 4: differentiate the Hessian in direction $h$.}
	From \eqref{eq:G2},
	\[
	\nabla^3 \LogPot[\bu,\bu,\bh]
	=\nabla\big(\mathbb E[B_I]\big)[\bh]+\nabla\big(\mathrm{Var}(A_I)\big)[\bh].
	\]
	
	\paragraph{Step 5: derivative of the expectation term.}
	Using the product rule,
	\[
	\nabla\big(\mathbb E[B_I]\big)[\bh]
	=\nabla\Big(\sum_{i=1}^n p_i B_i\Big)[\bh]
	=\sum_{i=1}^n \nabla p_i[\bh]\;B_i+\sum_{i=1}^n p_i\,\nabla B_i[\bh].
	\]
	Here $\nabla B_i[\bh]=\nabla^3 f_i[\bu,\bu,\bh]=T_i$. Also, by \eqref{eq:dp_rule} with $\bw=\bh$ and since \(H_i = \nabla f_i [\bh] \),
	\[
	\nabla p_i[\bh]=p_i(H_i-\nabla \LogPot[\bh])=p_i(H_i-\mathbb E[H_I]),
	\]
	where in the last equation we used that $\nabla \LogPot [\bh]=\sum_i p_i \nabla f_i[\bh]=\mathbb E[H_I]$.
	Hence,
	\[
	\nabla\big(\mathbb E[B_I]\big)[\bh]
	=\mathbb E[T_I]+\mathbb E\big[(H_I-\mathbb E[H_I])B_I\big]
	=\mathbb E[T_I]+\mathrm{Cov}(B_I,H_I).
	\]
	
	\paragraph{Step 6: derivative of the variance term.}
	Since $\mathrm{Var}(A_I) =\mathbb E[A_I^2]-(\mathbb E[A_I`'])^2$, we have
	\begin{equation}\label{eq:dvar_split}
		\nabla\big(\mathrm{Var}(A_I)\big)[\bh]
		=\nabla\mathbb E[A_I^2][\bh]-2\,\mathbb E[A_I]\;\nabla\mathbb E[A_I][\bh].
	\end{equation}
	We compute each piece. First,
	\[
	\nabla\mathbb E[A_I][\bh]
	=\nabla\Big(\sum_{i=1}^n p_i A_i\Big)[\bh]
	=\sum_{i=1}^n \nabla p_i[\bh]\;A_i+\sum_{i=1}^n p_i\,\nabla A_i[\bh].
	\]
	Now $\nabla A_i[\bh]=\nabla^2 f_i[\bu,\bh]=C_i$, and $\nabla p_i[\bh]=p_i(H_i-\mathbb E[H_I])$,
	so
	\begin{equation}\label{eq:dEA}
		\nabla\mathbb E[A][h]=\mathbb E[C]+\mathrm{Cov}(A,H).
	\end{equation}
	Second,
	\[
	\nabla\mathbb E[A^2][h]
	=\nabla\Big(\sum_{i=1}^n p_i A_i^2\Big)[h]
	=\sum_{i=1}^n \nabla p_i[h]\;A_i^2+\sum_{i=1}^n p_i\,\nabla(A_i^2)[h].
	\]
	Since $\nabla(A_i^2)[h]=2A_i\,\nabla A_i[h]=2A_iC_i$, we get
	\begin{equation}\label{eq:dEA2}
		\nabla\mathbb E[A_I^2][\bh]
		=\mathbb E[2A_IC_I]+\mathrm{Cov}(A_I^2,H_I).
	\end{equation}
	Substitute \eqref{eq:dEA}--\eqref{eq:dEA2} into \eqref{eq:dvar_split}:
	\[
	\nabla\big(\mathrm{Var}(A_I)\big)[\bh]
	=2\mathbb E[A_I C_I]+\mathrm{Cov}(A_I^2,H_I)-2\mathbb E[A_I]\mathbb E[C_I]-2\mathbb E[A_I]\mathrm{Cov}(A_I,H_I).
	\]
	Note that
	\[
	2\mathbb E[A_IC_I]-2\mathbb E[A_I]\mathbb E[C_I]=2\,\mathrm{Cov}(A_I,C_I),
	\]
	and also the covariance identity
	\[
	\mathrm{Cov}\big((A_I-\mathbb E[A_I])^2,H_I\big)
	=\mathrm{Cov}(A_I^2,H_I)-2\mathbb E[A_I]\mathrm{Cov}(A_I,H_I)
	\]
	follows from expanding $(A_I-\mathbb E[A_I])^2=A_I^2-2\mathbb E[A_I]A_I+(\mathbb E[A_I])^2$
	and using $\mathrm{Cov}(\E [A_I],H_I)=0$ since $\E [A_I]$ is a constant.
	Therefore,
	\[
	\nabla\big(\mathrm{Var}(A_I)\big)[\bh]
	=2\,\mathrm{Cov}(A_I,C_I)+\mathrm{Cov}\big((A_I-\mathbb E[A_I])^2,H_I\big).
	\]
	
	\paragraph{Step 7: conclude.}
	Combining Step 5 and Step 6 with Step 4 gives \eqref{eq:G3mixed}.
\end{proof}

	\begin{proof}[Proof of Lemma~\ref{lem:rough_cov_inequality}]
		Recall that
		\[
		\mathrm{Cov}(X,Y)
		=
		\mathbb E\big[(X-\mathbb EX)(Y-\mathbb EY)\big].
		\]

		Since $\mathbb E(Y-\mathbb EY)=0$, we may also write
		\[
		\mathrm{Cov}(X,Y)
		=
		\mathbb E\big[X(Y-\mathbb EY)\big].
		\]
		Using that $X\ge 0$, we obtain
		\[
		|\mathrm{Cov}(X,Y)|
		\le
		\mathbb E\big[X\,|Y-\mathbb EY|\big].
		\]
		If $\|Y\|_\infty := \sup |Y|<\infty$, then
		\[
		|Y-\mathbb EY|
		\le
		|Y| + |\mathbb EY|
		\le
		2\|Y\|_\infty.
		\]
		Therefore,
		\begin{equation*}
		\mathbb E\big[X\,|Y-\mathbb EY|\big]
		\le
		2\|Y\|_\infty\,\mathbb E[X].
		\end{equation*} 
	\end{proof}

\section{Detailed proof of Theorem~\ref{thm:main}: NormalHedge}
\label{sec:NHproof}

\subsection{Higher order derivatives of normal potential}
Next, we will bound the higher-order derivatives of $\phi$. 
\begin{lemma}[Bounding the higher order derivatives of $\phi$]\label{lem:derivatives}
		Let $\phi(x,t) = \phiNH(x,t)$. Assume $x^2/t \leq K$. For all $x$ and $t> 0$,
		$$
		\frac{\partial^2}{\partial x^2}\phi(x,t) =  \Big(\frac{x^2}{t^2} + \frac{1}{t} \Big) \phi(x,t) \leq  \frac{(1+K)}{t}\phi(x,t)
		$$
		$$
		\frac{\partial^3}{\partial x^3}\phi(x,t)  =  \Big(\frac{x^3}{t^3} + \frac{3x}{t^2} \Big) \phi(x,t) \leq \frac{K^{1.5} + 3 K^{0.5}}{t^{1.5}}\phi(x,t)
		$$
		$$
		\frac{\partial^4}{\partial x^4}\phi(x,t)  = 4 \frac{\partial^2}{\partial t^2} \phi(x,t) =  \Big(\frac{x^4 + 6 tx^2 + 3t^2}{t^4} \Big) \phi(x,t) \leq \frac{K^2 + 6 K + 3}{t^2}  \phi(x,t)
		$$
	\end{lemma}
    \begin{proof}
    The proof follows from standard calculus. We have verified the identities symbolically using both SymPy and Maple. 
    \end{proof}

\subsection{Proof of Lemma~\ref{lem:bahia-davis}}\label{sec:proof_using_bahia_davis}
\begin{lemma}[Restating Lemma~\ref{lem:bahia-davis}]
For $\phi=\phiNH$,
\[
\TimeErr_\Phi\le \frac{\Delta t}{4t}\big(\max_i\frac{x_i^2}{t}+4\big),\blue{
\quad\quad
|\MixErr_\Phi|\le B\sqrt{\frac{1}{t}\max_i\frac{x_i^2}{t}}.
}
\]
\end{lemma}

\begin{proof}[Proof of Lemma~\ref{lem:bahia-davis}]
Define \(\br \in \Delta^{N-1}\) by \(r_i \coloneqq \phi(x_i, t)/\Phi(\bx)\) for each \(i \in [n]\). By the formulas for the derivatives of $\phi$ in Lemma~\ref{lem:derivatives}, we get
\begin{align*}
\frac{1}{\Phi(\bm x, t)} \sum_i \frac{\partial^4}{\partial x^4} \phi(x_i,t) &= \E_{i\sim \br}\left[ \frac{x_i^4 + 6tx_i^2+3t^2}{t^4}\right],
\\
\text{and}\quad \frac{1}{\Phi(\bm x, t)} \sum_i \frac{\partial^2}{\partial x^2} \phi(x_i,t) &= \E_{i\sim \br}\left[ \frac{x_i^2 + t}{t^2}\right].
\end{align*} 
Set $u_i \coloneqq \frac{x_i^2}{t} + 1$. Plugging them into the definiton of $\TimeErr_\Phi$ for $\phi = \phiNH$ yields
\begin{align*}
\TimeErr_\Phi
= \frac{\Delta t}{4t}\left( \frac{\E_{i\sim \br}[u_i^2 +4 u_i -2]}{\E_{i\sim \br}[u_i]} - \E_{i\sim \br}[u_i] \right)
&\leq \frac{\Delta t}{4t}\left( \frac{\Var_{i\sim \br}[u_i]}{\E_{i\sim \br}[u_i]} + 4\right)\\
&\leq \frac{\Delta t}{4t}\left(\max_i \frac{x_i^2}{t} +4\right)
\end{align*}
where the last step uses the Bhatia-Davis bound, which says that for a random variable $X$ bounded between $[a,b]$ with mean $\mu$ we have $\Var(X)\leq (b-\mu)(\mu-a) $. Dividing both sides by $\mu$, one may verify that the right-hand side is maximized at $\mu = \sqrt{ab}$, which yields  
$$\Var[X] / \mu \leq (b-\mu)(\mu-a) / \mu \leq (\sqrt{b} -\sqrt{a})^2\leq b-a.$$
{%
For the mixed space--time error, the identity $\sum_i\partial_x\phi(x_i,t)\Delta x_i=0$
and Lemma~\ref{lem:derivatives} give
\begin{align*}
\MixErr_\Phi
&=-\frac{\sum_i(x_i^3/t^3)\phi(x_i,t)\Delta x_i}
       {\sum_i\partial_{xx}\phi(x_i,t)}
  &&\text{since $\sum_i(3x_i/t^2)\phi(x_i,t)\Delta x_i=0$}\\
&=-\E_{i\sim\bq}\left[\frac{x_i^3}{t(t+x_i^2)}\Delta x_i\right].
\end{align*}
Taking absolute values, using $|\Delta x_i|\le B$ and
$|x_i^3/[t(t+x_i^2)]|\le |x_i|/t\le\sqrt{\frac{1}{t}\max_j\frac{x_j^2}{t}}$ proves the claim.
}
\end{proof}

\subsection{Proof of the crude bound on $\Delta t$}\label{sec:proof_of_crude_bound}

\begin{lemma}[Restating Lemma~\ref{lem:crude_Delta_t_bound}]
	If $K > 0$ is such that $\max_{i\in[N]}\frac{x_i^2}{t} \leq K$, then
    \begin{equation*}
        t \geq  256 e^2 B^2\max\{K,1\} \implies \Delta t  \leq 2e B^2.
    \end{equation*}
\end{lemma}

\begin{proof}[Proof of Lemma~\ref{lem:crude_Delta_t_bound}]
  For this proof, we want to find the smallest $\sigma \ge 0$ such that
\[
\Phi(\bx + \Delta \bx,\, t + \sigma) \le \Phi(\bx, t).
\]
Since $\Phi$ is continuous,  such a $\sigma$ that makes the inequality an equation must necessarily equal $\Delta t$ by uniqueness (Lemma~\ref{lem:nonegativity}). Therefore, if we find \emph{any} $\sigma > 0$ that make the above inequality hold, we have that $\Delta t \leq \sigma$. Our claim is that $\sigma = 2 e B^2$ works for $t$ as in the lemma. Therefore, for the rest of this proof we will slightly overload the notation of $\Delta t$, considering an arbitrary $\Delta t > 0$ and showing that if $\Delta t = 2 e B^2$, then $\Phi(\bx + \Delta \bx,\, t + \Delta t) \le \Phi(\bx, t)$.
Accordingly, $\TimeErr_\Phi$ is evaluated at this candidate increment:
$\TimeErr_\Phi(\bx,t;\sigma)$ before the relabeling and
$\TimeErr_\Phi(\bx,t;\Delta t)$ afterwards. The spatial increment in
$\MixErr_\Phi$ remains $\Delta\bx$.

Let us first obtain a bound on the local-self-concordance parameter.
	At $(\bx,t)$, by hypothesis we have $\max_{i\in[N]}\frac{x_i^2}{t} \leq K$.
	Thus, since $\|\Delta \bx\|_\infty \leq B$, for every $i\in [N]$ we have
	$$
\frac{(x_i+\Delta x_i)^2}{t + \Delta t} \leq \frac{2x_i^2}{t + \Delta t} + \frac{2 \Delta x_i^2}{t+\Delta t} \leq 2K + \frac{2B^2}{t}.
$$
Define $(\bx(s),t(s)) = (\bx + s\Delta \bx, t + s\Delta t)$ for $s\in[0,1]$ to be the line segment between $(\bx,t)$ and $(\bx',t')$.	
It is clear that 
$t_\star:=\min_{s\in[0,1]} t(s)=t$, and  by the joint-convexity of $x^2 /t$ on $\R\times \R_+$, we have 
$$
K_{\rm seg}:=\sup_{s\in[0,1]}\ \max_{i\in[N]} \frac{x_i(s)^2}{t(s)} = \max_{s\in\{0,1\}} \max_{i\in[N]} \frac{x_i(s)^2}{t(s)}  \leq 2K + \frac{2B^2}{t }.
$$
Furthermore, if $t \geq  2B^2$ we can simplify the bound to $K_{\rm seg}\vee 1 \leq 2K+1$ .

Therefore, we can bound the self-concordance parameters of the NormalHedge potential from Proposition~\ref{prop:localGSC_normalhedge} by
\begin{equation}
	\label{eq:AxAt2}
A_x(t_\star,K_{\rm seg}) \leq 8\sqrt{(2K+1)/t} =: \overline{A}_x, \quad \quad A_t(t_\star,K_{\rm seg}) \leq 16(2K+1)/t) =: \overline{A}_t.
\end{equation}  
These bounds will be used later to show that $e^{\Lambda}$ is small.

By Lemma~\ref{lem:taylor_theorem_log_total_potential}  (Taylor's theorem of $\log \Phi(\bm x',t')$ around $(\bx,t)$), we have 
\begin{align*}	
\log \Phi(\bm x',t') - \log \Phi(\bm x,t) =- \frac{\sum_{i=1}^N\frac{\partial^2}{\partial x^2} \phi(x_i,t) \Delta t}{2\Phi(\bm x,t)}    +\frac{1}{2} \nabla^2 \log \Phi(\bar{x},\bar{t})[\Delta,\Delta]
\end{align*}
for some $(\bar{\bx},\bar{t})$ on the line segment between $(\bx,t)$ to $(\bx',t')$.

	Define the shorthand  $\overline{\Lambda} :=  \overline{A}_x \|\Delta \bx\|_\infty +  \overline{A}_t |\Delta t|$, generalized local self-concordance of the potential (Proposition~\ref{prop:localGSC_normalhedge}), we can upper-bound the Hessian of $\LogPot$ at $\bar{\bx}$ by Corollary~\ref{cor:sandwich_normal}, which yields
	\begin{align}
		&\log \Phi(\bm x',t') - \log \Phi(\bm x,t)  + \frac{1}{2}\frac{\sum_{i=1}^N\frac{\partial^2}{\partial x^2} \phi(x_i,t) }{\Phi(\bm x,t)}  \Delta t 
		=   \frac{1}{2} \nabla^2 \log \Phi(\bar{x},\bar{t})[\Delta,\Delta]\nonumber\\
		\leq& \frac{1}{2}e^{\bar\Lambda}  \nabla^2 \log \Phi(\bm x,t)[\Delta,\Delta], \tag{By Proposition~\ref{prop:localGSC_normalhedge}}. \nonumber
    \end{align}

    By following similar calculations to the ones in the proof of Lemma~\ref{lemma:delta_t_discretization_error}, one can show that
    \begin{align*}
        &\frac{\Phi(\bm x,t) (\log \Phi(\bm x',t') - \log \Phi(\bm x,t))}{\sum_{i=1}^N\frac{\partial^2}{\partial x^2}\phi(x_i,t) }+\frac{\Delta t}{2}\\
        &\quad\leq\frac12e^{\overline\Lambda}
        \left(\E_{i\sim\bq}[\Delta x_i^2]
        +\Delta t\left(\blue{\MixErr_\Phi}+\TimeErr_\Phi\right)\right).
    \end{align*}
    \blue{Lemma~\ref{lem:bahia-davis} gives
    $\TimeErr_\Phi\le\Delta t(K+4)/(4t)$ and
    $|\MixErr_\Phi|\le B\sqrt{K/t}$, since $\max_i x_i^2/t\le K$.}
    Therefore, putting everything together yields
    \begin{equation*}
        \frac{\Phi(\bm x,t) (\log \Phi(\bm x',t') - \log \Phi(\bm x,t))}{\sum_{i=1}^N\frac{\partial^2}{\partial x^2}\phi(x_i,t) }  + \frac{\Delta t}{2} 
        \leq \frac{1}{2}e^{\overline\Lambda} \left( 
    \E_{i\sim \bq} [\Delta x_i^2] + \blue{B\sqrt{K/t}\,\Delta t}
    + \frac{\Delta t^2}{4t} (K + 4)\right)
    \end{equation*}

	Thus, notice that in order for $\log \Phi(\bm x',t') - \log \Phi(\bm x,t)\leq 0$, it suffices for us to choose
\begin{equation}\label{eq:circular_inequality}
\Delta t \geq   e^{\overline\Lambda}   \E_{i\sim \bq} [\Delta x_i^2] +   e^{\overline\Lambda} \left(\blue{B\sqrt{K/t}\,\Delta t}+\frac{K+4}{4t}\Delta t^2\right).
\end{equation}
	 Since $\max_i|\Delta x_i|\leq B$, it suffices that
    \begin{equation}\label{eq:crude_circular_inequality}
			 \Delta t \geq   e^{\bar\Lambda}  (B^2 + \blue{B\sqrt{K/t}\,\Delta t} + \frac{K+4}{4t} \Delta t^2).
	\end{equation}

    We just need to show now that plugging $\Delta t = 2eB^2$ into \eqref{eq:crude_circular_inequality} satisfies the inequality.

    We claim that if we choose $ \Delta t = 2 e  B^2 $ and $t \geq 256 e^2 B^2\max\{K,1\}$ (both hypotheses of the lemma we are proving), then the inequality \eqref{eq:crude_circular_inequality} is true. 
    
    Thus, let us show that \eqref{eq:crude_circular_inequality} holds. First, to bound $\overline{\Lambda}$ we use that $\Delta t =  2 e  B^2 $ and that $t \geq 256 e^2 B^2\max\{K,1\}$ in the formulas of \(\overline{A}_x, \overline{A}_t\) in \eqref{eq:AxAt2} to get 
	\begin{align*}
	    \overline \Lambda & = \overline{A}_x \|\Delta \bx\|_\infty +  \overline{A}_t |\Delta t|
        \leq \frac{8\sqrt{2K+1}B}{\sqrt{t}} + \frac{16(2K+1) 2eB^2}{t}\\
        &\leq \blue{\frac{\sqrt{2K+1}}{2e\sqrt{\max\{K,1\}}}}
        + \frac{2(2K+1)}{16\max\{K,1\} e}
        \blue{{}\leq\frac{\sqrt3}{2e}+\frac{3}{8e}}\leq 1.
	\end{align*}
	Moreover, we have
    {%
    \begin{align*}
    B^2+B\sqrt{K/t}\,\Delta t+\frac{K+4}{4t}\Delta t^2
    &\leq B^2+\frac{2eB^2}{16e}
       +\frac{K+4}{4\times256e^2B^2\max\{K,1\}}\times4e^2B^4\\
    &\leq\left(1+\frac18+\frac5{256}\right)B^2
     =\frac{293}{256}B^2\leq2B^2.
    \end{align*}
    }

 \comment{   

    \vspace{2cm}
	We claim that for $\Delta t \geq  2 e  B^2 $ and $t \geq 256 e^2 B^2\max\{K,1\}$ (both hypotheses of the lemma we are proving), then the inequality \eqref{eq:crude_circular_inequality} is true. 
    
    Thus, let us show that \eqref{eq:crude_circular_inequality} holds. First, to bound $\overline{\Lambda}$ we use that $\Delta t \geq  2 e  B^2 $ and that $t \geq 256 e^2 B^2\max\{K,1\}$ in the formulas of \(\overline{A}_x, \overline{A}_t\) in \eqref{eq:AxAt2} to get 
	\begin{align*}
	    \overline \Lambda & = \overline{A}_x \|\Delta \bx\|_\infty +  \overline{A}_t |\Delta t|
        \leq \frac{8\sqrt{2K+1}B}{\sqrt{t}} + \frac{16(2K+1) 2eB^2}{t}\\
        &\leq \sqrt{\frac{2K+1}{8\max\{K,1\}e^2}} + \frac{2(2K+1)}{16\max\{K,1\} e} \leq 1. %
	\end{align*}
	Moreover, we have 
    	$$B^2 +  \frac{K+3}{4t} \Delta t^2 \leq B^2 + \frac{K+3}{4\times 256 e^2 B^2\max\{K,1\}} \times 4e^2B^4 \leq  2B^2.$$
    }
It follows that  
	$$
	\textsc{RHS of \eqref{eq:crude_circular_inequality}}  \leq e^{\overline \Lambda} (2B^2)\leq  2eB^2 \leq \Delta t.
	$$
	This checks that our choice of $\Delta t$ is valid for \eqref{eq:crude_circular_inequality}, hence implies that for this choice $\log \Phi(\bm x',t') - \log \Phi(\bm x,t)\leq 0$, which completes the proof for Lemma~\ref{lem:crude_Delta_t_bound}.
	\end{proof}

\subsection{Bounding $\Kseg$ and self-concordance parameters.}\label{sec:proof_circular}
Now, we will use the fact that $\Delta t$ is algorithmically chosen (see Line 4 of the \CP~algorithm) to ensure a non-increasing potential $\Phi$.  First of all, observe that for every iterate  $(\bx_j,t_j)$ that the algorithm generates, $\Phi(\bm x_j,t_j)  \leq  \Phi(0,t_0) = \frac{N}{\sqrt{t_0}}$.   Moreover, we have the following lemmas.

\begin{lemma}
	If $\Phi(\bm x,t) \leq C$ and $\Phi(\bm x',t')\leq C$, then for any $(\bar{\bx},\bar{t})$ on the line-segment between $(\bx,t)$ and $(\bx',t')$ satisfies $\Phi(\bar{\bx},\bar{t})\leq C.$
\end{lemma}
\begin{proof}
First, observe that $\phi$ is convex (the quadratic over linear function on $\R\times \R_+$). The function $\Phi$ is thus jointly convex in $x,t$ (the sum of convex functions is convex). It follows that the univariate function when we restrict $\Phi$ to the line segment between $(\bx,t),(\bx',t')$ is convex. Lastly, the maximum of a convex function in a closed interval occurs on the boundary. 
\end{proof}
\begin{lemma}
	If $\Phi(\bm x,t)\leq C$ then $\max_{i} \frac{|x_i|^2}{t} \leq \blue{\log t + 2\log C}$.
\end{lemma}
\begin{proof}
By the definition and the non-negativity of $\phi(x_i,t)$ for $t>0$, for any $i\in[N]$ we have
$$
\frac{1}{\sqrt{t}} \exp(x_i^2 / 2t)=\phi(x_i,t) \leq \sum_{i=1}^N \phi(x_i,t) = \Phi(\bm x,t) \leq C.
$$
The proof is complete by taking log on both sides.
\end{proof}

These two lemmas together allow us to obtain a great bound of the local generalized self-concordance parameter of $\log \Phi$.  
Specifically, the algorithm maintains the invariant that for any regret vector $\bx$ with time variable $t$ in the algorithm, 
\begin{equation}
	\label{eq:cp_Kseg_invariant}
	\blue{\max_i\frac{x_i^2}{t}\leq\log(t/t_0)+2\log N
    \leq\max\{1,\log(t/t_0)+2\log N\}\eqqcolon K(t).}
\end{equation}
Notice that this quantity is (mildly) growing in $O(\log t)$. 
Then, for $\Kseg$ define as in \eqref{eq:seg-controls} we have
\begin{equation}
	\label{eq:cp_Kseg_invariant_2}
\Kseg = \sup_{s\in[0,1]}\ \max_{i\in[N]} \frac{x_i(s)^2}{t(s)} \stackrel{\eqref{eq:cp_Kseg_invariant}}{\leq} \log (\max\{t,t'\} / t_0) + \blue{2\log N} \leq \log((t+\Delta t) / t_0) + 2\log N.
\end{equation}

The next lemma simply shows that, since $K(t)$ grows slowly in $t$, if we have $t_0 \geq C K(t_0)$, then we have $t \geq C K(t)$ for all $t \geq t_0$.

\begin{lemma}%
\label{lem:deal_with_logt_inK}
	Define $K(t) :=  \max\left\{ 1,  \log(t / t_0) + 2\log N\right\}$. 
	Assume $t_0 \geq 256 e^2 B^2K(t_0)$, then for any $t \blue{{}\geq{}} t_0$,
	\begin{equation}
		\label{eq:bound_on_Kt}
	t \geq 256 e^2 B^2K(t).
	\end{equation}
\end{lemma}
\begin{proof}
{%
For $t\ge t_0$, using $\log(t/t_0)\le t/t_0-1$ and $K(t_0)\ge1$ gives
\[
\log(t/t_0)+2\log N\le t/t_0-1+K(t_0)\le(t/t_0)K(t_0).
\]
Since also $1\le(t/t_0)K(t_0)$, we have $K(t)\le(t/t_0)K(t_0)$.
Multiplying by $256e^2B^2$ and using the assumption proves the claim.
}
\end{proof}

The next lemma summarizes the above discussion and provides a tighter bound on the self-concordance parameters of the NormalHedge potential for all $t\geq t_0$, which will be useful in the next section to get a refined bound on $\Delta t$.
\begin{lemma}\label{lem:bound_on_Lambda}
	Let $t_0 \blue{{}\geq{}} 256 e^2 B^2 \max\left\{1, 2\log N\right\}$. Moreover, let \((\bx, t)\) and \((\bx', t') = (\bx + \Delta x, t + \Delta t)\) be such that $t \geq t_0$ and \(\Phi(\bx', t') \leq \Phi(\bx,t) \leq \Phi(0, t_0)\), \blue{with $\|\Delta\bx\|_\infty\le B$ and $0\le\Delta t\le2eB^2$ as supplied by the crude bound}. Finally, let $A_x$, $A_t$ be as in \blue{\eqref{eq:AxAt} in} Proposition~\ref{prop:localGSC_normalhedge}. Then
	\begin{equation*}
		\Lambda:=A_x(t_\star,K_{\rm seg})\,\|\Delta \bm x\|_\infty
	+ A_t(t_\star,K_{\rm seg})\,|\Delta t| \leq 0.414.
	\end{equation*}
\end{lemma}
\begin{proof}
Since $\Phi(\bm x',t') \leq \Phi(0,t_0) = \frac{N}{\sqrt{t_0}}$, we have seen (see \eqref{eq:cp_Kseg_invariant_2}) that
$$
\Kseg(t) \stackrel{\eqref{eq:cp_Kseg_invariant_2}}{\leq}  \log((t+\Delta t) / t_0) + 2\log N  \leq  \log((t+ 2eB^2) / t_0) + 2\log N  \leq \log(1 + t / t_0) + 2\log N,
$$
where the last inequality follows from our choice $t_0 \blue{{}\geq{}} 256 e^2 B^2 \max\left\{1, 2\log N\right\} \geq  2 e B^2$. We want to upper bound the above by $K(t)$ defined as in Lemma~\ref{lem:deal_with_logt_inK}. The above expression depends on $\log(1 + t/t_0)$, while $K(t)$ depends on \(\log(t/t0)\), but fortunately they are not too far from each other. Indeed, notice that
\begin{equation*}
	\log\left(1 + \frac{t}{t_0}\right) =\log\left(\frac{t}{t_0}\left(\frac{t_0}{t} + 1\right)\right) = \log(\frac{t}{t_0}) + \log\left(\frac{t_0}{t} + 1\right) \leq \log\left(\frac{t}{t_0}\right) + \log 2 
\end{equation*} 
Therefore,
\begin{equation*}
	\Kseg \leq \log(1 + t / t_0) + 2\log N \leq \log(t / t_0) + 2\log N + \log 2 \leq K(t) + \log 2 \leq 2 K(t).
\end{equation*}
Therefore, we can upper bound $A_x(t_\star,K_{\rm seg})$ and $A_t(t_\star,K_{\rm seg})$ in terms of $K(t)$ as follows:
$$
A_x(t_\star,K_{\rm seg}) \leq \frac{8\sqrt{K_{\rm seg}\vee 1}}{\sqrt{t_\star}}\leq 8\sqrt{\frac{ 2 K(t)}{ t}},
$$
and
\begin{equation}
	\label{eq:interm_at_bound}
A_t(t_\star,K_{\rm seg})  \leq \frac{16 (K_{\rm seg}\vee 1)}{t_*}\leq  \frac{32 K(t) B^2}{tB^2}.
\end{equation}
Since $t\geq t_0$, Lemma~\ref{lem:deal_with_logt_inK} relates \(K(t)\) and \(t\),  we have
$$
\eqref{eq:interm_at_bound} = \sqrt{\frac{32 K(t) B^2}{t}} \sqrt{\frac{32 K(t) B^2}{t}} \frac{1}{B^2} \stackrel{\eqref{eq:bound_on_Kt}}{\leq} \frac{1}{e} \sqrt{\frac{8 K(t) B^2}{t}} \frac{1}{B^2}. 
$$
Therefore,
\begin{align*}
	\Lambda &= A_x(t_\star,K_{\rm seg})\,\|\Delta \bm x\|_\infty
	+ A_t(t_\star,K_{\rm seg})\,|\Delta t|
	\\
	& \leq 8\sqrt{\frac{2 K(t)}{ t}} \cdot  B
	+ \frac{1}{e} \sqrt{\frac{8 K(t) B^2}{t}} \frac{1}{B^2} \cdot 2 e B^2
	\\
	&\leq 8 \sqrt{\frac{2 K(t)}{t}} \cdot B + 2  \sqrt{\frac{8 K(t) }{t}} B
	\\ 
	&\leq 17 B \sqrt{\frac{K(t)}{ t}}
	 \stackrel{\eqref{eq:bound_on_Kt}}{\leq} 17B \sqrt{\frac{1}{256e^2 B^2}} \leq \frac{17	}{16 \cdot e} \leq 0.414.
\end{align*}

\end{proof}

 \subsection{A Refined Bound of $\Delta t$ and Proof of Theorem~\ref{thm:main}}

  Now we are ready to obtain a stronger second-order bound of $\Delta t$, which ultimately implies the regret bound in Theorem~\ref{thm:main}.  

  \begin{proof}[Proof of Theorem~\ref{thm:main}]
	In Lemma~\ref{ex:normalhedge} we established that, since the potential does not increase in $\CP$, if \((\bx^{(T)}, t^{(T)})\) is the final iterate of the \CP~algorithm after $T$ rounds, then for any $\epsilon\in(0,1)$, the regret to the top $\epsilon$-quantile of experts satisfies
	$$
				x_{(N\epsilon)} \leq \sqrt{t_T (\log (t_T/t_0) + 2\log(1/\epsilon))}.
	$$
	It suffices that we bound $t_T$. For that, we shall bound all the increments $\Delta t_j\coloneqq t_j - t_{j-1}$ for $\blue{j=1,\ldots,T}$.

	\blue{By Lemma~\ref{lem:reduction_delta_t}, it suffices to bound the unprojected equalizing increment: take $(\bx,t)$ at the start of a round, retain the regret increment $\Delta\bx$, and choose $\Delta t\ge0$ so that $(\bx',t')=(\bx+\Delta\bx,t+\Delta t)$ satisfies $\Phi(\bx,t)=\Phi(\bx',t')$. The actual projected clock increment is no larger.} Thus, Lemma~\ref{lemma:delta_t_discretization_error} and the local self-concordance of the NormalHedge potential (Proposition~\ref{prop:localGSC_normalhedge})  give
	\begin{equation}
		\label{eq:inter_step_main_thm}
		\begin{aligned}
		\Delta t &\leq e^{\Lambda}\left( 
 \E_{i\sim\bq}[\Delta x_i^2]
 +\Delta t\left(\blue{\MixErr_\Phi}+\TimeErr_\Phi\right),\right)
		\end{aligned}
	\end{equation}
	 where \(\bq \in \Delta^{N-1}\) is such that \(q_i \propto \partial_{xx} \phi(x_i, t)\). Let us first bound \(\MixErr_{\Phi} + \TimeErr_{\Phi}\). By Lemma~\ref{lem:bahia-davis}, we have
	\begin{equation}
		\label{eq:appx-two-errors-summed}
		\MixErr_{\Phi} + \TimeErr_{\Phi} \leq  B\sqrt{K(t)/t} + \frac{\Delta t}{4t} \left(\max_i \frac{x_i^2}{t} + 4\right).
	\end{equation}
	Moreover, since \(\Phi(\bx, t) \leq \Phi(0, t_0) = \blue{N/\sqrt{t_0}}\), if \(K(t)\) is defined as in Lemma~\ref{lem:deal_with_logt_inK}, then
	\begin{equation*}
		\max_i \frac{x_i^2}{t} \leq \log(t / t_0) + \blue{2\log N} \leq K(t).
	\end{equation*}
	The case $N=1$ has zero regret, so we can assume $N\ge2$. In this case, the choice of $t_0$ gives $t_0\ge256e^2B^2K(t_0)$, so
    Lemma~\ref{lem:deal_with_logt_inK} gives $t\ge256e^2B^2K(t)$.
    Hence the rough bound on \(\Delta t\) from Lemma~\ref{lem:crude_Delta_t_bound} yields \(\Delta t \leq 2 e B^2\). Therefore,
	\begin{align}
		\eqref{eq:appx-two-errors-summed} &\leq 
		B\sqrt{K(t)/t} + \frac{2 e B^2}{4t} \big(\underbrace{K(t) + 4}_{\leq 5 K(t)}\big) \nonumber
		\\
		 & \leq \blue{\frac{1}{16e}}+\frac{4eB^2}{256e^2B^2} \tag{by Lemma~\ref{lem:deal_with_logt_inK}}
		 \leq \frac{5}{64 e}.
	\end{align}
	We can plug this bound on the errors back in~\eqref{eq:inter_step_main_thm}, which yields 
	\begin{equation*}
	 \Delta t \leq e^{\Lambda} \left(
 \E_{i\sim\bq}[\Delta x_i^2]
 +\frac{5 \Delta t}{64e}\right)
	\end{equation*}
	It only remain to get a bound on the self-concordance parameter \(\Lambda\). Since \(\Delta t \leq 2 e B^2\) from our crude bound on \(\Delta t\) (Lemma~\ref{lem:crude_Delta_t_bound}) and since, by our choice of \(t_0\), we have \(t_0 \geq 256 e^2 B^2 K(t)\), we can use Lemma~\ref{lem:bound_on_Lambda} to conclude that
 $\Lambda\le0.414$. Finally, this yields
 \begin{equation*}
	 \Delta t \leq e^{0.414} \left(
 \E_{i\sim\bq}[\Delta x_i^2]
 +\frac{5 \Delta t}{64e}\right)
	\end{equation*}

Grouping $\Delta t$ terms on the left hand side and using that
\blue{$e^{-0.414}-\frac{5}{64e}>0.63>0.5$}, we get
\begin{equation}\label{eq:refined_Delta_t_bound}
	\Delta t \leq 2\E_{i\sim \bq}\left[\Delta x_i^2\right].
\end{equation}
	 It is worth remembering that the \(\Delta t\) we bounded above was using  the unprojected \(\Delta x\). \blue{Yet, by Lemma~\ref{lem:reduction_delta_t}, the same bound holds for the actual clock increments.}
    Let us apply the above bound to all the increments \(\Delta t_1, \Delta t_2, \dotsc, \Delta_T\) from \CP{}. Namely, by the description of \CP{} we have
	$$
	t = t_0 + \sum_{j=1}^T \Delta t_j \leq t_0 + 2 \sum_{j=1}^T \E_{i\sim \bq_{j}} [\blue{(\Delta x_{j,i})^2}] = \blue{t_0+2V_T}.
	$$
	Therefore, the final regret bounds is
    $$
    \sqrt{(t_0 + 2V_T) \left( \log(\frac{t_0 + 2V_T}{t_0})+2\log(1/\epsilon)\right)}.
    $$
  \end{proof}

	\begin{remark}[Special case when $B\rightarrow 0$]
		 For any finite $N$, choosing $t_0=1$, and $B\rightarrow 0$ allows us to simplify the regret bound into 
		 $$
		 \sqrt{(1 + V_T) (\log (1 + V_T) + 2\log(1/\epsilon))}.
		 $$
		 where $T$ is the first index such that $t > \tau$.  The factor of $2$ improvement on the constant in front of $V_T$ is obtained by using the stronger bound in \eqref{eq:refined_Delta_t_bound} by taking $B\rightarrow 0$.  This is slightly stronger than taking $B\rightarrow 0$ in \eqref{eq:regret_additive} as stated in Theorem~\ref{thm:main} (improved from $\log(1+2V_T)$ to $\log(1+V_T)$).
	\end{remark}

\section{Improved bound with asymptotically tight constant.}\label{sec:improved_bound}

\begin{theorem}[Improved bound]\label{thm:improved_bound}
    Under the same assumptions, \CP{} also satisfies a bound that improves the constant factor by $\sqrt{2}$ as $V_T \to \infty$:
    \begin{equation}\label{eq:regret_additive}
        \mathrm{Regret}_\epsilon(T) \leq 
        \sqrt{ (t_0 + V_T + \iota B\sqrt{V_T})( \log (t_0 + 2V_T) + 2\log(1/\epsilon))  }
    \end{equation}
    where $\iota := 144 \max\left\{1,\log(t_0 + 2V_T)+2\log N\right\} = \tilde{O}(1)$.
\end{theorem}
\begin{proof}
We can get a bound of the type 
$$
    \sqrt{\left(t_0 + V_T + \tilde{O}(\sqrt{V_T}) \right) \left( \log(t_0 + V_T)+2\log(1/\epsilon)\right)}.
$$

In the limit when $V_T\rightarrow \infty$ or $B\rightarrow 0$ and $ \log(1/\epsilon)\gg \log(V_T)$, then bound converges to the exact asymptotic limit (even for the leading constant) of $\sqrt{2V_T \log(1/\epsilon)}$. 

From \eqref{eq:Delta_t_handy_expression} and our crude bound $\Delta t \leq 2eB^2$, we have 
$$
\blue{\Delta t \leq
\frac{e^\Lambda}{1-e^\Lambda\left(B\sqrt{K(t)/t}+4eB^2K(t)/t\right)}}
\E_{i\sim \bq} [\Delta x_i^2].
$$ 
\blue{Here we first consider the unprojected equalizing increment, as in the proof above.
The denominator is positive because
$e^{-\Lambda}-B\sqrt{K(t)/t}-4eB^2K(t)/t
\ge e^{-0.414}-5/(64e)>0.63$.}

This is of the form: 
$$\Delta t \leq \frac{1+\epsilon_1}{1-\epsilon_2} \E_{i\sim \bq} [\Delta x_i^2].$$
\blue{with $\epsilon_1=e^\Lambda-1$ and
$\epsilon_2=e^\Lambda\left(B\sqrt{K(t)/t}+4eB^2K(t)/t\right)$.}
Since $\epsilon_2 < 1$ we can move things around and obtain
$$
\frac{1-\epsilon_2}{1+\epsilon_1} \Delta t\leq \E_{i\sim \bq} [\Delta x_i^2],
$$
which gives the following additive bound
$$
\Delta t\leq \E_{i\sim \bq} [\Delta x_i^2]  + \frac{\epsilon_1 + \epsilon_2}{1+\epsilon_1}\Delta t,
$$
where the coefficient \blue{is bounded as follows:}
{%
\begin{align*}
\frac{\epsilon_1+\epsilon_2}{1+\epsilon_1}
&=1-e^{-\Lambda}+B\sqrt{K(t)/t}+4eB^2K(t)/t\\
&\le\Lambda+B\sqrt{K(t)/t}+4eB^2K(t)/t
  &&\text{since $1-e^{-\Lambda}\le\Lambda$}\\
&\le\left(17+1+\frac14\right)B\sqrt{K(t)/t}\\
&\le\frac{36B\sqrt{K(t)}}{\sqrt t}.
\end{align*}
The penultimate step uses $\Lambda\le17B\sqrt{K(t)/t}$ and
$B\sqrt{K(t)/t}\le1/(16e)$.
Since $(1-\epsilon_2)/(1+\epsilon_1)>0$, its inequality also holds for the
smaller projected clock increment by Lemma~\ref{lem:reduction_delta_t}.
Thus the additive bound holds for the actual increments as well.
}
It follows that
\begin{align}
    t &= t_0 + \sum_{j=1}^T\Delta t_j \blue{{}\leq{}} t_0 + V_T + \sum_{j=1}^T
    \blue{\frac{36B\sqrt{K(t_{j-1})}\Delta t_j}{\sqrt{t_{j-1}}}}\nonumber\\
    &\leq t_0 + V_T + 36B\sqrt{K(t)}\sum_{j=1}^T \frac{\Delta t_j}{\blue{\sqrt{t_0+\sum_{\ell=1}^{j-1}\Delta t_\ell}}} \nonumber\\
    &\leq t_0 + V_T + 72B\sqrt{K(t)}\sqrt{1 + \frac{2eB^2}{t_0}}\sqrt{2V_T}\nonumber\\
    &\leq t_0 + V_T + 144B\sqrt{K(t)}\sqrt{V_T}.
\end{align}
The second last line follows from the following lemma that bounds a shifted self-normalizing series in Lemma~\ref{lem:self_normalizing_series}, with the weak bound in \eqref{eq:refined_Delta_t_bound} with a factor of $2$ allows us to bound $\sum_{j}\Delta t_j \leq 2V_T$. 

The last line uses our assumption on $t_0\geq 256e^2B^2K(t_0) \geq 2eB^2$.
\blue{Since $K(t)\ge1$ and $t_0\ge1$,} recall that
$$\blue{\sqrt{K(t)}\le{}}K(t) \leq \max\left\{1, \log(t_0 + 2V_T)+2\log N\right\} = \tilde{O}(1).$$
This proves the additive regret bound of the form
$$
\sqrt{ (t_0 + V_T + \iota\blue{B}\sqrt{V_T})( \log (t_0 + 2V_T) + 2\log(1/\epsilon))  }
$$
where $\blue{\iota := 144\max\left\{1,\log(t_0+2V_T)+2\log N\right\}
=O(1+\log(t_0+V_T)+\log N)}$, as in the theorem statement.
This completes the proof for \eqref{eq:regret_additive}.
\end{proof}

\begin{lemma}[Shifted self-normalizing bounds]\label{lem:self_normalizing_series}
Let $\Delta t_1,\dots,\Delta t_T \ge 0$, let $t_0>0$, and define
\[
t_j := t_0 + \sum_{i=1}^j \Delta t_i \qquad (j=1,\dots,T).
\]
Assume there exists $M>0$ such that $\blue{\Delta t_j}\le M$ for all $j$.
Then
\[
\sum_{j=1}^T \frac{\Delta t_j}{\sqrt{t_{j-1}}}
\;\le\;
2\sqrt{1+\frac{M}{t_0}}\;\bigl(\blue{\sqrt{t_T}}-\sqrt{t_0}\bigr).
\]
\end{lemma}

\begin{proof}[Proof of Lemma~\ref{lem:self_normalizing_series}]
Fix $j\in\{1,\dots,T\}$. Using $t_j=t_{j-1}+\Delta t_j$,
\[
\sqrt{t_j}-\sqrt{t_{j-1}}
=
\frac{t_j-t_{j-1}}{\sqrt{t_j}+\sqrt{t_{j-1}}}
=
\frac{\Delta t_j}{\sqrt{t_j}+\sqrt{t_{j-1}}}.
\]
Rearranging gives
\[
\Delta t_j
=
\bigl(\sqrt{t_j}-\sqrt{t_{j-1}}\bigr)\bigl(\sqrt{t_j}+\sqrt{t_{j-1}}\bigr),
\]
and therefore
\[
\frac{\Delta t_j}{\sqrt{t_{j-1}}}
=
\bigl(\sqrt{t_j}-\sqrt{t_{j-1}}\bigr)\,
\frac{\sqrt{t_j}+\sqrt{t_{j-1}}}{\sqrt{t_{j-1}}}
=
\bigl(\sqrt{t_j}-\sqrt{t_{j-1}}\bigr)\,
\left(1+\sqrt{\frac{t_j}{t_{j-1}}}\right).
\]
Next,
\[
\frac{t_j}{t_{j-1}}
=
1+\frac{\Delta t_j}{t_{j-1}}
\le
1+\frac{M}{t_0},
\]
since $t_{j-1}\ge t_0$ and $\Delta t_j\le M$. Hence
\[
1+\sqrt{\frac{t_j}{t_{j-1}}}
\le
1+\sqrt{1+\frac{M}{t_0}}
\le
2\sqrt{1+\frac{M}{t_0}},
\]
where the last inequality uses $1\le \sqrt{1+\frac{M}{t_0}}$.
Combining the previous displays,
\[
\frac{\Delta t_j}{\sqrt{t_{j-1}}}
\le
2\sqrt{1+\frac{M}{t_0}}\;\bigl(\sqrt{t_j}-\sqrt{t_{j-1}}\bigr).
\]
Summing over $j=1,\dots,T$ and telescoping yields
\[
\sum_{j=1}^T \frac{\Delta t_j}{\sqrt{t_{j-1}}}
\le
2\sqrt{1+\frac{M}{t_0}}\;\sum_{j=1}^n \bigl(\sqrt{t_j}-\sqrt{t_{j-1}}\bigr)
=
2\sqrt{1+\frac{M}{t_0}}\;\bigl(\sqrt{t_n}-\sqrt{t_0}\bigr),
\]
as claimed.
\end{proof}

\section{Proof of Theorem~\ref{thm:lower-bound} (Lower Bound)}
\label{sec:lbappendix}

Our lower bound shows that Theorem~\ref{thm:exponential-weights} is near optimal for every such sequences of loss scales.

After $T$ iterations, expert $i$ has loss $\sum_{j=1}^T {\bm \ell}_{j,i}$.
Since the $(\cdot)_{(\alpha N)}$ notation refers to the $(\alpha N)^{\text{th}}$ \emph{largest} value, the loss of the $(\eps N)^{\text{th}}$ best expert is $\big(\sum_{j=1}^T {\bm \ell}_j\big)_{(N-\floor{N\epsilon}+1)}$.
This leads to the bound
\begin{align*}
	&\min_{\mathrm{Alg}}\max_{\bm\ell_{1:T}\in \prod_{j=1}^T\{\pm \sigma_j\}^{N}}\mathrm{Regret}_\epsilon(T) \\
	\geq& \min_{\mathrm{Alg}} \E_{\bm\ell_{j}\sim \mathrm{Uniform}(\{-\sigma_j,\sigma_j\})^N }\left[\mathrm{Regret}_\epsilon(T) \right] \\
	=& \min_{\mathrm{Alg}} \E_{\bm\ell_{j}\sim \mathrm{Uniform}(\{-\sigma_j,\sigma_j\})^N }\left[  \sum_{j=1}^T\langle \bm \ell_j, \bm p_j\rangle  -  \left[\sum_{j=1}^T\bm \ell_j\right]_{(N-\floor{N\epsilon}+1)}\right]\\
	=&\min_{\mathrm{Alg}}   \E_{\bm\ell_{j}\sim \mathrm{Uniform}(\{-\sigma_j,\sigma_j\})^N }\left[  \sum_{j=1}^T\langle \bm \ell_j, \bm p_j\rangle \right] -   \E_{\bm\ell_{j}\sim \mathrm{Uniform}(\{-\sigma_j,\sigma_j\})^N }\left[\left[\sum_{j=1}^T\bm \ell_j\right]_{(N-\floor{N\epsilon}+1)}\right]
\intertext{Since the sequence $\big( \sum_{j=1}^t\langle \bm \ell_j, \bm p_j\rangle\big)_{t \geq 0}$ is a martingale, the first term disappears and we are left with}
	=&  \E_{\bm\ell_{j}\sim \mathrm{Uniform}(\{-\sigma_j,\sigma_j\})^N }\left[\left[\sum_{j=1}^T\bm \ell_j\right]_{(N\epsilon)}\right] 
\end{align*}
since the distribution on $\sum_{j=1}^T {\bm \ell}_j$ is symmetric.
Define 
\begin{align}\nonumber
\epsilon_0 &~=~ 1/40 \\\label{eq:LBQDef}
Q &~=~ (1/2) \sqrt{\sum_{j=1}^T \sigma_j^2 }
\end{align}
We will show that for any $\epsilon \in (0, \epsilon_0)$, with $T=T(\epsilon)$ as in \eqref{eq:chooseTeps} and $N=N(\epsilon)$ as in \eqref{eq:chooseNeps},
\[ \E_{\bm\ell_{j}\sim \mathrm{Uniform}(\{-\sigma_j,\sigma_j\})^N }\left[\left[\sum_{j=1}^T\bm
\ell_j\right]_{(N\epsilon)}\right] ~\geq~ 2Q \Big( \sqrt{2\log\smallfrac{1}{\epsilon}} - 6
\Big),
\]
which proves the theorem.

For each coordinate $i \in [N]$, define \[
X_i := \sum_{j=1}^T \ell_{j, i},
\]
where $\ell_{j, i} \in \{-\sigma_j, \sigma_j\}$.
By construction, $\{X_i\}_{i=1}^N$ are independent and satisfy $E[X_i] = 0$.
Our main objective is to prove that 
$\E\!\left[X_{(N\epsilon)}\right] = \Omega(Q\sqrt{\log(1/\epsilon)})$, which we will do in \eqref{eq:lbEX} and \eqref{eq:EXlastbound}.

Throughout this section, let $\Phi$ and $\phi$ denote the cumulative distribution function and probability density function of a standard normal random variable, respectively.
We then have the following lemma.

\begin{lemma}
\label{lem:Berry-Esseen}
    Fix any $i\in[N]$. Let $X_i$ be defined as above, where
$\{\ell_{j,i}\}_{j=1}^T$ are independent, mean-zero random variables satisfying
$\Pr(\ell_{j,i}=\sigma_j)=\Pr(\ell_{j,i}=-\sigma_j)=1/2$.
Recall that $\Phi$ denotes the cumulative distribution function of a standard normal random variable.
Then there exists an absolute constant $0 < C_0 < 1$ such that, letting $\delta := BC_0/(2Q)$, for all $x \in \mathbb{R}$,
\begin{equation}
\label{eq:Berry-Esseen}
\Pr\left[X_i \leq 2Qx\right] \leq \Phi(x) + \delta.
\end{equation}

\end{lemma}

\begin{proof}
    Let $F$ denote the cumulative distribution function of $X_i$, where \[
F(x) := \Pr[X_i \leq x].
\]

For each $j \in [T]$, $\ell_{j, i}$ takes values $\pm \sigma_j$ with equal probability.
Then a direct calculation yields \[
\hat{\sigma}_j^2 := \Exp[\ell_{j, i}^2] = \frac{1}{2}(-\sigma_j)^2 + \frac{1}{2}(\sigma_j)^2 = \sigma_j^2,
\qquad
\rho_j := \Exp|\ell_{j, i}|^3 = \frac{1}{2}|\sigma_j|^3 + \frac{1}{2}|\sigma_j|^3 = |\sigma_j|^3.
\]

Define \[
S_T := \frac{\sum_{j=1}^T \ell_{j, i}}{\sqrt{\sum_{j=1}^T \hat{\sigma}_j^2}}.
\]

Let $F_T$ denote the cumulative distribution function of $S_T$, and recall that $\Phi$ denotes the standard normal cumulative distribution function.
By the Berry--Esseen theorem
\citep{shevtsova2010improvement}, there exists an absolute constant $0 < C_0 < 1$ such that
\[
\sup_{x \in \R} |F_T(x) - \Phi(x)| \leq C_0 \cdot \left(\sum_{j=1}^T \hat{\sigma}_j^2\right)^{-3/2} \cdot \left(\sum_{j=1}^T \rho_j\right).
\]
Substituting the above expressions yields \begin{equation}
\label{eq:BEineq}
\sup_{x \in \R} |\Pr[S_T \leq x] - \Phi(x)| \leq C_0 \cdot \left(\sum_{j=1}^T \sigma_j^2\right)^{-3/2} \cdot \left(\sum_{j=1}^T |\sigma_j|^3\right).
\end{equation}

Since $\sigma_j \leq B/2$ for all $j \in [T]$, we can bound
\[
\sum_{j=1}^T |\sigma_j|^3 \leq (\max_j |\sigma_j|) \sum_{j=1}^T \sigma_j^2 \leq (B/2) \sum_{j=1}^T \sigma_j^2.
\]
Substituting this bound into \eqref{eq:BEineq} and recalling
from \eqref{eq:LBQDef} that \(Q=\frac{1}{2}\sqrt{\sum_{j=1}^T\sigma_j^2}\), we obtain \[
\sup_{x \in \R} |\Pr[S_T \leq x] - \Phi(x)| \leq \frac{BC_0}{2Q}.
\]

Let \(\delta := \frac{BC_0}{2Q}\). Then for all \(x\in\mathbb{R}\),
\[
\Pr[S_T \le x] \le \Phi(x) + \delta.
\]
Moreover, since \(S_T = X_i/(2Q)\), we have \(\Pr[X_i \le 2Qx] = \Pr[S_T \le x]\), and hence
\begin{equation*}
\Pr\left[X_i \leq 2Qx\right] \leq \Phi(x) + \delta.
\end{equation*}

\end{proof}

Recall that we assume \(\epsilon < \epsilon_0\). Since \(\sum_{j=1}^T \sigma_j^2 \to \infty\) as \(T\to\infty\) and \(\delta = \frac{BC_0}{2Q}\) (hence $\delta=\delta(T)$), we have $\delta\to 0$ as $T\to\infty$.
Therefore, there exists $T(\epsilon) \in \mathbb{N}$ such that
\begin{equation}
\label{eq:chooseTeps}    
\delta(T(\epsilon)) < \epsilon
\qquad \text{and} \qquad
\epsilon + 2\delta(T(\epsilon)) \le \epsilon_0.
\end{equation}

\begin{lemma}
    \label{lem:constructL}
    Recall that $\Phi$ denotes the cumulative distribution function of a standard normal random variable.
    Define 
    \begin{equation}
    \label{eq:defL}
L(\alpha) = \begin{cases}
        2Q \Phi^{-1}(1-\epsilon-\delta-\alpha)~&\text{if $\alpha \in [0, 1/2 - (\epsilon+\delta))$,}\\
        4Q \Phi^{-1}(1-\epsilon-\delta-\alpha)~&\text{if $\alpha \in [1/2 - (\epsilon+\delta), 1 - 2(\epsilon + \delta))$,}\\
        -2\sqrt{2}Q\sqrt{\log\left(\frac{1}{1-\epsilon-\alpha}\right)}~&\text{if $\alpha \in [1 - 2(\epsilon+\delta), 1 - \epsilon)$.}
    \end{cases}
\end{equation}
Then for $\alpha \in [0, 1-\epsilon)$, we have \begin{equation}
\label{eq:MedianTails}
\Pr[X_i \leq L(\alpha)] \leq 1 - \epsilon - \alpha.
\end{equation}
\end{lemma}

\begin{proof}
By \eqref{eq:Berry-Esseen}, we have \[
\Pr[X_i \leq 2Qx] \leq \Phi(x) + \delta.
\]

Moreover, for \(x\le 0\), applying \eqref{eq:Berry-Esseen} with \(2x\) in place of \(x\) yields
\[
\Pr[X_i \le 4Qx] \le \Phi(2x) + \delta \le \Phi(x) + \delta,
\]
where the last inequality uses the monotonicity of \(\Phi\) and the fact that \(2x\le x\).

We also invoke Hoeffding's inequality \cite[Theorem 2.2.2]{vershynin2018high}. Since
\(\ell_{j,i}\in[-\sigma_j,\sigma_j]\) and \(\mathbb{E}[\ell_{j,i}]=0\), for any \(s>0\),
\[
\Pr[X_i \leq -s] \leq \exp \left(-\frac{s^2}{2\sum_{j=1}^T\sigma_j^2}\right) = \exp\left(-s^2/(8Q^2)\right).
\]

Combining the above bounds with the definition of \(L(\alpha)\), we obtain
\begin{equation*}
\Pr[X_i \leq L(\alpha)] \leq 1 - \epsilon - \alpha.
\end{equation*}
\end{proof}

\begin{theorem}
Let $X_1,\ldots,X_N$ be independent random variables that satisfy \eqref{eq:MedianTails}.
Then, there exists 
\begin{equation}
\label{eq:defg}
    g(\alpha) = \begin{cases}
        \exp(-2\alpha^2 N)~&\text{if $\alpha \in [0, 1-2(\epsilon+\delta))$},\\
        (2(1-\epsilon-\alpha))^{N/8}~&\text{if $\alpha \in [1-2(\epsilon+\delta), 1-\epsilon)$}.
    \end{cases} 
\end{equation}
such that
for all $0 \leq \alpha < 1-\epsilon$, letting $L(\alpha)$ be as defined in Lemma~\ref{lem:constructL}, we have
\[
\Pr[\, X_{(N\epsilon)} \!\leq\! L(\alpha) \,]
    ~\leq~
        g(\alpha). \]
\end{theorem}

\begin{proof}
    Let $Z = \sum_{i=1}^N \Indicator{X_i \leq L}$.
    The event of interest may be restated as follows.    
    \begin{align*}
    \myset{\: X_{(N\epsilon)} \!\leq\! L(\alpha) \:}
    ~\subseteq~ \myset{\: \card{\setst{ i \in [N]}{ X_i \leq L(\alpha) }} \geq N-\floor{N\epsilon} \:}
    ~=~ \myset{\: Z \geq N-\floor{N\epsilon}\:}
    \end{align*}
    Define $\mu = N(1-\epsilon - \alpha)$.
    By \eqref{eq:MedianTails}, $\Pr[X_i \leq L(\alpha)] \leq 1-\epsilon - \alpha$, so $ \E[Z] ~\leq~ \mu$.
    By a Hoeffding bound, we have
    \begin{align*}
    \prob{ Z \geq N-\floor{N\epsilon} }
     ~\leq~ \prob{ Z \geq (1-\eps)N }
     ~\leq~ \prob{ Z \geq \mu + \alpha N }
     ~\leq~ \exp(-2\alpha^2 N).
    \end{align*}
    This bound is valid for all $\alpha$, but loose as $\alpha \rightarrow 1-\epsilon$.

    Let us now focus on the case $\alpha \geq 1-2(\epsilon+\delta)$.
    We will use a Chernoff bound, stated in the form
    \[
    \Pr[ Z \geq (1+\beta)\mu ]
        ~\leq~ \exp\big( - (1+\beta) \ln(1+\beta) \mu/4 \big)
        \qquad\forall\beta \geq 1.
    \]
    We will apply this with 
    \begin{equation}
    \label{eq:deltaLB}
    \beta = \frac{\alpha N}{\mu} 
     = \frac{\alpha}{1-\epsilon-\alpha}
     \geq \frac{1-2(\epsilon+\delta)}{\epsilon+2\delta}
     \geq \frac{1}{\epsilon+2\delta} - 2
     \geq 1,
    \end{equation}
    since $\epsilon + 2\delta \leq 1/4$.
    Observe that $\alpha \geq 1 - 2( \epsilon + \delta) \geq 1/2$, since $\epsilon+\delta \leq 1/4$.
    This yields
    \begin{alignat*}{2}
    \prob{ Z \geq N-\floor{N\epsilon} }
        &~\leq~ \Pr[ Z \geq \mu + \alpha N ] \\
        &~=~ \prob{ Z \geq (1+\beta)\mu } \\
        &~\leq~ \exp\big( - \beta \ln(\beta) \mu/4 \big)
        &&\qquad\text{(Chernoff bound)}
\\
       &~=~ (1/\beta)^{\beta \mu/4} 
       ~=~ (1/\beta)^{\alpha N/4}
               &&\qquad\text{(definition of $\beta$)} \\
       &~\leq~ (1/\beta)^{N/8}
           &&\qquad\text{(since $\beta \geq 1$ and $\alpha \geq 1/2$)}\\
       &~=~ \Big(\frac{1-\eps-\alpha}{\alpha}\Big)^{N/8} 
           &&\qquad\text{(by \eqref{eq:deltaLB})} \\
       &~\leq~ \Big(2(1-\eps-\alpha)\Big)^{N/8}
           &&\qquad\text{(since $\alpha \geq 1/2$).}~
    \end{alignat*}
\end{proof}

\begin{lemma}
\label{lem:partialintegral}
Suppose that $M$ is a random variable satisfying 
\[
\prob{ M < L(\alpha) } ~\leq~ g(\alpha) \quad\forall \alpha \in [0,\umax),
\]
where $\umax<\infty$ and $L,g : [0,\umax) \rightarrow \bR$ have the following properties.
We subdivide the interval $[0,\umax)$ using
values $0 = u_0 < u_1 < \cdots < u_k = \umax$, for some integer $k \geq 1$.
We require:
\begin{itemize}
\item $L(0)=\lim_{\alpha\downarrow0}L(\alpha)$,
\item $L$ is differentiable on $(u_{i-1},u_i)$
for all $i \in [k]$,
\item $L'(\alpha)<0$ for all $\alpha \in (u_{i-1},u_i)$,
for all $i \in [k]$
\item the left and right limits satisfy $L(u_i^-) \leq L(u_i^+)$ for all $i \in [k-1]$,
\item $\lim_{\alpha \to \umax^-}g(\alpha) = 0$.
\end{itemize}
Under the above conditions, we have
\[\expect{M} \geq 
L(0)-
\sum_{i=1}^k
\int_{u_{i-1}}^{u_i} g(\alpha)
\big(-L'(\alpha)\big) \,d\alpha.
\]
\end{lemma}

\begin{proof}
Define the intervals $\cI_i = \Union_{q \searrow 0} (L(u_i-q),L(u_{i-1}+q))$ and $\cI = \Union_i \cI_i$.
For any $M$ we have
\begin{align*}
\expect{M}
~\geq~ L(0)- \int_{-\infty}^{L(0)} \prob{ M < y } \,dy 
~\geq~ L(0)- \int_{\cI} \prob{ M < y } \,dy.
\end{align*}
The second inequality uses $L(0)=L(0^+)$,
the condition $L(u_i^-)\leq L(u_i^+)$,
and the fact that
\[
\lim_{\alpha\to\umax^-}\prob{M<L(\alpha)}
\leq
\lim_{\alpha\to\umax^-}g(\alpha)=0.
\]
Then, using $L(u_i^-) \leq L(u_i^+)$, we get
\[
\expect{M}
~\geq~ L(0)- \sum_{i=1}^k \int_{\cI_i} \prob{ M < y } \,dy.
\]
By a change of variables, we obtain
\[
\expect{M}
~\geq~ L(0)-
\sum_{i=1}^k
\int_{u_{i-1}}^{u_i} \prob{ M < L(\alpha) } \big(-L'(\alpha)\big) \,d\alpha.
\]
Since $-L'(\alpha) \geq 0$, we may replace $\prob{ M < L(\alpha) }$ with the upper bound $g(\alpha)$, which is the claimed inequality.
\end{proof}

Let \(L\) and \(g\) be as defined in \eqref{eq:defL} and \eqref{eq:defg}, respectively.
Since \(\Phi^{-1}(1/2)=0\), it follows from \eqref{eq:defL} that
\[
L((1/2 - (\epsilon+\delta))^-) \leq L((1/2 - (\epsilon+\delta))^+).
\]
Applying Lemma~\ref{lem:lower_bound_of_normal_quantile}, we obtain \[
\Phi^{-1}(\epsilon + \delta) \leq -\sqrt{\frac{\log\left(\frac{1}{\epsilon + \delta}\right)}{2}}.
\]
Since \(\delta\ge 0\), we have \(\epsilon+\delta \le \epsilon+2\delta\), and hence
\[
-\sqrt{\log\left(\frac{1}{\epsilon + \delta}\right)} \leq - \sqrt{\log\left(\frac{1}{\epsilon+2\delta}\right)},
\]
Then it follows again from \eqref{eq:defL} that
\[
L((1 - 2(\epsilon + \delta))^-) \leq L((1 - 2(\epsilon + \delta))^+).
\]

Additionally, by \eqref{eq:defL}, for all $\alpha\in(0,1-\epsilon)$ except at the breakpoints
$\alpha = 1/2-(\epsilon+\delta)$ and $\alpha = 1-2(\epsilon+\delta)$, where $\Phi$ and $\phi$ denote the cumulative distribution function and probability density function of a standard normal random variable, respectively, we have 
\begin{equation}
    -L'(\alpha) = \begin{cases}
        2Q \cdot \frac{1}{\phi(\Phi^{-1}(1-\epsilon-\delta-\alpha))}~&\text{if $\alpha \in (0, 1/2 - (\epsilon+\delta))$},\\
        4Q \cdot \frac{1}{\phi(\Phi^{-1}(1-\epsilon-\delta-\alpha))}~&\text{if $\alpha \in (1/2 - (\epsilon+\delta), 1 - 2(\epsilon + \delta))$},\\
        \sqrt{2}Q\frac{1}{\sqrt{\log\frac{1}{1-\epsilon-\alpha}}(1-\epsilon-\alpha)}~&\text{if $\alpha \in (1-2(\epsilon+\delta), 1-\epsilon)$}.
        \end{cases}
\end{equation}
In particular, \(L'(\alpha)<0\) for all \(\alpha\in(0,1-\epsilon)\) where the derivative exists.
Moreover, by the definition of function $g$ as stated in \eqref{eq:defg}, we have \[
\lim_{\alpha \to (1-\epsilon)^-}g(\alpha) = 0.
\]

Also, $L(0)=\lim_{\alpha\downarrow0}L(\alpha)$
by \eqref{eq:defL}.

Therefore, \(L\) and \(g\) defined in \eqref{eq:defL} and \eqref{eq:defg} satisfy all the conditions in Lemma~\ref{lem:partialintegral}.
Applying Lemma~\ref{lem:partialintegral} with \(M := X_{(N\epsilon)}\) yields 
\begin{equation}
\label{eq:lbEX}
\begin{aligned}
\E\!\left[X_{(N\epsilon)}\right]
\ge\;& L(0)
 - \int_{0}^{\frac{1}{2} - (\epsilon+\delta)} g(\alpha)\bigl(-L'(\alpha)\bigr)\, d\alpha \\
& - \int_{\frac{1}{2} - (\epsilon+\delta)}^{1-2(\epsilon + \delta)} g(\alpha)\bigl(-L'(\alpha)\bigr)\, d\alpha \\
& - \int_{1-2(\epsilon + \delta)}^{1-\epsilon} g(\alpha)\bigl(-L'(\alpha)\bigr)\, d\alpha .
\end{aligned}
\end{equation}

\subsection{Bound on $L(0)$}

By \eqref{eq:defL}, we have $L(0) = 2Q \Phi^{-1}(1-\epsilon-\delta)$.
Since $\delta + \epsilon \leq \epsilon_0$, we get 
\begin{align}\nonumber
L(0)
~=~ 2Q\cdot \Phi^{-1}(1-\epsilon-\delta)
\tag{by Lemma~\ref{lem:gaussiantail}}
&~\geq~ 2Q \cdot \Big( \sqrt{2\ln(\smallfrac{1}{\epsilon+\delta})} - 1.275 \Big)
\\
&~\geq~ 2Q \cdot \Big( \sqrt{2\ln(\smallfrac{1}{\epsilon})- 2 \ln 2} - 1.275 \Big)
\tag{since $\delta \leq \epsilon$} \\
&~\geq~ 2Q \cdot \Big( \sqrt{2\ln(\smallfrac{1}{\epsilon})} - 2 \Big),
\label{eq:L0LB}
\end{align}
since $\ln(1/\epsilon) > \ln(1/\epsilon_0) = \ln 40$.

\subsection{Bound on the integrals}

We now rescale the integrals in \eqref{eq:lbEX} appropriately, then separately prove that
\begin{align*}
\frac{1}{2Q}\int_{0}^{\frac{1}{2} - (\epsilon+\delta)} g(\alpha)(-L'(\alpha)) d\alpha &~\leq~ 1 \\
\frac{1}{4Q}\int_{\frac{1}{2} - (\epsilon+\delta)}^{1-2(\epsilon + \delta)} g(\alpha)(-L'(\alpha)) d\alpha &~\leq~ 1 \\
\frac{1}{\sqrt{2}Q}\int_{1-2(\epsilon + \delta)}^{1-\epsilon} g(\alpha)(-L'(\alpha)) d\alpha \Big\} &~\leq~ 1.
\end{align*}

Recall that $\Phi$ and $\phi$ denote the cumulative distribution function and probability density function of a standard normal random variable, respectively.
We will require the inequality
\begin{align}
    \label{eq:boundonphiPhi}
    &\frac{1}{\phi(\Phi^{-1}(x))} \leq \frac{1}{c_1 x \sqrt{\log (1/x)}}
    \qquad\forall x\in(0,\tfrac12],
    \\
\text{where}\quad\nonumber
&c_1:=\sqrt{\tfrac12}.
\end{align}
This follows from Lemma~\ref{lem:Mills}.

\paragraph{First integral.}
It suffices to upper bound
\[
\int_{0}^{\frac{1}{2} - (\epsilon+\delta)}
\frac{g(\alpha)}{\phi(\Phi^{-1}(1-(\epsilon+\delta)-\alpha))}\, d\alpha,
\qquad
g(\alpha)=\exp(-2N\alpha^{2}).
\]

Observe that for $\alpha\in\bigl[0,\tfrac12-(\epsilon+\delta)\bigr)$ we have
$(\epsilon+\delta)+\alpha \le \tfrac12$.

Let $\Phi$ and $\phi$ denote the standard normal distribution function and density, respectively.
Using the identity $\Phi^{-1}(1-u)=-\Phi^{-1}(u)$ for $u\in(0,1)$ and the evenness of $\phi$, we obtain
\begin{equation}
\label{eq:changetherange}
    \phi(\Phi^{-1}(1-(\epsilon+\delta)-\alpha)) = \phi(\Phi^{-1}((\epsilon+\delta) + \alpha)).
\end{equation}

Combining \eqref{eq:changetherange}, \eqref{eq:boundonphiPhi}, and the definition $g(\alpha)=e^{-2N\alpha^2}$, we obtain
\begin{align*}
\int_{0}^{\frac12-(\epsilon+\delta)}
\frac{g(\alpha)}{\phi(\Phi^{-1}(1-(\epsilon+\delta)-\alpha))}\,d\alpha
&=
\int_{0}^{\frac12-(\epsilon+\delta)}
\frac{g(\alpha)}{\phi(\Phi^{-1}((\epsilon+\delta)+\alpha))}\,d\alpha \\
&\le
\int_{0}^{\frac12-(\epsilon+\delta)}
\frac{\exp(-2N\alpha^2)}{c_1\,((\epsilon+\delta)+\alpha)\sqrt{\log\!\frac{1}{(\epsilon+\delta)+\alpha}}}\,d\alpha .
\end{align*}
Since $(\epsilon+\delta)+\alpha\ge \epsilon+\delta$ and $(\epsilon+\delta)+\alpha\le \frac12$ on the integration range, we have
$\log\!\bigl(1/((\epsilon+\delta)+\alpha)\bigr)\ge \log 2$, hence
\begin{align*}
\int_{0}^{\frac12-(\epsilon+\delta)}
\frac{\exp(-2N\alpha^2)}{c_1\,((\epsilon+\delta)+\alpha)\sqrt{\log\!\frac{1}{(\epsilon+\delta)+\alpha}}}\,d\alpha
&\le
\frac{1}{c_1(\epsilon+\delta)\sqrt{\log 2}}
\int_{0}^{\frac12-(\epsilon+\delta)} \exp(-2N\alpha^2)\,d\alpha \\
&=
\frac{1}{c_1(\epsilon+\delta)\sqrt{\log 2}}
\cdot
\frac{\sqrt{\pi}}{2\sqrt{2N}}
\,
\erf\!\left(\sqrt{2N}\Bigl(\tfrac12-(\epsilon+\delta)\Bigr)\right) \\
&\le
\frac{1}{c_1(\epsilon+\delta)}\cdot \frac{1}{\sqrt{N}},
\end{align*}
where we used $\erf(x)\le 1$.
Consequently, it suffices to assume $N \ge (c_1 (\epsilon+\delta))^{-2}$
to ensure that the above bound is at most $1$.

\comment{\vspace{2cm}

Then we have \begin{align*}
    \int_{0}^{\frac{1}{2} - (\epsilon+\delta)} \frac{g(\alpha)}{\phi(\Phi^{-1}(1-(\epsilon+\delta)-\alpha))} d\alpha &= \int_{0}^{\frac{1}{2} - (\epsilon+\delta)} \frac{g(\alpha)}{\phi(\Phi^{-1}((\epsilon+\delta) + \alpha))} d\alpha \tag{by \eqref{eq:changetherange}}\\
 &\leq \int_{0}^{\frac{1}{2} - (\epsilon+\delta)} \frac{\exp(-2\alpha^2 N)}{c_1 ((\epsilon+\delta) + \alpha) \sqrt{\log\frac{1}{(\epsilon+\delta) + \alpha}}} d\alpha \tag{by the definition of $g(\alpha)$ and \eqref{eq:boundonphiPhi}}\\
    &\leq \frac{1}{c_1 (\epsilon+\delta) \sqrt{\log 2}} \int_{0}^{\frac{1}{2} - (\epsilon+\delta)} \exp(-2\alpha^2 N) d\alpha \tag{by $(\epsilon+\delta) + \alpha \leq \frac{1}{2}$ and $(\epsilon+\delta) \leq (\epsilon+\delta) + \alpha$}\\
    &= \frac{1}{c_1 (\epsilon+\delta) \sqrt{\log 2}} \left(\frac{\sqrt{\pi}}{2} \frac{\erf\left(\sqrt{2N}\left(\frac{1}{2} - (\epsilon+\delta)\right)\right)}{\sqrt{2N}}\right)\\
    &\leq \frac{\sqrt{\pi}}{2\sqrt{2\log 2}c_1 (\epsilon+\delta)} \cdot \frac{1}{\sqrt{N}} \tag{by $\forall x \in \R, \erf(x) \leq 1$}\\
    &\leq \frac{1}{c_1 (\epsilon+\delta)} \cdot \frac{1}{\sqrt{N}}.
\end{align*}
Thus, if 
$ N \geq (c_1 (\epsilon+\delta))^{-2} $, then this bound is at most $1$.
}

\paragraph{Second integral.}
It suffices to upper bound
\[
\int_{\frac12-(\epsilon+\delta)}^{1-2(\epsilon+\delta)}
\frac{g(\alpha)}{\phi(\Phi^{-1}(1-(\epsilon+\delta)-\alpha))}\,d\alpha,
\qquad
g(\alpha)=\exp(-2N\alpha^{2}).
\]

Observe that for $\alpha \in \bigl[\tfrac12-(\epsilon+\delta),\, 1-2(\epsilon+\delta)\bigr]$ we have
$ 1-(\epsilon+\delta)-\alpha \leq \frac{1}{2}$.

Combining \eqref{eq:boundonphiPhi} with $g(\alpha)=\exp(-2N\alpha^2)$ yields
\[
\int_{\frac12-(\epsilon+\delta)}^{1-2(\epsilon+\delta)}
\frac{g(\alpha)}{\phi(\Phi^{-1}(1-(\epsilon+\delta)-\alpha))}d\alpha
\le\int_{\frac12-(\epsilon+\delta)}^{1-2(\epsilon+\delta)}
\frac{\exp(-2N\alpha^2)}{c_1\,(1-(\epsilon+\delta)-\alpha)\sqrt{\log\frac{1}{1-(\epsilon+\delta)-\alpha)}}} d\alpha.
\]
With the change of variables $x = 1-(\epsilon+\delta)-\alpha$ (so $dx=-d\alpha$), this becomes
\[
\int_{\frac12-(\epsilon+\delta)}^{1-2(\epsilon+\delta)}\frac{\exp(-2N\alpha^2)}{c_1\,(1-(\epsilon+\delta)-\alpha)\sqrt{\log\frac{1}{1-(\epsilon+\delta)-\alpha)}}}d\alpha=\int_{\epsilon+\delta}^{1/2}
\frac{\exp\!\bigl(-2N(1-(\epsilon+\delta)-x)^2\bigr)}{c_1\,x\sqrt{\log(1/x)}}\,dx.
\]
Since $\epsilon+\delta\le \tfrac14$, so that the function
$f(x):=1\big/\!\bigl(x\sqrt{\log(1/x)}\bigr)$ is decreasing on $x \in [(\epsilon+\delta), 1/2]$. Hence
\begin{align*}
\int_{\epsilon+\delta}^{1/2}
\frac{\exp\!\bigl(-2N(1-(\epsilon+\delta)-x)^2\bigr)}{c_1\,x\sqrt{\log(1/x)}}\,dx
&\le
\frac{1}{c_1(\epsilon+\delta)\sqrt{\log\!\bigl(1/(\epsilon+\delta)\bigr)}}
\int_{\epsilon+\delta}^{1/2} \exp\!\bigl(-2N(1-(\epsilon+\delta)-x)^2\bigr)\,dx\\
&= \frac{1}{c_1 (\epsilon+\delta) \sqrt{\log \frac{1}{(\epsilon+\delta)}}} \frac{\sqrt{\pi}}{2\sqrt{2N}}\\
    &\quad\cdot\left[-\erf\left[(2(\epsilon+\delta) - 1)\sqrt{2N}\right] + \erf\left(\frac{(2(\epsilon+\delta)-1)\sqrt{2N}}{2}\right)\right]\\
&\leq \frac{\sqrt{\pi}}{c_1 (\epsilon+\delta) \sqrt{2\log \frac{1}{(\epsilon+\delta)}}} \frac{1}{\sqrt{N}} \\
    &\leq \frac{\sqrt{\pi}}{c_1 (\epsilon+\delta) } \frac{1}{\sqrt{N}},
\end{align*}
where we used $|\erf(t)|\le 1$ and $\epsilon + \delta \leq 1/4$.

Consequently, it suffices to assume
$N \geq \pi (c_1(\epsilon+\delta))^{-2}$
to make the above bound at most $1$.

\paragraph{Third integral.}
It suffices to upper bound
\[
\int_{1-2(\epsilon+\delta)}^{1-\epsilon} \frac{g(\alpha)}{\sqrt{\log\frac{1}{1-\epsilon-\alpha}}(1-\epsilon-\alpha)} d\alpha,
\qquad
g(\alpha)=(2(1-\epsilon-\alpha))^{N/8}.
\]

Recall that $\epsilon+2\delta\le \tfrac12$. Set $\gamma := 1-\epsilon-\alpha$. Then $d\gamma=-d\alpha$, and the limits
$\alpha\in[1-2(\epsilon+\delta),\,1-\epsilon]$ correspond to $\gamma\in[0,\,\epsilon+2\delta]$. Hence
\begin{align*}
\int_{1-2(\epsilon+\delta)}^{1-\epsilon}
\frac{(2(1-\epsilon-\alpha))^{N/8}}{(1-\epsilon-\alpha)\sqrt{\log\!\bigl(1/(1-\epsilon-\alpha)\bigr)}}\,d\alpha
&=
\int_{0}^{\epsilon+2\delta}
\frac{(2\gamma)^{N/8}}{\gamma\sqrt{\log(1/\gamma)}}\,d\gamma.
\end{align*}
Since $\gamma\le \epsilon+2\delta\le \tfrac12$ on the integration range, we have
$\log(1/\gamma)\ge \log 2$, and therefore
\begin{align*}
\int_{0}^{\epsilon+2\delta}
\frac{(2\gamma)^{N/8}}{\gamma\sqrt{\log(1/\gamma)}}\,d\gamma
&\le
\frac{2^{N/8}}{\sqrt{\log 2}}
\int_{0}^{\epsilon+2\delta} \gamma^{N/8-1}\,d\gamma \\
&=
\frac{2^{N/8}}{\sqrt{\log 2}}\cdot \frac{8}{N}\,(\epsilon+2\delta)^{N/8}
=
\frac{8}{N\sqrt{\log 2}}\,[2(\epsilon+2\delta)]^{N/8}.
\end{align*}
Using $2(\epsilon+2\delta)\le 1$ and by selecting $N \geq 8$, we obtain the simpler bound
\[
\int_{1-2(\epsilon+\delta)}^{1-\epsilon}
\frac{(2(1-\epsilon-\alpha))^{N/8}}{(1-\epsilon-\alpha)\sqrt{\log\!\bigl(1/(1-\epsilon-\alpha)\bigr)}}\,d\alpha
\le
\frac{8}{N\sqrt{\log 2}}
\le
\frac{16}{N}.
\]
Consequently, it suffices to take $N\ge 16$ to make this bound at most $1$.

\comment{

\vspace{2cm}
Here we need to bound \[
\int_{1-2(\epsilon+\delta)}^{1-\epsilon} \frac{(2(1-\epsilon-\alpha))^{N/8}}{\sqrt{\log\frac{1}{1-\epsilon-\alpha}}(1-\epsilon-\alpha)} d\alpha.
\]

Here we use $\epsilon + 2\delta \leq \frac{1}{2}$.

Let $\gamma = 1- \epsilon - \alpha$, then $\gamma \leq \epsilon + 2\delta \leq \frac{1}{2}$.

Then we have \begin{align*}
    \int_{1-2(\epsilon+\delta)}^{1-\epsilon} \frac{(2(1-\epsilon-\alpha))^{N/8}}{\sqrt{\log\frac{1}{1-\epsilon-\alpha}}(1-\epsilon-\alpha)} d\alpha &= \int_{0}^{\epsilon + 2\delta} \frac{(2\gamma)^{N/8}}{\gamma\sqrt{\log \frac{1}{\gamma}}}d\gamma\\
    &\leq \frac{2^{N/8}}{\sqrt{\log 2}}\int_{0}^{\epsilon + 2\delta} \gamma^{N/8-1} d\gamma \tag{by $\gamma \leq \frac{1}{2}$}\\
    &= \frac{8}{N\sqrt{\log 2}}[2(\epsilon+2\delta)]^{N/8}\tag{by choosing $N \geq 8$}\\
    &\leq \frac{16}{N} \tag{by $\epsilon+\delta \leq \frac{1}{2}$}.
\end{align*}

Thus, if 
$ N \geq 16$, then this bound is at most $1$.
}

\subsection{Combination}

Let $N(\epsilon)\in\mathbb{N}$ satisfy
\begin{equation}
\label{eq:chooseNeps}
    N(\epsilon) \geq \max\{(c_1 (\epsilon+\delta))^{-2}, \pi (c_1(\epsilon+\delta))^{-2}, 16\}.
\end{equation}
Then with $T=T(\epsilon)$ as in \eqref{eq:chooseTeps} and $N=N(\epsilon)$ as in \eqref{eq:chooseNeps}, the integrals in \eqref{eq:lbEX} contribute
\begin{align*}
&-\int_{0}^{\frac{1}{2} - (\epsilon+\delta)} g(\alpha)(-L'(\alpha)) d\alpha
-\int_{\frac{1}{2} - (\epsilon+\delta)}^{1-2(\epsilon + \delta)} g(\alpha)(-L'(\alpha)) d\alpha -
\int_{1-2(\epsilon + \delta)}^{1-\epsilon} g(\alpha)(-L'(\alpha)) d\alpha \\
~\geq~& - 2Q - 4Q - \sqrt{2}Q.
\end{align*}
Combining this with our lower bound on $L(0)$ from \eqref{eq:L0LB}, we get
\begin{equation}
\label{eq:EXlastbound}
\E\!\left[X_{(N\epsilon)}\right]
 ~\geq~ 2Q \Big( \sqrt{2\ln(1/\epsilon) } - 2 - 1 - 2 - \frac{1}{\sqrt{2}} \Big).
\end{equation}
This completes the proof of Theorem~\ref{thm:lower-bound}.

\subsection{Bounds on Gaussians}

The following tail bound is well known.

\begin{lemma}[{\protect \cite[Theorem 1.2.6]{durrett2019probability}}]
\label{lem:durretttail}
\[
\frac{1}{\sqrt{2 \pi}} (1-x^{-2}) x^{-1}  \exp(-x^2/2) 
    ~\leq~ 1-\Phi(x)
    ~\leq~ \frac{1}{\sqrt{2 \pi}} x^{-1}  \exp(-x^2/2) 
    \qquad\forall x>0.
\]
\end{lemma}

The following fact appears in the literature.
See, e.g., \cite{Boucheron}, in the proof of Proposition 4.1.

\begin{lemma}
\label{lem:Mills}
Let $\Phi$ and $\phi$ denote the standard normal distribution function and density.
Then, for all $x\in(0,1/2]$,
\[
x\sqrt{(1/2)\log(1/x)}
\;\le\;
\phi\left(\Phi^{-1}(x)\right).
\]
\end{lemma}

We will prove the following bound on the inverse of the Gaussian cumulative distribution function.

\begin{lemma}
\label{lem:gaussiantail}
\label{lem:lower_bound_of_normal_quantile}
Let $\Phi$ be the cumulative distribution function of the standard Gaussian distribution.
For $\zeta \in (0,1/40]$, we have
\begin{alignat}{2}\label{eq:phiupperbound}
\Phi^{-1}(1-\zeta)
&~\leq~
\sqrt{2\ln(1/\zeta)-\ln\ln(1/\zeta)} \\\label{eq:lowordererror}
\Phi^{-1}(1-\zeta)
 &~\geq~ 
  \sqrt{2\ln(1/\zeta)-4 \ln\ln(1/\zeta)}
 &&\qquad\text{\rm(low-order error)}\\
 \label{eq:additiveerror}
 &~\geq~ \sqrt{2\ln(1/\zeta)} - 1.275
 &&\qquad\text{\rm(additive error)}\\
 \label{eq:multiplicativeerror}
 &~\geq~ \sqrt{(1/2)\ln(1/\zeta)}
 &&\qquad\text{\rm(multiplicative error)}.
\end{alignat}
\end{lemma}

Define the parameters
\[
\lambda_1 = \ln(40),
\qquad
\lambda_2 = 1.275,
\qquad
\lambda_3 = 4,
\qquad
\lambda_4 = 1.44,
\quad\text{and}\quad
\lambda_5 = \frac{(1-\lambda_4^{-2})}{\sqrt{2 \pi}}.
\]
Using this notation, we will assume throughout that
$\zeta \in (0, e^{-\lambda_1}]$.
The proof will use the function 
\[
    h(\zeta) = \sqrt{2 \ln(1/\zeta) - \lambda_3 \ln \ln(1/\zeta)},
\]
which is the right-hand side of \eqref{eq:lowordererror}.

\begin{proposition}
\label{prop:lambert}
Let $a> 0$ and $b \in (0,a \cdot e)$.
Consider the function
\[
g(z) ~=~ \sqrt{a z} - \sqrt{a z - b \ln z},
\]
for $z \geq 1$.
Then $g$ is uniquely maximized at
\[ \hat{z} = \exp\big(W_0(-b/ae^2)+2\big), \]
where $W_0$ is the principal branch of the Lambert W function.
\end{proposition}

\begin{proof}
Let us first remark that $z \geq e \ln z $ for all $z > 0$.
This holds because the right-hand side is concave and the left-hand side is its tangent line at the
point $z=e$.
Using that $b < a \cdot e$ and $z \geq 1$,
it follows that $a z - b \ln z \geq a ( z - e \ln z ) \geq 0$, so $g(z)$ is well-defined 
for $z \geq 1$.

We will characterize the critical points of $g$ on the interval $[1,\infty)$.
The derivative of $g$ is 
\[
g'(z) ~=~ \frac{a}{2 \sqrt{az}} - \frac{a-b/z}{2\sqrt{az-b\ln z}}.
\]
Multiplying by $2 \sqrt{az} \sqrt{az-b\ln z}$, which is positive, we see that $g'(z)=0$ is equivalent to
\begin{equation}
\label{eq:LambertGderiv}
a\sqrt{az-b\ln z} ~=~ \sqrt{az}(a-b/z).
\end{equation}
This can only be solvable if both sides have the same sign, which is equivalent to 
$a-b/z > 0$, i.e., $z > b/a$.
Thus, we will require the condition that $z > b/a$.
Squaring both sides of \eqref{eq:LambertGderiv} and dividing by $a$ we obtain the equivalent equation
\[
a (az-b\ln z) ~=~ a^2 z - 2ba + b^2/z.
\]
Canceling the $a^2 z$ term and multiplying by $-\frac{1}{ab}$, we 
obtain the equivalent condition
\begin{equation}
\label{eq:lambert}
\ln z ~=~ 2 - \frac{b}{az}.
\end{equation}
We claim that $\hat{z} = \exp\big(W_0(-b/ae^2)+2\big)$ is the unique solution
to \eqref{eq:lambert} that satisfies our required conditions.

To see this, let us manipulate \eqref{eq:lambert} as follows.
Since $z > 0$ we may make the change of variables $z = e^{W+2}$;
then, after simple manipulations, we see that \eqref{eq:lambert} is equivalent to
\[
W \exp(W) ~=~ -\frac{b}{a e^2}.
\]
Since $-b/ae^2 \in (-1/e,0)$,
Lemma~\ref{lem:Lambert} shows that this has exactly two solutions,
namely $W_0(-b/a e^2)$ and $W_{-1}(-b/a e^2)$.
This implies that \eqref{eq:lambert} has exactly two solutions, namely
\[
\hat{z} = \exp\big(W_0(-b/a e^2)+2\big)
    \quad\text{and}\quad
    \tilde{z} = \exp\big(W_{-1}(-b/a e^2)+2\big).
\]
By Lemma~\ref{lem:Lambert}, $W_0(-b/a e^2) > -1$ so $\hat{z} > e \geq b/a$.
Thus $\hat{z}$ is indeed a critical point of $g$ on $[1,\infty)$ since it satisfies 
the condition $z > b/a$ that we imposed earlier.

Next we turn to the other possible solution $\tilde{z}$.
From Lemma~\ref{lem:Lambert} we use \eqref{eq:LambertObs} to obtain that
$\exp(W_{-1}(-b/ae^2)) \leq b/ae^2$,
and therefore 
$\tilde{z} = \exp(W_{-1}(-b/ae^2)+2) \leq b/a$.
Therefore $\tilde{z}$ does not satisfy our required condition that $\tilde{z} > b/a$.

We conclude that $\hat{z}$ is the unique critical point of $g$ on the interval $[1,\infty)$.
Furthermore, $\hat{z}$ is the unique maximizer because $g(1)=0$,
$g(z)>0$ for $z > 1$, and $\lim_{z \rightarrow \infty} g(z)=0$.
\comment{
The Lambert $W$ function can take multiple values (the principal and secondary branches)
when its argument is negative. 
Above we have have used the primary branch.
For the secondary branch, all of its values are at most 
whose values are all at most $-1$, which would imply that $z \leq e$, 
and therefore $a-b/z \leq a-b/e < 0$, since we require $b \in (0,ae)$.
This violates our requirement that $a-b/z > 0$, and so the solution $\hat{z}$ is unique.
}
\end{proof}

\begin{proposition}
\label{prop:hz_additive}
For $\zeta \in (0,e^{-\lambda_1}]$,
\[ \sqrt{2\ln(1/\zeta)} - \lambda_2
    ~\leq~ h(\zeta)
    ~\leq~ \sqrt{2\ln(1/\zeta)}.
\]
\end{proposition}
\begin{proof}
The right-hand inequality is immediate because $\ln\ln(1/\zeta) \geq 0$ since
$\zeta \leq 1/e$.
Next we turn to the left-hand inequality.

Using Proposition~\ref{prop:lambert} with $a=2$, $b=4 < a\cdot e$,
we have
\[
\sqrt{2 z} - \sqrt{2 z - 4 \ln z} \leq
\sqrt{2 \hat{z}} - \sqrt{2 \hat{z} - 4 \ln \hat{z}}
\]
where $\hat{z} = \exp\big(W(-2/e^2)+2\big)$.
Numerically we have 
\[ \sqrt{2 \hat{z}} - \sqrt{2 \hat{z} - 4 \ln \hat{z}} < 1.275 = \lambda_2. \]
Letting $z = \ln(1/\zeta)$ proves the left-hand inequality.
\end{proof}

\comment{
\begin{proof}
We apply the identity
$ \sqrt{a}-\sqrt{a-b} = b/\big(\sqrt{a}+\sqrt{a-b})$,
which is valid for all $0 \leq b < a$.
This yields
\[
\sqrt{2 \ln(1/\zeta)} - h(\zeta) 
~\leq~ \frac{2 \ln \ln(1/\zeta)}{\sqrt{2 \ln(1/\zeta)}}
~=~ \sqrt{2} \frac{\ln(z)}{\sqrt{z}},
\]
where $z = \ln(1/\zeta)$.
Since $\sqrt{z} \geq \ln(z)$ for all $z \geq 0$, this completes the proof.
\end{proof}
}

\begin{corollary}
\label{cor:simplehbound}
For $\zeta \in (0,e^{-\lambda_1}]$, we have $h(\zeta) \geq \lambda_4$.
\end{corollary}
\begin{proof}
We use Proposition~\ref{prop:hz_additive} and a numerical calculation to obtain
\[
    h(\zeta)
        ~\geq~ \sqrt{2 \ln(1/\zeta)} - \lambda_2 
        ~\geq~ \sqrt{2 \lambda_1} - \lambda_2 
        ~>~ 1.44
        ~=~ \lambda_4. 
\]
\end{proof}

\comment{
\begin{proof}
We have
\[ 
\frac{1}{\zeta} \geq 400 \geq \exp\big((2+\sqrt{2})^2/2\big).
\]
This yields
$2\ln(1/\zeta) \geq (2+\sqrt{2})^2$.
Take the square root and subtract $\sqrt{2}$ to obtain
\[
\sqrt{2 \ln(1/\zeta)} - \sqrt{2} \geq 2.
\]
Now the result follows from Proposition~\ref{prop:hz_additive}.
\end{proof}
}

\begin{proposition}
\label{prop:mainzeta}
For $\zeta \in (0, e^{-\lambda_1}]$, we have
\[
\zeta
~\leq~ \frac{\lambda_5}{h(\zeta)} \exp\Big( - \frac{h(\zeta)^2}{2} \Big).
\]
\end{proposition}
\begin{proof}
We may verify numerically that
\[
\frac{2}{\lambda_5^2} ~<~ 50 ~<~ \big(\ln(1/e^{-\lambda_1})\big)^{\lambda_3-1}
 ~\leq~ \big(\ln(1/\zeta)\big)^{\lambda_3-1}.
\]
Multiplying by $\lambda_5^2 \ln(1/\zeta)$ then taking the square root,
and using Proposition~\ref{prop:hz_additive},
we obtain
\[
h(\zeta) ~\leq~ \sqrt{2 \ln(1/\zeta)} ~\leq~ \lambda_5 \big(\ln(1/\zeta)\big)^{\lambda_3/2}.
\]
Multiplying both sides by $\frac{\zeta}{h(\zeta)}$,
which is $\frac{1}{h(\zeta)} \exp(-\ln(1/\zeta))$, 
we obtain
\[
\zeta
~\leq~ \frac{\lambda_5}{h(\zeta)} \exp\Big(-\ln(1/\zeta) + \frac{\lambda_3}{2}\ln\ln(1/\zeta)\Big)
~=~ \frac{\lambda_5}{h(\zeta)} \exp\big( - h(\zeta)^2/2 \big).
\]
This completes the proof.
\end{proof}

\begin{proof}(of Lemma~\ref{lem:gaussiantail})
First let us consider \eqref{eq:phiupperbound}.
Let $x = \sqrt{2 \ln(1/\zeta)- \ln\ln(1/\zeta)}$.
Since $z \geq \ln z$ for all $z > 0$,
we have $\ln(1/\zeta) \geq \ln\ln(1/\zeta)$,
which yields
\[ x ~\geq~ \sqrt{\ln(1/\zeta)} ~\geq~ 1. \]
Thus, from Lemma~\ref{lem:durretttail}, we have
\begin{align*}
    1-\Phi(x)
    &~\leq~ \exp(-x^2/2 - \ln x) \\
    &~=~ \exp\Big(-\smallfrac{1}{2}
        \big(2 \ln(1/\zeta)-\ln\ln(1/\zeta)\big) - \ln x\Big) \\
    &~\leq~ \exp\Big(- \ln(1/\zeta)+ \smallfrac{1}{2} \ln\ln(1/\zeta)
        - \ln\big(\sqrt{\ln(1/\zeta)}\big)\Big) \\
    &~=~ \zeta.
\end{align*}
That is, $\Phi(x) \geq 1-\zeta$.
Since $\Phi^{-1}$ is monotonically increasing, this implies that
$x \geq \Phi^{-1}(1-\zeta)$,
which is \eqref{eq:phiupperbound}.

Next we turn to \eqref{eq:lowordererror}.
Recalling the definition of $\lambda_5$, we obtain
\begin{alignat*}{2}
\zeta
&~\leq~ \frac{(1-\lambda_4^{-2})}{\sqrt{2 \pi}} \cdot \frac{1}{h(\zeta)} \cdot \exp\Big( -
\frac{h(\zeta)^2}{2} \Big)
&&\qquad\text{(by Proposition~\ref{prop:mainzeta})} \\
&~\leq~ \frac{1}{\sqrt{2 \pi}} \cdot \frac{(1-h(\zeta)^{-2})}{h(\zeta)} \cdot \exp\Big( -
\frac{h(\zeta)^2}{2} \Big)
&&\qquad\text{(by Corollary~\ref{cor:simplehbound})} \\
&~\leq~ 1 - \Phi(h(\zeta))
&&\qquad\text{(by Lemma~\ref{lem:durretttail})}
\end{alignat*}
Rearranging, we have $1-\zeta \geq \Phi(h(\zeta))$.
Since $\Phi^{-1}$ is monotonically increasing, this implies that
$\Phi^{-1}(1-\zeta) \geq h(\zeta)$,
which establishes \eqref{eq:lowordererror}.
In turn, Proposition~\ref{prop:hz_additive} establishes \eqref{eq:additiveerror}.

To prove \eqref{eq:multiplicativeerror},
we may numerically verify that
\[
\ln(1/\zeta)
~\geq~ \lambda_1
~=~ \ln(40) 
~>~ 3.5
~\geq~ \Big(\frac{\lambda_2}{\sqrt{2}-\sqrt{1/2}}\Big)^2.
\]
Taking the square root and rearranging establishes \eqref{eq:multiplicativeerror}.
\end{proof}

In this proof we have used certain properties of the Lambert W function. 
They are captured in the following lemma.

\begin{lemma}
\label{lem:Lambert}
On the interval $[-1/e,0)$, there are two branches of the Lambert W function, called $W_0$ and
$W_{-1}$.
The former is called the primary branch, and the latter is called the secondary branch.
\begin{enumerate}
\item The functions $W_0$ and $W_{-1}$ are both $C^\infty$ on $(-1/e,0)$.
\item $W_0(-1/e)=W_{-1}(-1/e)=-1$.
\item $W_0$ is strictly increasing, whereas $W_{-1}$ is strictly decreasing.
\item For any $z \in [-1/e,0)$, the equation $W e^W = z$ has exactly two solutions,
namely $W_0(z)$ and $W_{-1}(z)$.
These solutions are distinct, except at $z=-1/e$.
\item $\lim_{z \rightarrow 0} W_{-1}(z)=-\infty$.
\item For $z \in (-1/e,0)$, we have
$\frac{d^2}{dz^2} \exp(W_{-1}(z)) = -\frac{W_{-1}(z)}{(1+W_{-1}(z))^3 z}$.
Furthermore, this is strictly positive since $W_{-1}(z) < -1$ on this interval.
\item
\begin{equation}
\label{eq:LambertObs}
\exp(W_{-1}(z)) \leq -z \quad\forall z \in [-1/e,0).
\end{equation}
\end{enumerate}
\end{lemma}

\begin{proof}
We provide only a proof of \eqref{eq:LambertObs}.
Note that
$\exp(W_{-1}(z)) = -z = 1/e$ when $z=-1/e$ since $W_{-1}(z)=-1$ by statement 2.
Furthermore, 
$\lim_{z \rightarrow 0} \exp(W_{-1}(z)) = \lim_{z \rightarrow 0} -z = 0$
since $\lim_{z \rightarrow 0} W_{-1}(z) = -\infty$ by statement 5.
Thus the left and right hand sides of \eqref{eq:LambertObs} agree at $z=-1/e$ and as $z \rightarrow
0$.
The left-hand side is convex by statement 6, whereas the right-hand side is linear,
so this establishes \eqref{eq:LambertObs}.
\end{proof}

\section{Proof of regret bounds with sparse variance}
\label{apx:sparse-regret}

\begin{proof}[Proof of Corollary~\ref{cor:sparsity}]
Our goal is to use $\breve{\Delta\bx_j}\coloneqq\Delta\bx_j^{\mathrm{ar}}$ in
Lemma~\ref{lem:reduction_delta_t}.   Writing $\widetilde x_{j,i}=[x_{j,i}]_+$, we need to verify that
\[
\Pi_{\R_+^N}
(\widetilde{\bx}_{j-1}+\Delta\bx_j^{\mathrm{ar}})
=\widetilde{\bx}_j.
\]
Indeed, on an active coordinate $i$ (that is, \(i\) such that $x_{j-1,i} > 0$) the projection yields
$[x_{j-1,i}+\Delta x_{j,i}]_+$. Otherwise, if $x_{j-1,i} \leq 0$ then \(0 \leq \tilde{x}_{j-1,i}+\Delta x_{j,i}^{\mathrm{ar}} =\widetilde x_{j,i}\) by the definition of $\Delta x_{j,i}^{\mathrm{ar}}$.

For the proof to go through, we also need to verify that $\langle\nabla_{\bx}\PhiNH(\widetilde{\bx}_{j-1},t_{j-1}),
\Delta\bx_j^{\mathrm{ar}}\rangle = 0$. Indeed, define
$Z_j\coloneqq\sum_i\partial_x\phiNH(\widetilde x_{j-1,i},t_{j-1})$. If \(Z_j > 0\), since 
$\partial_x\phiNH(0,t)=0$ we have
\[
\bigl\langle\nabla_{\bx}\PhiNH(\widetilde{\bx}_{j-1},t_{j-1}),
\Delta\bx_j^{\mathrm{ar}}\bigr\rangle
=Z_j\sum_{i \colon x_{j-1,i} > 0} p_{j,i}\Delta x_{j,i}^\mathrm{ar} = Z_j \langle \bp, \Delta \bx \rangle=0.
\]
When $Z_j=0$, the gradient itself is zero.

Finally, to see that $|\Delta x^{\mathrm{ar}}_{j,i}|\le|\Delta x_{j,i}|$ for all $j$, it suffices to note that for if $x_{j-1,i} \leq 0$ we have
\[
0\le[x_{j-1,i}+\Delta x_{j,i}]_+
\le[\Delta x_{j,i}]_+\le B,
\]
while if $x_{j-1,i} > 0$ we have $|\Delta x^{\mathrm{ar}}_{j,i}| = |\Delta x_{j,i}|$. 
Therefore, $V_T^{q,\mathrm{ar}}\le V_T$ and
$\|\Delta\bx_j^{\mathrm{ar}}\|_\infty\le B$.
The same local-GSC argument as in Section~\ref{sec:proof_normalhedge},
with $\MixErr_\Phi(\widetilde{\bx}_{j-1},t_{j-1};\Delta\bx_j^{\mathrm{ar}})$
and $\TimeErr_\Phi(\widetilde{\bx}_{j-1},t_{j-1};\breve{\Delta t_j})$,
therefore bounds the time increment $\breve{\Delta t_j}$, and
Lemma~\ref{lem:reduction_delta_t} bounds the algorithms time increment $\Delta t_j$ by
\[
\Delta t_j\le\breve{\Delta t_j}
\le2\E_{i\sim\bq_j}\bigl[(\Delta x^{\mathrm{ar}}_{j,i})^2\bigr].
\]
Summing yields $t_T\le t_0+2V_T^{q,\mathrm{ar}}$; the regret bound
then follows as in Theorem~\ref{thm:main}.
\end{proof}

\input{discussion_MZ_lowerbound}

\begingroup
\section{A discrete second-derivative bound without explicit dependence on $N$}
\label{sec:discrete_sec_deriv}

For the NormalHedge potential $\phi \coloneqq \phiNH$ in \eqref{eqn:normalhedge-pot},
recall our notation
$\Phi(\bx,t)=\sum_i\phi(x_i,t)$ and
$\widetilde{\bx}_j=[\bx_j]_+$.
Define the discrete second-order quantity
\begin{equation}
\label{eq:discrete_second_order_measure}
W_T\coloneqq\sum_{j=1}^T
\frac{2\bigl[
\Phi(\widetilde{\bx}_{j-1}+\Delta\bx_j,t_{j-1})
-\Phi(\widetilde{\bx}_{j-1},t_{j-1})
\bigr]}
{\sum_i\partial_{xx}\phi(\widetilde x_{j,i},t_j)}.
\end{equation}
The numerator is evaluated before projection at the previous local time value, whereas
the denominator uses the projected state at the new/current local time.
This quantity is computable along the original algorithm's run.

\begin{proposition}[Discrete bound without a warm-up condition]
\label{prop:discrete_second_order}
For every $t_0>0$, NormalHedge.BH satisfies $t_T\le t_0+W_T$ and,
simultaneously for quantiles with $N\epsilon\in\{1,\ldots,N\}$,
\begin{equation}
\label{eq:discrete_second_order_regret}
\Reg_\epsilon(T)\le
\sqrt{(t_0+W_T)
\left(\log\left(1+\frac{W_T}{t_0}\right)+2\log\frac1\epsilon\right)}.
\end{equation}
In particular, $t_0=1$ is allowed, with no condition linking $t_0$ to
$B$ or $N$.
For other $\epsilon\in[1/N,1]$, replace $\epsilon$ on the right by
$\lfloor N\epsilon\rfloor/N$, following the paper's quantile convention.
\end{proposition}

\begin{proof}
Coordinatewise,
$\widetilde x_{j,i}\le
[\widetilde x_{j-1,i}+\Delta x_{j,i}]_+$.
Projection does not increase $\phi$, and $\phi$ is nondecreasing on
$\R_+$. Consequently,
\begin{align*}
&\Phi(\widetilde{\bx}_{j-1}+\Delta\bx_j,t_{j-1})
 -\Phi(\widetilde{\bx}_{j-1},t_{j-1})\\
&\quad\ge
 \Phi(\widetilde{\bx}_j,t_{j-1})-\Phi(\widetilde{\bx}_j,t_j)
 &&\text{(projection and \CP)}\\
&\quad\ge
 -\Delta t_j\,\partial_t\Phi(\widetilde{\bx}_j,t_j)
 &&\text{(convexity in $t$)}\\
&\quad=
 \frac{\Delta t_j}{2}
 \sum_i\partial_{xx}\phi(\widetilde x_{j,i},t_j)
 &&\text{(BHE).}
\end{align*}
Divide by the positive curvature sum and sum over $j$ to obtain
$t_T-t_0\le W_T$.
Lemma~\ref{ex:normalhedge} and monotonicity of
$t\mapsto t(\log(t/t_0)+2\log(1/\epsilon))$ for $t\ge t_0$
then give \eqref{eq:discrete_second_order_regret}.
\end{proof}

\begin{remark}[Connection to $V_T$]
\label{rem:discrete_second_order_limit}
The choice of $\bp_j$ cancels the linear term:
$\langle\nabla_x\Phi(\widetilde{\bx}_{j-1},t_{j-1}),
\Delta\bx_j\rangle=0$.
Taylor's integral remainder therefore rewrites
\eqref{eq:discrete_second_order_measure} as
\begin{equation}
\label{eq:discrete_second_order_curvature}
W_T=\sum_{j=1}^T\sum_i
\frac{2\int_0^1(1-s)\,
 \partial_{xx}\phi(\widetilde x_{j-1,i}+s\Delta x_{j,i},t_{j-1})\,ds}
{\sum_k\partial_{xx}\phi(\widetilde x_{j,k},t_j)}
(\Delta x_{j,i})^2.
\end{equation}
The spatial curvature average is a discrete second derivative; since
the smooth potential is used before projection, this identity also
holds at zero crossings.

For fixed $N$, consider a small-step limit $B\to0$ in which the projected
states and clocks remain in a compact set with $t\ge t_0>0$.
The spatial remainder is $O(B^2)$, and the time derivative is locally
bounded away from zero, so $\Delta t_j=O(B^2)$.
Also, $\|\widetilde{\bx}_j-\widetilde{\bx}_{j-1}\|_\infty\le B$.
Thus the coefficient of $(\Delta x_{j,i})^2$ in
\eqref{eq:discrete_second_order_curvature} is
$(1+o(1))q_{j,i}$ uniformly on this compact set, giving
\[
W_T=(1+o(1))\sum_{j=1}^T\sum_iq_{j,i}(\Delta x_{j,i})^2
=(1+o(1))V_T.
\]
Hence \eqref{eq:discrete_second_order_regret} recovers the continuous-time
bound with leading constant one. This convergence is not uniform over
growing $N$ or unbounded states. At finite step sizes the coefficients
are nonnegative but need not sum to one, so the argument does not give
$W_T\le TB^2$ or a universal constant-factor comparison to $V_T$.
\end{remark}

Unlike Theorem~\ref{thm:main}, the clock bound $t_T\le t_0+W_T$
requires no burn-in condition on $t_0$. The same argument extends to
other good potentials, requiring existence but not continuity of
$\partial_{xx}\phi$, and no third derivatives, provided the \CP{} update
is well-defined and the curvature denominator is positive. The conversion
from this clock bound to regret depends on the chosen potential.

\endgroup

\section{A small-loss application in online convex optimization}
\label{sec:oco_small_loss}

The purpose of this section is to illustrate one concrete application
of the second-order bound in Theorem~\ref{thm:main} outside the
experts' setting, recovering a familiar type of small-loss bound in OCO.  More specifically, we show that the
curvature-weighted variance in Theorem~\ref{thm:main} gives a comparator small-loss guarantee for
online convex optimization over the convex hull of finitely many
actions.

Consider a OCO instance where the learner has to pick points from $\mathcal K\coloneqq\operatorname{conv}\{u_1,\ldots,u_N\}\subseteq\R^d$ and the adversary picks non-negative convex $L$-smooth functions. We use the following standard self-bounding property of nonnegative
smooth functions; see, e.g.,
\citet[Lemma~3.1]{srebro2010smoothness}.

\begin{lemma}[Self-bounding property]
\label{lem:oco_self_bounding}
If $f:\R^d\to\R_+$ is $L$-smooth, then, for every $w\in\R^d$,
\begin{equation}
\label{eq:oco_self_bounding}
    \|\nabla f(w)\|_2^2\le 2L f(w).
\end{equation}
\end{lemma}

\begin{proof}
For $L>0$, the descent lemma and global nonnegativity give
\[
    0\le f\left(w-\frac{1}{L}\nabla f(w)\right)
    \le f(w)-\frac{1}{2L}\|\nabla f(w)\|_2^2.
\]
If $L=0$, the gradient is constant and nonnegativity further implies that the gradient is zero.
\end{proof}

For $\epsilon\in[1/N,1)$, let $u_\epsilon^\star$ be the
$\lfloor\epsilon N\rfloor$-th smallest-loss action from \(\{u_1, \dotsc, u_N\}\) in hindsight, and
write
\[
    S_\epsilon\coloneqq\sum_{j=1}^T f_j(u_\epsilon^\star),
    \qquad
    \RegOCO_\epsilon(T)
    \coloneqq\left[\sum_{j=1}^T
      \bigl(f_j(w_j)-f_j(u_\epsilon^\star)\bigr)\right]_+.
\]
We take the positive part only on the definition of OCO regret to state a slightly stronger guarantee:
since the untruncated regret is at most \(\RegOCO_\epsilon(T)\), any
upper bound on the latter also bounds the usual regret, and the proof
below applies without change.

\paragraph{An OCO learner using NormalHedge.BH.} We can define a learner for this game using a standard reduction to the expert's problem. Namely, at round $j$, NormalHedge.BH outputs $\bp_j$ and the OCO learner plays
$w_j\coloneqq\sum_i p_{j,i}u_i \in \cK$. After the adversary reveals the function $f_j:\R^d\to\R_+$ chosen on round $j$,  we feed the
linearized losses
\[
    \ell_{j,i}\coloneqq\langle g_j,u_i\rangle,
    \qquad g_j\coloneqq\nabla f_j(w_j),
\]
to NormalHedge.BH at round $j$.  For our results, let us further assume that the functions picked by the adversary are bounded on $\cK$, that is,  $\sup_{w \in \cK} f_j(w)\le M$ for every $j$, and define
let
\[
    D\coloneqq \max_{i,i'\in[N]}\|u_i-u_{i'}\|_2.
\]

\begin{corollary}[Quantile small-loss bound for finite-action OCO]
\label{cor:oco_small_loss}
Run NormalHedge.BH with
\[
    t_0\coloneqq\max\{1024e^2LD^2M\log N,1\}.
\]
Let $V_T\coloneqq\sum_{j=1}^T\E_{i\sim\bq_j}[(\Delta x_{j,i})^2]$ be the
curvature-weighted variance returned by the algorithm, and define
\[
    \iota_\epsilon
    \coloneqq\log(t_0+2V_T)+2\log(1/\epsilon).
\]
Then, simultaneously for all $\epsilon\in[1/N,1)$,
\begin{equation}
\label{eq:oco_small_loss_implicit}
    \RegOCO_\epsilon(T)
    \le
    2D\sqrt{L\iota_\epsilon S_\epsilon}
    +4LD^2\iota_\epsilon
    +\sqrt{t_0\iota_\epsilon}.
\end{equation}
Moreover,  the $V_T$ term inside the logarithm can be removed.  In
particular, with $a\coloneqq LD^2$, $C_\epsilon\coloneqq 2\log(1/\epsilon)$, and
$A_\epsilon\coloneqq t_0+4aS_\epsilon$,
\begin{equation}
\label{eq:oco_small_loss_explicit}
    \RegOCO_\epsilon(T)
    \le
    \sqrt{2A_\epsilon\bigl(\log(2A_\epsilon)+C_\epsilon\bigr)}
    +16a\bigl(1+C_\epsilon+2\log(1+8a)\bigr).
\end{equation}
Consequently, the bound depends on the comparator loss rather than
directly on the horizon; when $S_\epsilon=0$, it is independent of
$T$.
\end{corollary}

\begin{proof}
By Lemma~\ref{lem:oco_self_bounding} and since the functions are bounded on \(\cK\), the range of the linearized loss
vector satisfies
\[
    \max_{i,i'}|\ell_{j,i}-\ell_{j,i'}|
    \le D\|g_j\|_2
    \le D\sqrt{2LM}.
\]
Thus Theorem~\ref{thm:main} applies with
$B\coloneqq D\sqrt{2LM}$ and the stated choice of $t_0$.  By the definitions of
$w_j$ and the linearized losses,
\begin{equation}
\label{eq:oco_regret_increment}
    \Delta x_{j,i}
    =\sum_k p_{j,k}\ell_{j,k}-\ell_{j,i}
    =\langle g_j,w_j-u_i\rangle.
\end{equation}
Moreover, $\|w_j-u_i\|_2\le D$ for every $i$, and hence
\begin{align}
    V_T
    &=\sum_{j=1}^T\E_{i\sim\bq_j}
      [\langle g_j,w_j-u_i\rangle^2] \le D^2\sum_{j=1}^T\|g_j\|_2^2
    \le 2LD^2\sum_{j=1}^T f_j(w_j).            \label{eq:oco_variance_bound}
\end{align}

By convexity,
\[
    f_j(w_j)-f_j(u_i)
    \le \langle g_j,w_j-u_i\rangle
    =\Delta x_{j,i}.
\]
Put $m\coloneqq\lfloor\epsilon N\rfloor$.  The preceding inequality holds for
each of the $m$ actions whose cumulative loss is no larger than
$S_\epsilon$.  Therefore, by the definition of the $m$-th largest
coordinate and Theorem~\ref{thm:main},
\[
    \RegOCO_\epsilon(T)
    \le [x_{T,(m)}]_+
    \le \sqrt{(t_0+2V_T)\iota_\epsilon}.
\]
In particular,
\begin{equation}
\label{eq:oco_regret_from_main}
    \bigl(\RegOCO_\epsilon(T)\bigr)^2
    \le (t_0+2V_T)\iota_\epsilon.
\end{equation}
On the other hand, \eqref{eq:oco_variance_bound} implies
\begin{equation}
	\label{eq:varian_oco_regret_bound}
    V_T\le 2LD^2
    \bigl(S_\epsilon+\RegOCO_\epsilon(T)\bigr).
\end{equation}
Indeed, when the untruncated OCO regret is positive, the learner's
cumulative loss equals
$S_\epsilon+\RegOCO_\epsilon(T)$. Otherwise, it is
at most $S_\epsilon$.  Setting
$R\coloneqq\RegOCO_\epsilon(T)$ and $a\coloneqq LD^2$, substitution
into \eqref{eq:oco_regret_from_main} gives
\begin{align*}
    R^2
    &\le \iota_\epsilon
       \bigl(t_0+4aS_\epsilon+4aR\bigr)\\
    \implies R
    &\le 2a\iota_\epsilon
      +\sqrt{4a^2\iota_\epsilon^2+t_0\iota_\epsilon
             +4a\iota_\epsilon S_\epsilon}\\
    &\le 4a\iota_\epsilon+\sqrt{t_0\iota_\epsilon}
      +2\sqrt{a\iota_\epsilon S_\epsilon},
\end{align*}
where the last line uses that squareroot is sub-additive .
This proves \eqref{eq:oco_small_loss_implicit}.

It remains to handle the \(V_T\) term inside the logarithm.  Recall that
\[
    a\coloneqq LD^2,\qquad
    C_\epsilon\coloneqq 2\log(1/\epsilon),\qquad
    A_\epsilon\coloneqq t_0+4aS_\epsilon.
\]
The variance bound from~\eqref{eq:varian_oco_regret_bound} gives
\[
    t_0+2V_T
    \le t_0+4a(S_\epsilon+R)
    =A_\epsilon+4aR.
\]
Moreover, by definition,
$\iota_\epsilon\coloneqq\log(t_0+2V_T)+C_\epsilon$.  Thus
\eqref{eq:oco_regret_from_main} can be written as
\[
    R^2\le (t_0+2V_T)
    \bigl(\log(t_0+2V_T)+C_\epsilon\bigr).
\]
Since $t_0\ge1$, $C_\epsilon\ge0$, and
$y\mapsto y(\log y+C_\epsilon)$ is increasing for $y\ge1$, substituting
the preceding upper bound on $t_0+2V_T$ yields
\begin{equation}
\label{eq:oco_bootstrap_start}
    R^2\le (A_\epsilon+4aR)
    \bigl(\log(A_\epsilon+4aR)+C_\epsilon\bigr).
\end{equation}
If $4aR\le A_\epsilon$, the first term on the right-hand side of
\eqref{eq:oco_small_loss_explicit} bounds $R$.  Otherwise,
\eqref{eq:oco_bootstrap_start} gives
\[
    R\le 8a\bigl(\log(8aR)+C_\epsilon\bigr).
\]
If $a>0$, set $z \coloneqq R/(8a)$. Then since
$64a^2\le(1+8a)^2$ we have
\[
    z\le \log z+2\log(1+8a)+C_\epsilon \implies z \leq 4\log(1+8a)+C_\epsilon,
\]
where the implication holds since $\log z\le z/2$. By the definition of \(z\), we get
$R\le16a(1+C_\epsilon+2\log(1+8a))$.  The case $a=0$ follows
directly from \eqref{eq:oco_bootstrap_start}.  Combining the two cases
proves \eqref{eq:oco_small_loss_explicit}.
\end{proof}

The constant-potential argument also gives a slightly stronger
guarantee against the average loss of the best $\epsilon$-fraction of
actions, rather than only against the action at the quantile threshold.

\begin{corollary}[Average of the best $\epsilon$-fraction]
\label{cor:oco_small_loss_average}
Let $m\in\{1,\ldots,N-1\}$, set $\epsilon\coloneqq m/N$, and let $I_\epsilon$
index the $m$ actions with the smallest cumulative losses.  Define
\[
    \overline S_\epsilon
    \coloneqq\frac{1}{m}\sum_{i\in I_\epsilon}\sum_{j=1}^T f_j(u_i),
    \qquad
    \AvgRegOCO_\epsilon(T)
    \coloneqq\left[\sum_{j=1}^T f_j(w_j)-\overline S_\epsilon\right]_+.
\]
Under the conditions of Corollary~\ref{cor:oco_small_loss},
\begin{equation}
\label{eq:oco_average_small_loss}
    \AvgRegOCO_\epsilon(T)
    \le
    2D\sqrt{L\iota_\epsilon\overline S_\epsilon}
    +4LD^2\iota_\epsilon
    +\sqrt{t_0\iota_\epsilon}.
\end{equation}
Furthermore, the explicit bound \eqref{eq:oco_small_loss_explicit}
holds with $\RegOCO_\epsilon(T)$, $S_\epsilon$, and
$A_\epsilon$ replaced by
$\AvgRegOCO_\epsilon(T)$,
$\overline S_\epsilon$, and
$\overline A_\epsilon\coloneqq t_0+4LD^2\overline S_\epsilon$, respectively.
\end{corollary}

\begin{proof}
Let
$\overline u_\epsilon\coloneqq m^{-1}\sum_{i\in I_\epsilon}u_i\in\mathcal K$.
For the NormalHedge.BH potential $\phiNH$, monotonicity and convexity
give
\begin{align*}
    m\phiNH\left(\left[\frac{1}{m}
       \sum_{i\in I_\epsilon}x_{T,i}\right]_+,t_T\right)
    &\le \sum_{i\in I_\epsilon}\phiNH([x_{T,i}]_+,t_T)\\
    &\le \PhiNH(\widetilde\bx_T,t_T)
     =N\phiNH(0,t_0).
\end{align*}
Here, the first inequality also uses
$[m^{-1}\sum_i x_{T,i}]_+\le m^{-1}\sum_i[x_{T,i}]_+$, and the last
equality is the constant-potential identity.  Inverting $\phiNH$ and
using $t_T\le t_0+2V_T$ as in Theorem~\ref{thm:main} yields
\begin{equation}
\label{eq:oco_average_regret_coordinates}
    \left[\frac{1}{m}\sum_{i\in I_\epsilon}x_{T,i}\right]_+
    \le \sqrt{(t_0+2V_T)\iota_\epsilon}.
\end{equation}

By convexity of every $f_j$ and by
\eqref{eq:oco_regret_increment},
\begin{align*}
    \sum_{j=1}^T f_j(w_j)-\overline S_\epsilon
    &\le \sum_{j=1}^T
       \bigl(f_j(w_j)-f_j(\overline u_\epsilon)\bigr)\\
    &\le \frac{1}{m}\sum_{i\in I_\epsilon}x_{T,i}.
\end{align*}
Together with \eqref{eq:oco_average_regret_coordinates}, this gives the
analogue of \eqref{eq:oco_regret_from_main}.  Finally,
\eqref{eq:oco_variance_bound} gives
\[
    V_T\le 2LD^2\bigl(
       \overline S_\epsilon
       +\AvgRegOCO_\epsilon(T)\bigr).
\]
The quadratic calculation and the explicit bootstrap from the proof of
Corollary~\ref{cor:oco_small_loss} now give the two stated bounds.
\end{proof}

\begin{remark}[Role of the assumptions]
\label{rem:oco_small_loss_assumptions}
Global nonnegativity is used only in
Lemma~\ref{lem:oco_self_bounding}.  The upper bound $M$ on
$\mathcal K$ is used only to obtain the fixed range parameter
$B=D\sqrt{2LM}$, and hence to choose $t_0$.
\end{remark}

Small-loss guarantees for OCO with nonnegative smooth losses are well
established.  The central ingredient is the self-bounding property of
smooth losses, which was used by \citet{srebro2010smoothness} to obtain
a comparator-loss bound for mirror descent.  The optimistic mirror
descent framework of \citet{rakhlin2013predictable} provides a general
route from data-dependent second-order bounds to such problem-dependent
guarantees.  Subsequent work removed the need to know the comparator
loss in advance and extended the guarantee to adaptive and dynamic
regret, among other settings
\citep{zhang2019adaptive,zhao2020dynamic}.  The results in this section should
therefore be viewed as an interpretation of our curvature-weighted
bound in this classical setting.  It is not a generic expert-small-loss
theorem: the reduction relies on the self-bounding property of
nonnegative smooth losses.

%% file: discussion_MZ_lowerbound.tex
\begingroup
\section{Further discussion of the Marinov--Zimmert lower bound}
\label{sec:further_discussion_MZ}

In this appendix, we discuss the work of \cite{marinov2021pareto} and how it does not contradict with our results. This section is unimportant for general readers, but could be helpful for readers who are interested in how the construction of \cite{marinov2021pareto} works and why it may indicate that the $\log N$ factor in $t_0$ from our Theorem~\ref{thm:main} might be necessary.

There are two differences between our result and the second-order quantile
bound ruled out by \citet{marinov2021pareto}: the curvature distribution
$\bq$ need not equal the play distribution $\bp$, and the initial time contributes
an additive $O(B^2\log N)$ inside the second-order term.
Either difference prevents the \emph{black-box contradiction below}; this
does not establish that either difference alone suffices for an achievable
bound. We first explain the lower bound as a product tradeoff, following a
suggestion of Julian Zimmert, and then examine the two differences separately.
Throughout this appendix, losses have range $[0,1]$, and constants in
$O(\cdot)$ and $\Omega(\cdot)$ are independent of $T,k,\gamma$.

\subsection{The regret-product tradeoff and the wrapper}
\label{sec:contextual_bandits_wrapper}

The lower bound concerns model selection in the  contextual
bandits setting. At each round $t$, the learner observes a context
$x_t\in\mathcal X$, chooses an action $A_t\in\mathcal A$, and observes
only its loss $\ell_{t,A_t}$, where
$\ell_t\in[0,1]^{|\mathcal A|}$ and $\mathcal A$ is finite. A policy
$\pi:\mathcal X\to\mathcal A$ specifies an action for each context.
Given known nested policy classes $\Pi_1\subset\cdots\subset\Pi_M$,
define the expected pseudo-regret
\[
\Reg(T,\Pi_m)=\max_{\pi\in\Pi_m}\mathbb E\!\left[
\sum_{t=1}^T\bigl(\ell_{t,A_t}-\ell_{t,\pi(x_t)}\bigr)\right],
\]
where the expectation is over the learner's and the environment's
randomness. Model selection asks a single algorithm to control these
regrets simultaneously, with complexity dependence determined by each
comparator class. For a fixed number of actions, the ideal scale is
$\sqrt{T(1+\log|\Pi_m|)}$, up to logarithmic overheads in $T$ and $M$.
The lower bound below already holds in the stochastic setting, where
the pairs $(x_t,\ell_t)$ are drawn independently from a fixed unknown
distribution at every round. In particular, the loss vector may depend
on the current context. Therefore, this lower bound also holds for the harder deterministic case.

\paragraph{The lower bound.}
The construction of \citet{marinov2021pareto} uses three actions and nested
policy classes $\Pi_1=\{\pi_0\}$ and
$\Pi_2=\{\pi_0,\pi_1,\ldots,\pi_k\}$. A proper algorithm chooses a policy
independently of the current context and then follows its recommendation.
In the construction, $\pi_0$ always plays action 3, while every other
policy plays action 1 or 2. Action 3 is slightly preferable in environment
$\mathcal E_0$; one of the other policies is preferable in each alternative
environment. Playing actions 1 or 2 reveals information about which
environment is present, but costs regret in $\mathcal E_0$.

\begin{theorem}[Model-selection lower bound, {\citealp[Theorem~4 and its proof]{marinov2021pareto}}]
\label{thm:model-selection}
For the policy classes in this construction, a proper algorithm with
expected regret at most $C\sqrt T$ to $\Pi_1$ in all environments has,
in some environment, expected regret to $\Pi_2$ at least
\[
\Omega\!\left(\min\left\{T,\frac{T\log k}{C\sqrt T}\right\}\right).
\]
The nontrivial parameter regime in the proof is
$c_2 C^2\le\log k\le T/2$, for an absolute constant $c_2>0$.
\end{theorem}
Write $R(\Pi_m)$ for \emph{uniform expected-regret budgets}, not for
two realized regrets on the same sample path. In the uncapped regime,
the theorem says, up to absolute constants,
\begin{equation}\label{eq:mz_product}
R(\Pi_1)\,R(\Pi_2)\gtrsim T\log k.
\end{equation}
More precisely, a budget $R(\Pi_1)=C\sqrt T$ forces
$R(\Pi_2)\gtrsim\min\{T,T\log k/R(\Pi_1)\}$.
Thus a contradiction requires upper bounds whose product is
$o(T\log k)$ \emph{and} whose second component is $o(T)$, with the
theorem's parameter conditions satisfied. The cap at $T$ matters.

\paragraph{Duplication and quantiles.}
Replace $\pi_0$ by $k$ identical copies, giving $N=2k$ experts.
For every realized surrogate-loss sequence, the median expert loss is
at most the loss of $\pi_0$: at least half the experts are its copies.
Consequently a $1/2$-quantile bound controls regret to $\pi_0$, regardless
of whether $\pi_0$ is optimal. A $1/(2k)$-quantile bound controls regret
to every policy in $\Pi_2$. This is the link from quantile adaptivity to
model selection.

\paragraph{A proper revealing-action wrapper.}
The hypothetical full-information experts bound considered
in Section~\ref{sec:self_variance_case} would imply regret bounds for the
stochastic contextual-bandit setting that contradict
Theorem~\ref{thm:model-selection}. The wrapper below provides this reduction.

For clarity about properness, we use the revealing-action structure of
the lower-bound construction. Its losses satisfy
$\ell_{t,2}=1-\ell_{t,1}$, so observing either action 1 or action 2
reveals the losses of \emph{all} policies other than $\pi_0$.
Let $p_t$ be the experts distribution chosen before the context, and
let $p_t(\pi_0)$ denote the total mass on its copies. With probability
$1-\gamma$ choose a policy from $p_t$; with probability $\gamma$ choose
$\pi_0$ or $\pi_1$, each with probability $1/2$, where $0<\gamma\le1/2$.
All policy choices are made before the context, so this wrapper is proper.
The two feedback events have probabilities
\[
s_t(0)=(1-\gamma)p_t(\pi_0)+\gamma/2,\qquad
s_t(1)=(1-\gamma)(1-p_t(\pi_0))+\gamma/2.
\]
Here event 0 is playing action 3 and event 1 is playing a revealing
action. Write $g(\pi)=0$ for copies of $\pi_0$ and $g(\pi)=1$ otherwise.
On event $g(\pi)$ the loss $\ell_{t,\pi(x_t)}$ is available, so we can feed
the experts algorithm the bounded, scaled importance-weighted losses
\begin{equation}\label{eq:mz_estimator}
\widehat\ell_t(\pi)=
\frac{\mathbf 1\{\text{event }g(\pi)\}}{s_t(g(\pi))}
\ell_{t,\pi(x_t)},\qquad
\widetilde\ell_t(\pi)=\frac\gamma2\widehat\ell_t(\pi)\in[0,1].
\end{equation}
Conditional on the past, context and loss vector, the unscaled estimator
is unbiased. This variant of the wrapper uses the special feedback
structure above; it is not a reduction for arbitrary contextual bandits.

For $\varepsilon_1=1/2$ and $\varepsilon_2=1/(2k)$, unbiasedness and
the at-most-$\gamma T$ exploration cost give
\begin{equation}\label{eq:mz_wrapper_regret}
\Reg(T,\Pi_m)\le\gamma T+
\frac{2(1-\gamma)}\gamma\,
\mathbb E\bigl[\Reg^{\mathrm{exp}}_{\varepsilon_m}(T;\widetilde\ell)\bigr].
\end{equation}
The left side is expected regret to the best fixed policy in expectation.
To justify the inequality, first compare to each fixed policy using
unbiasedness, then use the pathwise quantile bound on the right.
No interchange of expectation and a random best-policy minimum is needed.

\subsection{Why the unmodified self-variance bound is impossible}
\label{sec:self_variance_case}

Let $\mathbb E_t$ be the exepctation conditioning on the past, current context and current losses for round $t$,
leaving only the policy-selection randomization. Since
$s_t(g)\ge(1-\gamma)\sum_{\pi:g(\pi)=g}p_t(\pi)$,
\begin{align}
\mathbb E_t\operatorname{Var}_{\pi\sim p_t}(\widetilde\ell_t(\pi))
&\le\mathbb E_t\sum_\pi p_t(\pi)\widetilde\ell_t(\pi)^2\nonumber\\
&=\frac{\gamma^2}{4}\sum_\pi
\frac{p_t(\pi)\ell_{t,\pi(x_t)}^2}{s_t(g(\pi))}
\le\frac{\gamma^2}{2(1-\gamma)}\le\gamma^2.
\label{eq:mz_self_moment}
\end{align}
Thus $\mathbb E V_T^p(\widetilde\ell)\le\gamma^2T$.
If a full-information algorithm satisfied the hypothetical bound
\[
\Reg^{\mathrm{exp}}_\varepsilon(T)
\le c\sqrt{V_T^p\log(1/\varepsilon)},
\]
then \eqref{eq:mz_wrapper_regret}, \eqref{eq:mz_self_moment} and Jensen's
inequality would give the simultaneous upper-bound budgets
\begin{equation}\label{eq:mz_no_slack_budgets}
U_1=\gamma T+O(\sqrt T),\qquad
U_2=\gamma T+O(\sqrt{T\log k}).
\end{equation}
For example, choose $\gamma=T^{-1/2}$ and $\log k=\Theta(T^{1/3})$.
Then $U_1=O(\sqrt T)$, $U_2=O(T^{2/3})$, and
$U_1U_2=O(T^{7/6})=o(T\log k)$, contradicting
\eqref{eq:mz_product}. All the lower-bound conditions hold for sufficiently
large $T$. This illustrates the mechanism behind the stronger impossibility
result in \citet{marinov2021pareto}.

\subsection{Why the $\log N$ slack alone removes this contradiction}
\label{sec:logNplusV}

Keep the variance measure $V_T^p$, but instead suppose
\begin{equation}\label{eq:mz_slack_hypothesis}
\Reg^{\mathrm{exp}}_\varepsilon(T)
\le c\sqrt{(V_T^p+\log N)\log(1/\varepsilon)}.
\end{equation}
The scaled losses still lie in $[0,1]$, so the $\log N$ term remains
inside this bound before undoing the scaling in
\eqref{eq:mz_wrapper_regret}. Since $N=2k$, this term is $\log(2k)$.
For a sufficiently large absolute constant $c'>0$, valid budgets are
\begin{align}
U_1&=\gamma T+\frac{c'}\gamma\sqrt{\gamma^2T+\log(2k)},\nonumber\\
U_2&=\gamma T+\frac{c'}\gamma\sqrt{(\gamma^2T+\log(2k))\log(2k)}.
\label{eq:mz_slack_budgets}
\end{align}
The obstruction is now visible without optimizing $\gamma$:
\begin{equation}\label{eq:mz_slack_product}
U_1U_2
\ \ge\ (\gamma T)\left(\frac{c'\log(2k)}\gamma\right)
\ =\ c'T\log(2k).
\end{equation}
In words, the exploration term in the first budget times the rescaled
$\log N$ slack in the second already has the order of the lower-bound
product. No choice of $\gamma$ makes these displayed budgets violate
\eqref{eq:mz_product} in order. In the capped regime $U_1\lesssim\log(2k)$,
we have $\gamma T\le U_1\lesssim\log(2k)$, and hence
$U_2\ge c'\log(2k)/\gamma\gtrsim T$; there is no sublinear upper bound to
contradict the cap either.

Importantly, \eqref{eq:mz_slack_product} is an inequality about the
\emph{derived upper-bound expressions} using the importance sampling estimator from the hypothesized upper bound \eqref{eq:mz_slack_hypothesis}. Our argument does not say \eqref{eq:mz_slack_hypothesis} is achievable or impossible. It only shows that these wrapper estimates do not yield an order-wise contradiction with the lower bound of \citet{marinov2021pareto}; a sharper use of the same construction is not excluded.

\paragraph{A direct impossibility of second-order bound depending on $V_T^p$.}  As we were preparing for the camera-ready version of this paper, \citet{LuoZhang2026PersonalCommunication} shared with us, in a personal communication dated October 22, 2026, a new lower bound ruling out \eqref{eq:mz_slack_hypothesis} for any algorithm when $V_T^p$ is the
variance under its played distribution. Their new construction does not go through the wrapper reduction to model-selection, and can be viewed as an improved version of \citet{marinov2021pareto}. Thus, retaining $V_T^p$ and merely adding $\log N$ is insufficient for this clean quantile bound. This does not establish that changing the variance is necessary for weaker guarantees with additional slack or logarithmic factors.

\subsection{Why changing the variance can also remove the contradiction}
\label{sec:different_var_def}

Now omit the $\log N$ slack, but replace $p_t$ by a predictable distribution
$r_t$ in the second moment:
\[
V_T^r(\widetilde\ell)=\sum_{t=1}^T\sum_\pi r_t(\pi)
\bigl(\langle p_t,\widetilde\ell_t\rangle-\widetilde\ell_t(\pi)\bigr)^2.
\]
This is the form used in our paper when $r_t=q_t$. It is generally
\emph{not} $\operatorname{Var}_{r_t}(\widetilde\ell_t)$: the centering
remains at the loss of the played distribution $p_t$.
Unlike \eqref{eq:mz_self_moment}, the generic estimator bound is only
\[
\mathbb E_t\sum_\pi r_t(\pi)\widetilde\ell_t(\pi)^2
=\frac{\gamma^2}{4}\sum_\pi
\frac{r_t(\pi)\ell_{t,\pi(x_t)}^2}{s_t(g(\pi))}
\le\frac\gamma2.
\]
Using $(a-b)^2\le2a^2+2b^2$,
$\langle p_t,\widetilde\ell_t\rangle^2\le
\langle p_t,\widetilde\ell_t^2\rangle$, and
\eqref{eq:mz_self_moment}'s second-moment estimate, we obtain
$\mathbb E V_T^r(\widetilde\ell)\le(2\gamma^2+\gamma)T\le2\gamma T$.
The missing factor of $\gamma$ reflects the mismatch between the
sampling weights and the weights in the second moment.

Consequently a hypothetical $c\sqrt{V_T^r\log(1/\varepsilon)}$ bound
gives budgets of the form
\begin{equation}\label{eq:mz_other_variance_budgets}
U_1=\gamma T+c'\sqrt{T/\gamma},\qquad
U_2=\gamma T+c'\sqrt{T\log(2k)/\gamma}.
\end{equation}
These generic budgets again do not contradict the lower bound. Indeed,
$U_2=o(T)$ requires $\log(2k)=o(\gamma T)$, in which case
\begin{align*}
\frac{U_1U_2}{T\log(2k)}
&\ge c'\frac{\gamma T\sqrt{T\log(2k)/\gamma}}{T\log(2k)}\\
&=c'\sqrt{\frac{\gamma T}{\log(2k)}}\longrightarrow\infty.
\end{align*}
Conversely, when $\log(2k)$ is of order $\gamma T$ or larger, the displayed
$U_2$ is already of order $T$ or larger.
This explains why the self-variance argument does not extend verbatim.
A sharper analysis using special relations between $p_t$ and the
particular curvature weights $q_t$ is not excluded by this calculation.

\subsection{Are both differences necessary?}

Our proof uses both, but the lower bound above does not show that both
are necessary. Their roles in the proof are distinct. The curvature
weights arise from the second spatial derivative of the potential in
the backward heat equation, whereas the $\log N$ warm-up controls the
discrete-time error uniformly over coordinates. The latter has a different
origin from the continuous-time variance measure.
Within this potential-based approach, retaining the curvature weights
while improving the discretization control therefore appears the more
direct route to removing one difference. This is an intuition about our
proof, not a necessity statement about all algorithms.

For the clean quantile form \eqref{eq:mz_slack_hypothesis}, merely adding
the $\log N$ slack while retaining action variance $V_T^p$ is insufficient,
by the result of \citet{LuoZhang2026PersonalCommunication} discussed in Section~\ref{sec:logNplusV}.
The remaining question is whether one can retain the specific curvature
measure $V_T^q$ while removing the explicit $\log N$ slack from our
guarantee. We do not resolve this question here: the mismatch calculation
does not rule it out, but is not an existence proof. The alternative
second-order expression in the next appendix is not identical to the
prescribed $V_T^q$, so it does not settle this question either.
The examples in Appendix~\ref{sec:variance_examples} illustrate why
changing the measure is substantive, including after active-reentry
refinement; they do not establish that the explicit $\log N$ warm-up is
necessary.
\endgroup